\documentclass{article}
\usepackage{amsmath,amsfonts}
\usepackage{algorithm}
\usepackage{array}
\usepackage{textcomp}
\usepackage{url}
\usepackage{verbatim}
\usepackage{graphicx}
\usepackage{algpseudocode}
\usepackage{amssymb, bm, mathtools, mathrsfs}
\usepackage{amsthm}
\usepackage{amsfonts}
\usepackage{footnote}
\usepackage{dsfont}
\usepackage{microtype}
\usepackage{cite}
\usepackage[top=2.5cm, bottom=2.5cm, left=2.5cm, right=2.5cm]{geometry}
\usepackage{hyperref}

\newcommand{\ind}{\mathds{1}}

\newcommand{\brc}[1]{\left\{{#1}\right\}}
\newcommand{\prn}[1]{\left({#1}\right)} 
\newcommand{\brk}[1]{\left[{#1}\right]} 
\newcommand{\norm}[1]{\left\|{#1}\right\|} 
\newcommand{\abs}[1]{\left|{#1}\right|} 

\newcommand{\wtilde}[1]{\widetilde{#1}}
\newcommand{\sgn}{\mathsf{sign}}

\newcommand{\<}{\langle} 
\renewcommand{\>}{\rangle}

\DeclareMathOperator*{\argmin}{argmin}

\def\gA{{\mathcal{A}}}

\def\gF{{\mathcal{F}}}

\def\gH{{\mathcal{H}}}
\def\gI{{\mathcal{I}}}

\def\gL{{\mathcal{L}}}
\def\gM{{\mathcal{M}}}
\def\gN{{\mathcal{N}}}
\def\gO{{\mathcal{O}}}
\def\gP{{\mathcal{P}}}

\def\gR{{\mathcal{R}}}
\def\gS{{\mathcal{S}}}
\def\gT{\bm{\mathcal{T}}}

\def\gW{{\mathcal{W}}}

\def\sP{{\mathscr{P}}}

\def\bA{{\bm{A}}}

\def\bb{{\bm{b}}}
\def\be{{\bm{e}}}
\def\bB{{\bm{B}}}
\def\bC{{\bm{C}}}

\def\bh{{\bm{h}}}
\def\bH{{\bm{H}}}
\def\bw{{\bm{w}}}
\def\bF{{\bm{F}}}
\def\bW{{\bm{W}}}
\def\bM{{\bm{M}}}
\def\bUp{{\bm{\Upsilon}}}
\def\bI{{\bm{I}}}
\def\bg{{\bm{g}}}
\def\bG{{\bm{G}}}
\def\bJ{{\bm{J}}}

\def\bR{{\bm{R}}}
\def\bP{{\bm{P}}}
\def\bp{{\bm{p}}}
\def\bPi{{\bm{\Pi}}}
\def\bphi{{\bm{\phi}}}
\def\bpsi{{\bm{\psi}}}
\def\bPhi{{\bm{\Phi}}}

\def\bLamb{{\bm{\Lambda}}}

\def\bSigma{{\bm{\Sigma}}}

\def\btheta{{\bm{\theta}}}
\def\bTheta{{\bm{\Theta}}}

\def\bu{{\bm{u}}}

\def\bxi{{\bm{\xi}}}
\def\btheta{{\bm{\theta}}}
\def\bV{{\bm{V}}}
\def\bY{{\bm{Y}}}

\def\bz{{\bm{z}}}

\newcommand{\LCTD}{\textnormal{\itshape Linear-CTD}}
\newcommand{\VCTD}{\textnormal{\itshape VrFLCTD}}
\newcommand{\LTD}{\textnormal{\itshape Linear-TD}}

\newcommand{\inner}[2]{\left\langle #1, #2 \right\rangle}

\def\RB{{\mathbb R}}
\def\EB{{\mathbb E}}

\def\NB{{\mathbb N}}

\def\vect{\mathsf{vec}}

\newcommand{\tr}{\operatorname{Tr}}

\providecommand{\citep}[1]{\cite{#1}}
\providecommand{\citet}[1]{\cite{#1}}
\allowdisplaybreaks
\newtheorem{proposition}{Proposition}
\newtheorem{corollary}{Corollary}
\newtheorem{assumption}{Assumption}
\newtheorem{remark}{Remark}
\newtheorem{theorem}{Theorem}
\newtheorem{lemma}{Lemma}
\begin{document}
\title{Accelerated Distributional Temporal Difference Learning with Linear Function Approximation}

\author{%
    Kaicheng Jin\thanks{School of Mathematical Sciences, Peking University. Email: \texttt{kcjin@pku.edu.cn}}, 
    Yang Peng\thanks{School of Mathematical Sciences, Peking University. Email: \texttt{pengyang@pku.edu.cn}}, 
    Jiansheng Yang\thanks{School of Mathematical Sciences, Peking University. Email: \texttt{yjs@math.pku.edu.cn}}, 
    Zhihua Zhang\thanks{School of Mathematical Sciences, Peking University. Email: \texttt{zhzhang@math.pku.edu.cn}}
}
\maketitle
\begin{abstract}
In this paper, we study the finite-sample statistical rates of distributional temporal difference (TD) learning with linear function approximation. 
The purpose of distributional TD learning is to estimate the return distribution of a discounted Markov decision process for a given policy.
Previous works on statistical analysis of distributional TD learning focus mainly on the tabular case. 
We first consider the linear function approximation setting and conduct a fine-grained analysis of the linear-categorical Bellman equation. 
Building on this analysis, we further incorporate variance reduction techniques in our new algorithms to establish tight sample complexity bounds independent of the support size $K$ when $K$ is large. 
Our theoretical results imply that, when employing distributional TD learning with linear function approximation, learning the full distribution of the return function from streaming data is no more difficult than learning its expectation.
This work provide new insights into the statistical efficiency of distributional reinforcement learning algorithms.

\noindent\textbf{Keywords:} distributional reinforcement learning, distributional policy evaluation, distributional temporal difference learning, stochastic approximation, sample complexity
\end{abstract}

\section{Introduction}\label{Section:intro}
Distributional policy evaluation (DPE) \citep{morimura2010nonparametric,bellemare2017distributional,bdr2022} aims to estimate the return distribution of a policy in a Markov decision process (MDP).
It is a crucial issue for many uncertainty-aware or risk-sensitive tasks \citep{NEURIPS2022_c88a2bd0,NEURIPS2023_b0cd0e80,moghimi2025beyond, pires2025optimizing}.
While conventional policy evaluation  methods estimate the expected value of returns, DPE models the entire return distribution, providing a more comprehensive understanding of the uncertainty and risk profiles of policies in MDPs.
The foundational work \citep{bellemare2017distributional} presented distributional temporal difference (TD) learning algorithm to solve DPE, which generalizes the classical TD method for conventional policy evaluation \citep{sutton1988learning}.
Although classic TD learning has been well studied theoretically \cite{bertsekas1995neuro, NIPS1996_e0040614, bhandari2018finite, Dalal_Szörényi_Thoppe_Mannor_2018, patil2023finite,li2024q,li2024high, chen2024lyapunov,samsonov2024gaussian,samsonov2024improved, wu2024statistical}, the theoretical foundation of distributional TD learning has been less explored.
In the tabular setting, where algorithms directly estimate the parameters characterizing the return distributions, recent developments \cite{rowland2018analysis,doi:10.1137/20M1364436,zhang2025estimation,rowland2024analysis,rowland2024nearminimaxoptimal,peng2024statistical,NEURIPS2023_b0cd0e80} 
provide basic understandings.
Among them, \citet{rowland2024nearminimaxoptimal} and \cite{peng2024statistical} showed that, in the tabular setting, learning the return distribution under the $1$-Wasserstein distance is statistically efficient as learning its mean.

The theoretical analysis of distributional TD learning beyond the tabular setting remains challenging.
In particular, the theoretical implications of introducing function approximation in the parameterization of the return distribution, which is crucial for scaling to practical problems where state spaces are large, have not been fully investigated.
In this context, function approximation refers to mapping state features onto parametrized return distribution via a defined function class.
An open question is: When a specific form of function approximation is employed, does learning the return distribution remain as statistically efficient as learning its expectation?

A natural starting point for distributional TD learning with function approximation is the linear case.
For instance, with the categorical parameterization of the return distribution, features are mapped linearly to the probabilities (point masses) over the distribution supports.
This parameterization method is termed linear-categorical parameterization. Within this framework, one can establish the linear system that characterizes the linear-categorical projected Bellman equation and propose a corresponding distributional TD learning algorithm as in \cite{bdr2022}.
In the conference version \cite{peng2025finitesampleanalysisdistributional} 
of this paper, we propose a new algorithm, linear-categorical TD learning ({\LCTD}), which can be seen as a preconditioned version of the vanilla algorithm proposed in \cite{bdr2022}. 
In \cite{peng2025finitesampleanalysisdistributional}, we investigate the finite-sample performance of {\LCTD} with Polyak-Ruppert tail averaging and a constant step size.
Notably, the sample complexity required for {\LCTD} to achieve an $\varepsilon$-accurate estimator of the return distributions under the $1$-Wasserstein distance parallels the classical results for linear TD learning algorithm ({\LTD})\citep{samsonov2024improved}.
 
In this work, as an extension of our previous work \citep{peng2025finitesampleanalysisdistributional}, we advance the theoretical understanding of {\LCTD} algorithm and propose improved algorithms.
Specifically, we first provide an instance-dependent analysis of the baseline algorithm {\LCTD} and derive a better convergence result. 
Furthermore, by incorporating the variance reduction techniques from \citet{li2023accelerated} and \cite{doi:10.1137/040615961} into our new variance-reduced fast {\LCTD} algorithm ({\VCTD}), we derive a sharp upper bound of estimation error with respect to the length of the effective horizon. 
In particular, the leading term in the corresponding sample complexity matches the asymptotic minimax lower bound for learning TD fixed point under linear-categorical parameterization, which we establish in this work. 
Consequently, we obtain a better sample complexity upper bound which is a counterpart to the instance-optimal sample complexity achieved by variance-reduced {\LTD} in \citet{li2023accelerated}.
This further strengthens the core theme that, with TD learning algorithms, DPE with linear-categorical parameterization is not statistically harder than policy evaluation with linear function approximation.
Crucially, for this claim to hold, the sample complexity upper bound should hold for an arbitrary support-size $K$ of the categorical distribution, which means that $K$ must not appear polynomially in the sample complexity bounds.

The preservation of this property is a central objective of this work.
The value of $K$ directly affects the solution of the linear-categorical projected Bellman equation and the resulting magnitude of the approximation error.
Crucially, our algorithm ensures that with the same sample complexity and an instance-independent step size, $K$ can be increased to achieve the best possible approximation performance, despite incurring a higher space complexity.
In our conference version \citep{peng2025finitesampleanalysisdistributional}, this property is realized through a preconditioning technique, which stems directly from a non-trivial formulation of the linear-categorical parameterization.
In this work, we inherit this basic formulation in our new variance-reduced algorithms.
In both the generative model setting and the Markovian model setting, we analyze the terms related to $K$ in the corresponding statistical error bounds and prove that the sample complexity is independent of $K$ under mild conditions.
Together, these theoretical developments provide an affirmative answer to the open problem in the linear function approximation case.

\subsection{Contributions}
\begin{itemize}
\item 
We perform an instance-dependent analysis of distributional TD learning with linear-categorical parameterization.
First, we present an asymptotic minimax lower bound on the sample complexity required to estimate the TD fixed point under linear-categorical parameterization, based on a discrete MDP formulation with categorical parameterization of the return distributions.
We further establish an instance-dependent convergence analysis for the baseline algorithm {\LCTD}, which covers existing convergence results for {\LCTD}.
The dominant term of the sample complexity upper-bound matches the derived minimax lower bound.

\item 
We propose a novel variance-reduced fast algorithm {\VCTD} in the generative model setting.
We further provide a sharp analysis of its convergence rate, achieving a sample complexity that matches the minimax lower bound that we establish without a sample size barrier.
Specifically, {\VCTD} needs
\begin{equation*}
    \wtilde{\gO}\prn{\varepsilon^{-2}(1-\gamma)^{-2}\lambda_{\min}^{-1}\prn{\norm{\btheta^{\star}}^2_{V_1}+1}}.
\end{equation*}
samples to achieve an $\varepsilon$-estimator in the $\mu_{\pi}$-weighted 1-Wasserstein metric when the support size $K\geq (1-\gamma)^{-1}$.
We further extend our algorithm {\VCTD} to the Markovian setting, where observations are drawn from a stationary Markovian data trajectory.
Notably, the convergence performance of {\VCTD} in both settings does not depend on the size support $K$ of the categorical distribution when $K\geq (1-\gamma)^{-1}$.
\end{itemize}

\subsection{Organization}
The remainder of this paper is organized as follows.
Section 2 introduces the preliminaries of distributional policy evaluation and linear-categorical parameterization.
Section 3 presents the linear system that determines the TD fixed point under linear-categorical parameterization and the convergence guarantees for the baseline {\LCTD} algorithm we establish in the conference version \cite{peng2025finitesampleanalysisdistributional}.
Section 4 provides an instance-dependent analysis of learning the TD fixed point, establishing an asymptotic minimax lower bound along with an improved convergence result for {\LCTD}.
Section 5 proposes the variance-reduced fast {\VCTD} algorithm and derives a statistical error upper-bound that match the minimax lower bound in the generative model setting.
We further extend {\VCTD} to the Markovian setting along with corresponding convergence guarantees. In both settings, we establish sample complexity upper bounds independent of the support size $K$.

\section{Backgrounds}\label{Section:background}
\noindent In this section, we present the basics of distribution policy evaluation (DPE) and the learning objective of distributional TD learning with linear-categorical parameterization.
Regarding DPE with categorical parameterization and tabular categorical TD learning, we refer readers to Appendix~\ref{appendix:CTD}.
\subsection{Notation}
Here and later, ``$\lesssim$'' (resp. ``$\gtrsim$'') means no larger (resp. smaller) than up to a multiplicative universal constant, and $a\simeq b$ means $a\lesssim b$ and $a\gtrsim b$ hold simultaneously.
The asymptotic notation $f(\cdot)=\wtilde{\gO}(g(\cdot))$ (resp. $\wtilde{\Omega}(g(\cdot))$)
means that $f(\cdot)$ is order-wise no larger (resp. smaller) than $g(\cdot)$, ignoring the logarithmic factors of polynomials of $(1-\gamma)^{-1}, \lambda_{\min}^{-1}, \alpha^{-1}$, $\varepsilon^{-1}$, $K$, $\|\bpsi^\star\|_{\bSigma_{\bphi}}$, $\|\btheta^{\star}\|_{\bI_K\otimes\bSigma_{\bphi}}$, $t_{\text{mix}}$. 
We will explain the concrete meaning of these notations once we have encountered them for the first time. 
We denote by $\bm{1}_K{\in}\RB^K$ the all-ones vector, 
$\bm{0}_K{\in}\RB^K$ the all-zeros vector, 
$\bI_K{\in}\RB^{K{\times} K}$ the identity matrix, 
$\|\bu\|$ the Euclidean norm of any vector $\bu$, 
$\|\bB\|$ the spectral norm of any matrix $\bB$, 
and $\|\bu\|_B:=\sqrt{\bu^{\top}\bB\bu}$ when $\bB$ is positive semi-definite (PSD).
For any matrix $\bB{=}[\bb(1), \ldots, \bb(n) ]{\in}\RB^{m{\times} n}$, we define its vectorization as $\vect(\bB){=}(\bb(1)^{\top},\ldots, \bb(n)^{\top})^{\top}{\in}\RB^{mn}$. 
The trace of a matrix is denoted by $\tr(\cdot)$.
Given a set $A$, we denote by $\Delta(A)$ the set of all probability distributions over $A$. 
For simplicity, we abbreviate $\Delta([0,(1{-}\gamma)^{-1}])$ as $\sP$. 
We denote by $\mathbf{Diag}\{\bM_{\alpha} | \alpha \in \mathcal{I}\}$ the block diagonal matrix formed from a family of matrices, where $\mathcal{I}$ is an index set. We use $\bB^{\dagger}$ to denote the Moore-Penrose pseudoinverse of a matrix $\bB$. 

Notations of loss function: $\ell_{2,\mu_\pi}(\bm{\eta}_1,\bm{\eta}_2):=(\EB_{s\sim\mu_{\pi}}[\ell_2^2(\eta_1(s),\eta_2(s))])^{1/2}$ is the $\mu_\pi$-weighted Cram\'er distance between $\bm{\eta}_1,\bm{\eta}_2\in(\sP^{\sgn})^\gS$, which is also called the $\mu_\pi$-weighted $L^2$-distance. $\mu_\pi$-weighted 1-Wasserstein distance is denoted as $W_{1,\mu_{\pi}}(\bm{\eta}_1,\bm{\eta}_2):=(\EB_{s\sim\mu_{\pi}}[W_1^2(\eta_1(s),\eta_2(s))])^{1/2}$.

We clarify the following terminology: the phrase ``distributional TD learning with linear-categorical parameterization'' formally refers to the process of finding the parameter $\bm{\theta}^{\star}$ for which $\bm{\eta}_{\bm{\theta}^{\star}}$ becomes the TD fixed point of the linear-categorical projected Bellman equation (Proposition~\ref{thm:linear_cate_TD_equation}). 
In the following sections, for conciseness, we will use the abbreviated expressions ``learning $\bm{\theta}^{\star}$'' and ``estimating $\bm{\theta}^{\star}$'', when the meaning is clear from the context. 
The abbreviation ``i.i.d. '' stands for ``independent and identically distributed''.

\subsection{Preliminaries}
In this subsection, we recap the basics of linear-categorical temporal difference learning and the theoretical foundation of the algorithm design. 
\paragraph{Distributional Policy Evaluation}
A discounted MDP is defined by a $4$-tuple $M=\<\gS,\gA,\gP,\gamma\>$.
We assume that the state space $\gS$ and the action space $\gA$ are both Polish spaces, that is, complete separable metric spaces. 
${\gP(\cdot,\cdot\mid s,a)}$ is the joint distribution of the reward and the next state conditioned on $(s,a)\in\gS\times\gA$. 
All rewards are bounded random variables that take values in the interval $[0, 1]$.
And $\gamma\in(0,1)$ is the discount factor. 
Given a policy $\pi\colon\gS\to\Delta(\gA)$ and an initial state $s_0=s\in\gS$, a random trajectory $\{(s_t,a_t,r_t)\}_{t=0}^\infty$ can be sampled: $a_t\mid s_t\sim\pi(\cdot\mid s_t)$, $(r_t,s_{t+1})\mid (s_t,a_t)\sim \gP({\cdot,\cdot}\mid s_t,a_t)$, for any time step $t\in\NB$. 

We assume that the Markov chain $\{s_t\}_{t=0}^\infty$ has a unique stationary distribution $\mu_\pi\in\Delta(\gS)$. 
We assume that the reward $r \in [0,1]$ and define the return of the trajectory by $G^\pi(s):=\sum_{t=0}^\infty \gamma^t r_t$. 
The value function $V^\pi(s)$ is defined as the expectation of $ G^\pi(s)$. 
It follows that the value function lies in the effective horizon $\brk{0,(1-\gamma)^{-1}}$. The task of finding the value function for a certain policy $\pi$ is called policy evaluation. 
The task of finding the whole distribution of $G^\pi(s)$ is called distributional policy evaluation (DPE).

Within this formulation, we use $\eta^\pi(s)\in\sP$ to denote the distribution of the return $G^\pi(s)$ and let ${\bm{\eta}}^\pi:=(\eta^\pi(s))_{s\in\gS}\in\sP^\gS$. 
Then ${\bm{\eta}}^\pi$ is characterized as the solution of the distributional Bellman equation:
\begin{equation}\label{eq:distributional_Bellman_equation}
    \eta^\pi(s)=\EB_{a\sim\pi(\cdot\mid s),(r,s^\prime)\sim\gP(\cdot,\cdot\mid s,a)}[\prn{b_{r,\gamma}}_\#\eta^\pi(s^\prime)],\quad\forall s\in\gS,
\end{equation}
where RHS is the distribution of $r_0+\gamma G^\pi(s_1)$ conditioned on $s_0=s$. 
Here $b_{r,\gamma}(x):=r+\gamma x$ for any $x\in\RB$, and $f_\#\nu\in\sP$ is defined as $f_\#\nu(A):=\nu(\{x\colon f(x)\in A\})$ for any function $f\colon \RB\to\RB$, probability measure $\nu\in\sP$ and Borel set $A\subseteq\RB$. 
The distributional Bellman equation is written as ${\bm{\eta}}^\pi={\bm{\gT}}^\pi{\bm{\eta}}^\pi$ where the operator ${{\gT}}^\pi\colon \sP^\gS\to \sP^\gS$ is called the distributional Bellman operator. 
In the task of DPE, our purpose is to find a good approximation of ${\bm{\eta}}^\pi$ for some given policy $\pi$.

\paragraph{Observation model}

The type of observation model determines the intrinsic mechanism of the data stream processed by algorithms solving DPE. In this work, we consider observation models that generate online interactions. 
The input samples for the algorithms satisfy the condition that any state $s_t$ in the streaming data $\{\xi_t\}_{t=0}^\infty=\{(s_t,a_t,s_t^{\prime},r_t\})_{t=0}^\infty$ follows the stationary distribution of the underlying Markov chain. 

Following previous work, we consider two types of online observation model.  
In the generative model setting, the streaming data $\{\xi_t\}_{t=0}^\infty=\{(s_t,a_t,s_t^{\prime},r_t\})_{t=0}^\infty$ are collected through $s_t\sim\mu_{\pi}(\cdot), a_t\sim\pi(\cdot|s_t), (r_t,s_t^\prime)\sim \gP(\cdot,\cdot|s_t,a_t)$ from the generative model (of the corresponding MDP). 
Another important observation model is called the Markovian observation model or the Markovian noise model. In the Markovian observation model, the streaming data are collected through $s_0\sim\mu_{\pi}(\cdot), a_t\sim\pi(\cdot|s_t), (r_t,s_{t+1})\sim \gP(\cdot,\cdot|s_t,a_t)$, and $s_{t}^{\prime}=s_{t+1}$. 
In both settings, samples can be seen as online interactions with the MDP since each state $s_t$ follows the stationary distribution $\mu_{\pi}$. 
Correlated data samples produced by the Markovian observation model pose an additional challenge for algorithm design and convergence rate analysis. 
For conciseness, the rest of this paper will refer to these two settings as the generative model setting and the Markovian setting.

\paragraph{Linear-Categorical Parameterization}
A key choice in distributional policy evaluation (DPE) is the parameterization of the return distribution. Linear-categorical parameterization offers a tractable framework. 
For a clear description of this method, we consider the space of linear-categorical parameterized signed measures with total mass $1$ we introduce in \citet{peng2025finitesampleanalysisdistributional}:
\begin{equation*}
    \sP^{\sgn}_{\bphi,K} := \left\{\bm{\eta}_{\btheta}=\prn{\eta_{\btheta}(s)}_{s\in\gS}\colon \eta_{\btheta}(s)=\textstyle\sum_{k=0}^Kp_k(s;\btheta)\delta_{x_k}, \btheta=(\btheta(0)^{\top}, \ldots, \btheta(K{-}1)^{\top})^{\top}\in\RB^{dK} \right\},
\end{equation*}
which is an affine subspace of $(\sP^{\sgn}_{K})^{\gS}$. Here, $\sP^{\sgn}_K$ denotes the
the space of all categorical parameterized signed measures with total mass $1$:
\begin{equation}\label{eq:def_linear_parametrize}
    \sP^{\sgn}_K := \{ \nu_\bp=\textstyle\sum_{k=0}^K p_k \delta_{x_k} \colon  \bp=\prn{p_0, \ldots, p_{K{-}1}}^{\top}\in \RB^{K}, p_K=1-
    \textstyle\sum_{k=0}^{K-1}p_k \}.
\end{equation}
To incorporate state features, we assume that there is a $d$-dimensional feature vector for each state $s\in \gS$, which is given by the feature map $\bphi :\gS\to \RB^{d}$, such that $\|\bphi(s)\|\leq1$ for all $s\in \gS$.

With this representation, the cumulative probability function (CDF) $F_k(s;\btheta)$ is defined through
\begin{equation*}
F_k(s;\btheta)= 
    \begin{cases}
        0, & \text{ for}\ k=-1, \\
        \bphi(s)^{\top}\btheta(k)+\frac{k+1}{K+1}, & \text{ for}\ k\in\brc{0, 1, \ldots, K-1}, \\
        1, & \text{ for}\ k=K.
    \end{cases}
\end{equation*}
And the point mass function is defined as $p_k(s;\btheta){=}F_k(s;\btheta){-}F_{k{-}1}(s;\btheta)$. 
The above formulation of linear-categorical parameterization leads to the definition of the linear-categorical projection operator $\bPi_{\bphi, K}^{\pi}{\colon}(\sP^{\sgn})^{\gS}{\to}\sP^{\sgn}_{\bphi,K}$ as follows:
\begin{equation*}
    \bPi_{\bphi, K}^{\pi}\bm{\eta}:=\argmin\nolimits_{\bm{\eta}_{\btheta}\in\sP^{\sgn}_{\bphi,K}}\ell_{2,\mu_\pi}\prn{\bm{\eta}, \bm{\eta}_{\btheta}},\quad\forall\bm{\eta}\in (\sP^{\sgn})^{\gS}.
\end{equation*}
In this definition, $\bPi_{\bphi, K}^{\pi}\bm{\eta}$ represents the projection onto the linear-categorical parameterized signed measures in the $\mu_{\pi}$-weighted $L^2$ metric. 
We provide several useful properties of linear-categorical  parameterization in the conference version \citep{peng2025finitesampleanalysisdistributional}, which will be cited in later sections as needed. 
With this formulation, the linear-categorical projected Bellman equation can be written as $\bm{\eta}_{\btheta}{=}\bPi_{\bphi, K}^{\pi}\gT^{\pi}\bm{\eta}_{\btheta}$ and its solution $\bm{\eta}_{\btheta^{\star}}$ serves as an approximation of the return distribution. 
We will show the existence and uniqueness of $\btheta^{\star}$ later in Proposition~\ref{thm:linear_cate_TD_equation}. 
More discussion on the validity of this solution is deferred to Appendix~\ref{appendix:projection}. 
Consequently, we establish an approach to DPE by learning the TD fixed point under linear-categorical parameterization, namely, by estimating $\btheta^{\star}$ from the data stream generated by a given observation model. 

To this end, we begin by analyzing how the discount factor $\gamma$, the support size $K$ and the feature map $\bm{\phi}$ shape the approximation error of $\bm{\eta}_{\btheta^{\star}}$ with respect to the ground truth $\bm{\eta}^{\pi}$.
We have the following important notations. Let $\iota_{K} = 1/((1-\gamma)K)$ be the distance between the equally spaced support points of the return distribution. 
The feature covariance matrix is denoted as $\bSigma_{\bphi}:=\EB_{s\sim\mu_{\pi}}\brk{\bphi(s)\bphi(s)^{\top}} \in \RB^{d\times d}$. 
Let $\lambda_{\max}$ and $\lambda_{\min}$ denote its largest and smallest eigenvalues. 
It is easy to see that $0\leq\lambda_{\min}\leq \lambda_{\max}\leq 1$. 

Here we provide an instance-dependent upper bound for the approximation error in Cram\'er distance to measure the quality of $\btheta^{\star}$ by the distance between $\bm{\eta}_{\btheta^{\star}}$ and the ground truth $\bm{\eta}^\pi$. 
This result is a refinement of the upper bound we establish in \citet{peng2025finitesampleanalysisdistributional}.

\begin{proposition}\label{prop:fine_grained_approx_error}
Let
\begin{align*}
\bUp=&\EB\brc{\brk{(\bC^{\top}\bC)^{\frac{1}{2}}\tilde{\bG}(r)(\bC^{\top}\bC)^{-\frac{1}{2}}} \otimes \prn{\bSigma_{\bphi}^{-\frac{1}{2}}\bphi(s)\bphi(s^\prime)^\top\bSigma_{\bphi}^{-\frac{1}{2}}}}, 
\end{align*} 
where the expectation is taken on the tuple $(s,a,s^{\prime},r)$ drawn from the generative model. Then it holds that
\begin{align*}
\ell^2_{2,\mu_{\pi}}\prn{\bm{\eta}_{\btheta^{\star}},\bm{\eta}^\pi}\
    \leq&\prn{1+\norm{(\bI_{dK}-\bUp)^{-1}(\gamma\bI_{dK}-\bUp\bUp^{\top})(\bI_{dK}-\bUp)^{-\top}}}\ell^2_{2,\mu_{\pi}}\prn{\bm{\eta}^\pi,\bPi_{\bphi, K}^{\pi}\bm{\eta}^\pi},
\end{align*} 
\end{proposition}
The proof can be found in Appendix~\ref{sub:proof_fine-grained}. This result matches the instance-dependent upper bound on the approximation error established in \citet{doi:10.1287/moor.2022.1341} in the context of linear TD learning. We also derive an analogous result for the TD fixed point of tabular categorical TD learning, which is deferred to Appendix~\ref{appendix:CTD}.

\subsection{Related Works}
In this subsection, we provide a supplementary review of the related works.

\paragraph{Distributional Temporal Difference learning}
Distributional temporal difference (TD) learning was first proposed in \citet{bellemare2017distributional}. As introduced in Section~\ref{Section:intro}, several works have analyzed theoretical properties of distributional TD learning in the tabular setting. 
Among them, \citet{rowland2024nearminimaxoptimal,peng2024statistical} focused on categorical distributional TD learning in Cram\'er loss and established that, in the tabular setting, learning the full return distribution is statistically as easy as learning its expectation in the model-based and model-free settings. 
Recently \citet{tang2024off,pmlr-v202-wu23s} have extend distributional TD learning to the offline observation model setting.

To scale distributional TD learning to large or continuous state spaces, the introduction of function approximation is necessary. 
\citet{bellemare2019distributional,pmlr-v97-qu19b,lyle2019comparative,bdr2022} proposed various distributional TD learning algorithms with linear function approximation under different parameterizations, namely categorical and quantile parameterizations. \citet{wiltzer2024foundations} studied distributional policy evaluation in the multivariate reward setting and proposed the corresponding categorical distribution TD learning algorithms. \citet{kastner2025categorical} consider categorical TD learning with a KL divergence loss and derive asymptotic variance of the categorical estimates under different learning rate regimes.

\paragraph{Stochastic Optimization with Variance-reduction}
Developing variance-reduction techniques in stochastic optimization has been an active area in the past decades. 
Early works including IAG \cite{doi:10.1137/20M1381678}, SAG \cite{schmidt2017minimizing}, SVRG \cite{zhang2025estimation}, and
SAGA \cite{NIPS2014_93796419} established several first-order methods for variance reductions. In the family of TD learning algorithms, a line of works employing variance-reduction methods to achieve better algorithm performance \cite{bhandari2018finite,khamaru2021temporal,li2023accelerated}. 
Among them, \cite{li2023accelerated} directly inspires our algorithm design, which combines variance reduction approaches with the operator extrapolation method proposed by \cite{khamaru2021temporal}.

\begin{table}[!t]
\begin{center}
\caption{Sample complexity results of policy evaluation and DPE algorithms with linear function approximation}
\label{table:sample_complexity}
\begin{tabular}{|c|c|c|}
\hline 
Paper & Sample Complexity & Method\\
\hline 
\citet{samsonov2024improved} (Theorem 3) & $\widetilde{\mathcal{O}}\left(\frac{\|\bpsi^\star\|^2_{\bSigma_{\bphi}}+1}{(1-\gamma)^2\lambda_{\min}}(\frac{1}{\varepsilon^2}+\frac{1}{\lambda_{\min}})\right)$ & {\LTD}\\
\hline 
\citet{li2023accelerated} (Theorem 2)& $\widetilde{\mathcal{O}}\left(\frac{\|\bpsi^\star\|^2_{\bSigma_{\bphi}}+1}{\varepsilon^2(1-\gamma)^2\lambda_{\min}}\right)$ & {\LTD} with Variance Reduction\\
\hline 
\citet{bdr2022} (Section 9.7)& Contraction Analysis & Vanilla {\LCTD}\\
\hline 
\citet{peng2025finitesampleanalysisdistributional} (Theorem E.2)& $\widetilde{\mathcal{O}}\left(\frac{K^4(\|\btheta^\star\|^2_{V_1}+1)}{(1-\gamma)^2\lambda_{\min}}(\frac{1}{\varepsilon^2}+\frac{K^2}{\lambda_{\min}})\right)$ & Vanilla {\LCTD}\\
\hline 
\citet{peng2025finitesampleanalysisdistributional} (Theorem 4.1)& $\widetilde{\mathcal{O}}\left(\frac{\|\btheta^\star\|^2_{V_1}+1}{(1-\gamma)^2\lambda_{\min}}(\frac{1}{\varepsilon^2}+\frac{1}{\lambda_{\min}})\right)$ & {\LCTD}\\
\hline 
This work (Theorem~\ref{thm:VrFLCTD_convergence}) & $\widetilde{\mathcal{O}}\left(\frac{\|\btheta^\star\|^2_{V_1}+1}{\varepsilon^2(1-\gamma)^2\lambda_{\min}}\right)$ & {\VCTD}\\
\hline 
\end{tabular}
\end{center}
\end{table}

\paragraph{Comparisons of Theoretical Results with our Previous Work}
In Table~\ref{table:sample_complexity}, we summarize our main theoretical results and compare with previous work. 
In the table, for the task of policy evaluation, the sample complexity is defined via the $\mu_\pi$-weighted $L^2$ norm error metric for value functions; for the task of distributional policy evaluation, the sample complexity is defined via the $\mu_\pi$-weighted $W_1$ error metric for return distributions. 
Our result shares the same structure as the sample complexity of variance-reduced {\LTD} algorithm proposed in \citet{li2023accelerated}. 
Section 9.7 of \citet{bdr2022} provided a contraction analysis for the dynamic programming version of the Vanilla {\LCTD} algorithm. 

\paragraph{Comparisons of Theoretical Results with the Conference Version}
In Table~\ref{table:sample_complexity}, we also list the theoretical results in the conference version of this paper \citep{peng2025finitesampleanalysisdistributional}. 
The conference version of this paper provides a finite-sample analysis of the vanilla algorithm and removes the dependence on $K$ in the analysis of the new proposed algorithm {\LCTD}. Preserving {\LCTD}'s property of $K$-independent sample complexity, our improvement upon \citet{peng2025finitesampleanalysisdistributional} proceeds in two steps. 
First, our instance-dependent analysis of {\LCTD} recovers the convergence result in Theorem 4.1 of \citet{peng2025finitesampleanalysisdistributional} and reveals its connection to the minimax lower bound of the sample complexity. 
Furthermore, as shown in Table~\ref{table:sample_complexity}, the new proposed algorithm {\VCTD} removes the additional $\lambda_{\min}^{-1}$ term and breaks the sample size barrier of {\LCTD}, while inheriting the merit of $K$-independence.
\section{Linear-Categorical Temporal Difference Learning}\label{Section:linear_ctd}
\noindent In this section, we recap the theoretical foundations of {\LCTD} and introduce preliminary convergence results.

\subsection{The Linear System of Linear-Categorical TD learning}
The following proposition provides an analytical expression of the solution of the linear-categorical projected Bellman equation $\bm{\eta}_{\btheta}{=}\bPi_{\bphi, K}^{\pi}\gT^{\pi}\bm{\eta}_{\btheta}$ as a linear system. 
Here we work with the matrix version of $\bTheta{:=}[\btheta(0), {\ldots},\btheta(K{-}1)]{\in}\RB^{d{\times}K}$.
\begin{proposition}\label{thm:linear_cate_TD_equation}
The linear-categorical projected Bellman equation  $\bm{\eta}_{\btheta}{=}\bPi_{\bphi, K}^{\pi}\gT^{\pi}\bm{\eta}_{\btheta}$ admits a unique solution $\bm{\eta}_{\btheta^{\star}}$, where the matrix parameter $\bTheta^{\star}$ is the unique solution to the linear system for $\bTheta{\in}\RB^{d{\times}K}$
\begin{equation}\label{eq:fixed_point_equation}
\bSigma_{\bphi}\bTheta{-}\EB_{s,s^\prime,r}\!\!\brk{\bphi(s)\bphi(s^\prime)^{\top}\!\bTheta(\bC\tilde{\bG}(r)\bC^{{-}1})^{\top}\!}\!{=}\textstyle\frac{1}{K{+}1}\EB_{s,r}\!\!\brk{\bphi(s)(\textstyle\sum_{j=0}^K\!\bg_j(r){-}\bm{1}_{K})^{\top}\!\bC^{\top}\!},
\end{equation}
where for any $r\in[0,1]$ and $j,k\in\{0,1, \ldots, K\}$,
\begin{equation*}\label{eq:def_C}
    \bC {=} \brk{\ind\brc{i\geq j}}_{i,j\in[K]}\in\RB^{K{\times} K}
\end{equation*}

\begin{equation*}
     g_{j,k}(r):=\prn{1-\abs{(r+\gamma x_j-x_k)/{\iota_K}}}_+, \quad \bg_j(r):=\prn{g_{j,k}(r)}_{k=0}^{K-1}\in\RB^{K},
\end{equation*}

\begin{equation*}
    \bG(r):=
    \begin{bmatrix}
    \bg_0(r),  \ldots,  \bg_{K-1}(r)
    \end{bmatrix}\in\RB^{K{\times}K},\quad \tilde{\bG}(r):=\bG(r)-\bm{1}_K^{\top}\otimes\bg_K(r)\in\RB^{K{\times} K}.
\end{equation*}
\end{proposition}
The proof can be found in Appendix C in the conference version \citep{peng2025finitesampleanalysisdistributional}. 
Consequently, one can solve Equation~\eqref{eq:fixed_point_equation} through linear stochastic approximation (LSA) given the streaming data $\{(s_t,a_t,s_t^{\prime},r_t\})_{t=0}^\infty$:
\begin{equation}\label{eq:linear_CTD}
\begin{aligned}
\text{{\LCTD}: }\qquad\bTheta_t{\gets}&\bTheta_{t{-}1}{-}\alpha\bphi(s_t)\Big[\bphi(s_t)^{\top}\bTheta_{t{-}1}{-}\bphi(s_{t^{\prime}})^{\top}\bTheta_{t{-}1}(\bC\tilde{\bG}(r_t)\bC^{-1})^{\top}\\
&{-}\prn{K{+}1}^{-1}(\textstyle\sum_{j=0}^K\bg_j(r_t){-}\bm{1}_{K})^{\top}\bC^{\top}\Big],
\end{aligned}
\end{equation}
for any $t\geq 1$, where $\alpha$ is a constant step size. 
The update rule in Equation~\eqref{eq:linear_CTD} can be written in vectorized form using Kronecker's product as:
\begin{equation}\label{eq:linear_CTD_vec}
\begin{aligned}
    \text{{\LCTD}(vec): }&\btheta_t{\gets}\btheta_{t-1}{-}\alpha\prn{\bA_{t}\btheta_{t-1}{-}\bb_t}, \\
    & \bA_t{=}\brk{\bI_K{\otimes}\prn{\bphi(s_t)\bphi(s_t)^{\top}}}{-}[(\bC\tilde{\bG}(r_t)\bC^{-1}){\otimes}\prn{\bphi(s_t)\bphi(s^\prime_t)^{\top}}], \\
    & \bb_t=(K+1)^{-1}[\bC(\textstyle\sum_{j=0}^K\bg_j(r_t)-\bm{1}_K)]\otimes\bphi(s_t).
\end{aligned}
\end{equation}
This vectorized formulation of the algorithm is convenient for theoretical analysis. In the following, unless otherwise specified, we use {\LCTD} for the vectorized form of the algorithm, since these two formulations are equivalent. {\LCTD} can be seen as a counterpart to {\LTD} in \cite{samsonov2024improved}, which estimates the solution $\bV_{\bpsi^\star}$ of the linear projected Bellman equation through LSA.

\subsection{Convergence Results and Sample Size Barrier}\label{Subsection:L2_convergence}
In this subsection, we introduce the convergence results for {\LCTD} in the generative model setting, which we establish in \citet{peng2025finitesampleanalysisdistributional}. 
In the following, we have two choices of loss functions on the CDF, as appropriate for the specific use: $\mu_\pi$-weighted $L^2$-distance $\gL(\btheta) = \ell_{2,\mu_\pi}(\bm{\eta}_{\btheta},\bm{\eta}_{\btheta^{\star}})$ and $\mu_\pi$-weighted 1-Wasserstein distance $\gW(\btheta) = W_{1,\mu_{\pi}}(\bm{\eta}_{\btheta},\bm{\eta}_{\btheta^{\star}})$.

We present the non-asymptotic convergence rates for the estimation error of {\LCTD}. Specifically, the estimator $\bar{\btheta}_{T}$ is given by the Polyak-Ruppert tail average of the updates of {\LCTD} in Equation~\eqref{eq:linear_CTD}, that is, $\bar{\btheta}_{T}:=(T{/}2{+}1)^{-1}\sum_{t{=}T{/}2}^T\btheta_t$. 
The following proposition formally states the convergence performance of {\LCTD} in the $\mu_{\pi}$-weighted $1$-Wasserstein metric. 
This choice of loss function aligns with our previous analysis in \citet{peng2025finitesampleanalysisdistributional}.
\begin{proposition}\label{thm:l2_error_linear_ctd}
Let $\bJ = \bI_K\otimes\bSigma_{\bphi}$.
For any $K\geq (1-\gamma)^{-1}$ and $\alpha\in(0,(1-\sqrt\gamma)/256)$, it holds that
\begin{align*}
\EB^{1/2}[(\gW(\bar\btheta_T))^2]
\lesssim&\frac{\norm{\btheta^{\star}}_{V_1}+1}{\sqrt{T}(1-\gamma)\sqrt{\lambda_{\min}}}\prn{1+\sqrt{\frac{\alpha}{(1-\gamma)\lambda_{\min}}}}+\frac{\norm{\btheta^{\star}}_{V_1}+1}{T\sqrt{\alpha }(1-\gamma)^{3/2}\lambda_{\min}}\\
&+\frac{(1-\frac{1}{2}\alpha (1-\sqrt\gamma)\lambda_{\min} )^{T/2}}{T\alpha(1-\gamma)\lambda_{\min}}\prn{1+\sqrt{\frac{\alpha}{(1-\gamma)\lambda_{\min}}}}\norm{\btheta_0-\btheta^{\star}}_{V_2},
\end{align*}
where $\norm{\btheta^{\star}}_{V_1}:=\frac{1}{\sqrt{K}(1-\gamma)}\norm{\btheta^{\star}}_{\bJ}$ and $\norm{\btheta_0-\btheta^{\star}}_{V_2}:=\frac{1}{\sqrt{K}(1-\gamma)}\norm{\btheta_0-\btheta^{\star}}$.
\end{proposition}

\begin{corollary}
\label{coro:l2_sample_complexity_linear_ctd}
For any $K\geq (1-\gamma)^{-1}$ and $\alpha\in(0,(1-\sqrt\gamma)/76)$, to achieve $ \EB^{1/2}[(\gW(\bar\btheta_T))^2]\leq \varepsilon$,
\begin{equation*}
    \wtilde{\gO}\prn{\frac{\norm{\btheta^{\star}}^2_{V_1}+1}{\varepsilon^2(1-\gamma)^2\lambda_{\min}}\prn{1+\frac{\alpha}{(1-\gamma)\lambda_{\min}}} +\frac{\norm{\btheta^{\star}}_{V_1}+1}{\varepsilon\sqrt{\alpha }(1-\gamma)^{3/2}\lambda_{\min}}}
\end{equation*}
samples are sufficient. 
\end{corollary}
The proof of Theorem~\ref{thm:l2_error_linear_ctd} is based on an exponential stability analysis of the underlying linear system. For the proof outlines, refer to Section 5 in the conference version \citep{peng2025finitesampleanalysisdistributional} of this work.

With Corollary~\ref{coro:l2_sample_complexity_linear_ctd} at hand, we can analyze the sample complexity under different step size choices. 
Taking the largest possible instance-independent step size $\alpha \simeq (1-\gamma)$, we obtain the sample complexity bound for $\varepsilon \in (0, 1)$ as  
\begin{equation*}
\wtilde{\gO}\prn{\varepsilon^{-2}(1-\gamma)^{-2}\lambda_{\min}^{-2}\prn{\norm{\btheta^{\star}}^2_{V_1}+1}}.  
\end{equation*}  
Alternatively, with the optimal instance-dependent step size $\alpha \simeq (1-\gamma)\lambda_{\min}$, the bound becomes
\begin{equation}\label{eq:instance_dependent_step_size_l2_sample_complexity}  
    \wtilde{\gO}\prn{\prn{\varepsilon^{-2}+{\lambda_{\min}^{-1}}}(1-\gamma)^{-2}\lambda_{\min}^{-1}\prn{\norm{\btheta^{\star}}^2_{V_1}+1}}.  
\end{equation}  
This observation reveals a fundamental sample size barrier for {\LCTD}: achieving theoretically optimal scaling requires the target accuracy $\varepsilon$ to satisfy $\varepsilon = \tilde{\mathcal{O}}(\sqrt{\lambda_{\min}})$. 
As a result, the algorithm requires a computational budget in terms of $T$ that is independent of $\varepsilon$. 
This barrier has practical implications particularly when $\lambda_{\min}$ is small. 
In such cases, the sample complexity is dominated by the $\lambda_{\min}^{-1}$ term unless $T$ is sufficiently large. 
This issue is analogous to the sample size barrier of {\LTD} \citep{samsonov2024improved} and can be resolved by using variance reduction techniques.
\section{Instance-dependent Analysis of \LCTD}\label{Section:instance-dependent_analysis}
\noindent In this section, we provide an instance-dependent analysis of {\LCTD}.
In the following, the quality of the estimator is measured using the $\mu_{\pi}$-weighted $L^2$ distance loss function $\gL(\btheta)=\ell_{2,\mu_{\pi}}^2(\bm{\eta}_{\btheta},\bm{\eta}_{\btheta^{\star}})$. 
An upper bound in this $\mu_{\pi}$-weighted $L^2$ distance implies an upper bound in the $\mu_{\pi}$-weighted $W_1$ distance, as Lemma~\ref{lem:prob_basic_inequalities} guarantees that $W_1(\nu_1,\nu_2) \leq \ell_2(\nu_1,\nu_2)/\sqrt{1-\gamma}$ for any $\nu_1, \nu_2\in\sP^{\sgn}$. 
Moreover, the $\mu_{\pi}$-weighted $L^2$ distance loss admits an explicit formulation as follows.
\begin{equation*}
    \ell_{2,\mu_{\pi}}^2\prn{\bm{\eta}_{\btheta},\bm{\eta}_{\btheta^{\star}}}
    =\iota_K\norm{\btheta-\btheta^\star}^2_{\bI_K\otimes\bSigma_{\bphi}}.
\end{equation*}
This relation is due to the isometry of the affine space $(\sP^{\sgn}_{\bphi,K}, \ell_{2,\mu_{\pi}})$ with $(\RB^{dK}, \sqrt{\iota_K}\norm{\cdot}_{\bI_K\otimes\bSigma_{\bphi}})$, which we establish in Proposition D.1 of \citet{peng2025finitesampleanalysisdistributional}.

\subsection{Lower Bound}\label{subsection:analysis_of_lower_bounds}
Driven by the interest in the information-theoretic limits of estimating the target parameter $\btheta^{\star}$, this subsection establishes the instance-dependent, asymptotic local minimax lower bound for this estimation problem.
To begin with, we introduce a simplified generative observation model for the data stream $\{\xi_{t}\}_{t=0}^\infty=\{(s_t,s_t^{\prime},r_t)\}_{t=0}^\infty$. 
In the $t$-th iteration, samples are collected as follows: the current state $s_t$ is drawn from a distribution $\mu$ over the discrete state space (which is not necessarily the stationary distribution $\mu_\pi$), the next state $s_t'$ is sampled from the transition kernel $\gP(\cdot|s_t)$, and the reward $r_t$ is drawn from the conditional reward distribution $\gR(\cdot|s_t, s_t')$. For simplicity, rewards take values in a discrete reward space. 

Within this framework, a problem instance $\gI$ of is formally parameterized by the tuple $\gI = (\mu,\bP,\bR)$. 
Here we model $\gP$ as a transition matrix $\bP $ and $\gR$ as a three-dimensional tensor $\mathbf{R}$, where entry $\bR_{i,j,k}$ denotes the probability that the transition $i\rightarrow j$ results in a reward indexed $k$ in a reward space $S_r = \{r_{k}|k\in [n_r]\}$. Based on this representation, we define the $\epsilon$-neighborhood of an instance $\mathcal{I}$ as
\begin{equation*}
    \mathcal{N}(\mathcal{I},\epsilon) = \Big\{\mathcal{I}^{\prime}=(\mu^{\prime},\bP^{\prime},\bR^{\prime}):\norm{\mu^{\prime}-\mu} + \norm{\bP^{\prime}-\bP}_{F} + \norm{\bR^{\prime}-\bR}_{F}\leq \epsilon \Big\},
\end{equation*}
where $\|\cdot\|_{F}$ denotes the Frobenius norm. 
The solution to the linear system in Proposition~\ref{eq:linear_CTD} for $\gI$ is given by (written in vectorized form)
\begin{equation}\label{eq:instance_solution_of_linear_ctd}
    \btheta^{\star}(\mathcal{I}) =  \frac{1}{K+1}\brk{ \bJ - \EB\brk{ \prn{\bC\tilde{\bG}(r)\bC^{-1}} \otimes \bphi(s)\bphi(s')^\top } }^{-1}\EB \brk{  \prn{\bC \prn{ \sum_{j=0}^K \bg_j(r) - \bm{1}_K }}\otimes\bphi(s) }.
\end{equation}
Before heading to the results, we define some important statistical quantities for {\LCTD}. 
From the basic update rule of $\LCTD:\btheta_t{=}\btheta_{t-1}{-}\alpha\prn{\bA_{t}\btheta_{t-1}{-}\bb_t}$, we denote by $\bar{\bA} $ and $\bar{\bb}$ the expectations of $\bA_t$ and $\bb_t$ respectively. 
Define
\begin{equation*}
    \be_t=\prn{\bA_t-\bar{\bA}}\btheta^{\star}-\prn{\bb_t-\bar{\bb}} = \bA_t\btheta^{\star}-\bb_t.
\end{equation*}
Let $\bSigma_{\be}=\mathbb{E}[\be\be^{\top}]$ be the covariance of this stochastic operator and 
\begin{equation}\label{eq:trace_def}
    \widetilde{\bSigma}_{\be}:=\bJ^{1 / 2} \bar{\bA}^{-1} \bSigma_{\be} \bar{\bA}^{-\top} \bJ^{1 / 2}.
\end{equation}
The subscript $t$ is omitted here because $\be_{t}$ is i.i.d. in the generative model setting. 
Now we are ready to write the local asymptotic minimax risk as
\begin{equation*}
    \mathfrak{M}(\mathcal{I}):=\lim _{c \rightarrow \infty} \lim _{T \rightarrow \infty} \inf _{\hat{\btheta}_T} \sup _{\mathcal{I}^{\prime} \in \mathcal{N}(\mathcal{I}, c / \sqrt{T})} T \cdot \EB_{\mathcal{I}^{\prime}}    \brk{\ell_{2,\mu_{\pi}}^2\prn{\bm{\eta}_{\hat\btheta_T},\bm{\eta}_{\btheta^{\star}(\mathcal{I})}}}.
\end{equation*}
The following theorem characterizes the local asymptotic risk.
\begin{theorem}\label{thm:minimax_lower_bound}
Let $Z \in \mathbb{R}^d$ be a multivariate Gaussian random variable
\begin{equation*}
    Z \sim \mathcal{N}(0,\widetilde{\bSigma}_{\be}).
\end{equation*}
The local asymptotic minimax risk for the estimation of $\bm{\eta}_{\btheta^{\star}}$ is given by
\begin{equation*}
    \mathfrak{M}(\mathcal{I})=\iota_K \EB[\norm{Z}_{2}^2]=\frac{\tr{(\widetilde{\bSigma}_{\be})}}{K(1-\gamma)}.
\end{equation*}
\end{theorem}
 The proof of Theorem~\ref{thm:minimax_lower_bound} can be found in Appendix~\ref{appendix:proof_instance_dependent}. This result parallels the instance-specific lower bound on stochastic error in \citet{li2023accelerated}. 
 While incorporating the perturbation of the initial state distribution $\mu$ as in \citet{li2023accelerated}, our lower bound construction introduces an additional perturbation at the reward level. 
 To be concrete, we perturb the point mass probability of the categorical return distributions. 
 With this formulation, Theorem~\ref{thm:minimax_lower_bound} provides a minimax lower bound that shares a consistent structure with the corresponding result derived in \citet{li2023accelerated} for estimating the TD fixed point $\bV_{\bpsi^\star}$ in the context of linear TD learning.  
 Since \cite{li2023accelerated} further proposed accelerated algorithms and succeeded in derive sharp convergence rate with respect to the lower bound, this consistency strengthens the foundation for developing algorithms for learning $\btheta^{\star}$ that achieve an instance-optimal convergence rate.
\begin{remark}
The local asymptotic minimax risk of estimating $\bV_{\bpsi^\star}$ in \cite{li2023accelerated} has the form $\tr((\bI_{dK} - \widetilde{\bM})^{-1} \widetilde{\bSigma}(\bI_{dK} - \widetilde{\bM})^{-\top})$. 
This form appeared frequently in previous work for the minimax lower bound in the TD learning family. Our result derived in Theorem~\ref{thm:minimax_lower_bound} can also be written in this form as follows.
\begin{equation*}
    \mathfrak{M}(\mathcal{I}) = \tr((\bI_{dK} - \widetilde{\bM})^{-1} \widetilde{\bSigma}(\bI_{dK} - \widetilde{\bM})^{-\top}),
\end{equation*}
where $\widetilde{\bM} = \bJ^{-1/2}(\bar{\bA}-\bJ)\bJ^{-1/2}$, $\widetilde{\bSigma} = \bJ^{-1/2}\bSigma_{\be}\bJ^{-1/2}$. 
To ensure notational consistency with the conference version \citep{peng2025finitesampleanalysisdistributional} and simplify the presentation of subsequent theoretical results, we adhere to the notations we introduced in this work.
\end{remark}

\subsection{Upper Bound}
In this subsection, we provide an instance-dependent upper bound for the baseline algorithm {\LCTD}, whose dominant term matches the minimax lower bound in the previous subsection. 
The following theorem relates the non-asymptotic convergence analysis of the statistical estimator $\bar{\btheta}_T$ to the trace term $\tr(\widetilde{\bSigma}_{\be})$ we define in Equation~\eqref{eq:trace_def}.
\begin{theorem}\label{thm:leading_term}
For any $K\geq (1-\gamma)^{-1}$ and $ \alpha\leq \alpha_{2,\infty}:=(1-\sqrt\gamma)/256$, it holds that
\begin{align*}
    \EB^{1/2}[(\gL(\bar\btheta_T))^2] 
    &\lesssim \frac{1}{\sqrt{T}} \sqrt{\frac{\tr(\widetilde{\bSigma}_{\be})}{K(1-\gamma)}}
    +\frac{ \norm{\btheta^{\star}}_{V_1}+1}{\sqrt{T}(1-\gamma)\lambda_{\min}}\prn{\frac{1}{\sqrt{T\alpha}}+ \sqrt\alpha} \\
    &+\frac{(1-\frac{1}{2}\alpha (1-\sqrt\gamma)\lambda_{\min} )^{T/2}}{T\alpha\sqrt{(1-\gamma)\lambda_{\min}}}\prn{1+\sqrt{\frac{\alpha}{(1-\gamma)\lambda_{\min}}}}\norm{\btheta_0-\btheta^{\star}}_{V_2}.
\end{align*}
\end{theorem}
Clearly, this result can be translated into the following upper bound for the sample complexity of {\LCTD}. 
\begin{corollary}\label{coro:ind_l2_sample_complexity_linear_ctd}
Under the same conditions as in Theorem~\ref{thm:leading_term}, to achieve $ \EB^{1/2}[(\gL(\bar\btheta_T))^2]\leq \varepsilon$, 
\begin{align*}
\widetilde{\gO}\prn{\frac{1}{\varepsilon^2}\prn{\frac{\tr(\widetilde{\bSigma}_{\be})}{K(1-\gamma)}+\frac{\alpha\prn{\norm{\btheta^{\star}}^2_{V_1}+1}}{(1-\gamma)^2\lambda_{\min}^2}}+\frac{\norm{\btheta^{\star}}_{V_1}+1}{\varepsilon\sqrt{\alpha }(1-\gamma)\lambda_{\min}}}
\end{align*}
samples are sufficient.
\end{corollary}
In this instance-dependent estimation-error analysis for {\LCTD}, we follow \citet{samsonov2024improved} in the choice of the step size. 
To recover from Theorem~\ref{thm:leading_term} the corresponding result in Proposition~\ref{coro:l2_sample_complexity_linear_ctd}, it suffices that the following lemma is valid. 
\begin{lemma}\label{lem:trace_critical}
For $K\geq (1-\gamma)^{-1}$, $\tr(\widetilde{\bSigma}_{\be}) \lesssim  (\|\btheta^{\star}\|_{V_1}^2+1)K/\lambda_{\min}$.
\end{lemma}
The proofs of Lemma~\ref{lem:trace_critical} and Theorem~\ref{thm:leading_term} can be found in Appendix~\ref{appendix:proof_instance_dependent}. 
Lemma~\ref{lem:trace_critical} builds a connection between the critical trace term $\tr(\widetilde{\bSigma}_{\be})$ and the parameters that determine a problem instance. 
This connection enables the derivation of a K-independent convergence rate result with variance-reduced techniques, as presented in Section~\ref{Section:Algorithms}.

As we note at the beginning of this subsection, the dominant term in Corollary~\ref{coro:ind_l2_sample_complexity_linear_ctd} is consistent with the minimax lower bound established in Theorem~\ref{thm:minimax_lower_bound}. 
Nevertheless, this result does not break the sample size barrier, because the leading term truly dominates only when the sample size $T$ is large enough, analogous to the analysis of Theorem~\ref{thm:l2_error_linear_ctd}. 
We will resolve this in Section~\ref{Section:Algorithms} by introducing new algorithms, thereby closing the gap between the upper and lower bounds.

\section{Variance Reduced Linear-Categorical TD Learning Algorithms}\label{Section:Algorithms}
\noindent We now turn to the introduction of variance-reduced algorithms for learning $\btheta^{\star}$ in the generative model setting and the Markovian setting. 
For convenience, we define the deterministic temporal difference operator and stochastic temporal difference operator in Equation~\eqref{eq:linear_CTD_vec} as follows:
\begin{equation*}
    \bar{\bh}(\btheta) = \bar{\bA}\btheta-\bar{\bb},\quad \bh_{t}(\btheta) = \bA_{t}\btheta-\bb_{t},
\end{equation*} 
where $\bA_t$ and $\bb_{t}$ are given by Equation~\eqref{eq:linear_CTD_vec}. 
With these notations, the baseline algorithm {\LCTD} is given by $\btheta_{t+1} \leftarrow \btheta_t-\alpha \bh_t(\btheta_t)$. 
A preliminary variance-reduction method replaces $\bh_t$ with the empirical estimate of $\bar{\bh}$ at each update. 
More advanced techniques tighten estimation errors by controlling the bias of $\bh_t$ relative to $\bar{\bh}$ and incorporating further algorithm design. 

\subsection{Variance-reduced Fast Algorithm in the Generative Model Setting}
We propose variance-reduced fast linear-categorical TD learning algorithm ({\VCTD}) which adapts the variance-reduction techniques introduced by \cite{doi:10.1137/20M1381678} and \citet{li2023accelerated}. 
{\VCTD} is formally outlined in Algorithm~\ref{alg:VrFLCTD}. 

Central to this approach is a re-centering update strategy. 
At the beginning of the each $n$-th outer loop, {\VCTD} uses $l_{n}$ samples to compute an empirical stochastic operator $\widehat{\bh}$ as an estimation for $\bar{\bh}$ at $\widetilde{\btheta}$, where $\widetilde{\btheta}$ is the output of the previous epoch. 
The vector $\widehat{\bh}(\widetilde{\btheta})$ re-center the target of the updates in the inner loop. 
Following \citet{li2023accelerated}, the inner loop iteration of {\VCTD} is accelerated using the operator extrapolation technique introduced in \citet{doi:10.1137/20M1381678}. 
In each inner loop $t$, a mini-batch of size $m$ is used to calculate the average stochastic operator $\bar{\bh}_t$. 
Then a customized operator for $\btheta_t$ is calculated as $\bF_t(\btheta_t) = \bar{\bh}_{t}(\btheta_t)-\bar{\bh}_{t}(\widetilde{\btheta})+\widehat{\bh}(\widetilde{\btheta})$ for each minibatch. 
According to \citet{li2023accelerated}, incorporation of an operator extrapolation step for the linear-categorical temporal difference term $\bF_{t}(\btheta_t)+ \lambda(\bF_{t}(\btheta_t)-\bF_{t-1}(\btheta_{t-1}))$ in Equation~\eqref{eqn:extrapolat_update} is essential to achieve a tighter error bound on the deterministic error term in the convergence rate analysis. 
This property is preserved under the linear-categorical parameterization framework, and we keep the term “fast” in the algorithm’s name as a tribute to this technique. 
The convergence rate of {\VCTD} in the generative model setting is rigorously presented in the following theorem.

  \begin{algorithm}[!t]
    \caption{Variance-reduced fast linear-categorical TD learning}
    \begin{algorithmic}[1]
      \State Initialize $\btheta_{0} = \widehat{\btheta}_0$, $\alpha>0$, $\lambda = 1$, integers $\{l_{n}\}_{n=1}^{N}$.
      \For{$n = 1$ to $N$}
        \State Set $\btheta_0 =\btheta_1=\widetilde{\btheta}=\widehat{\btheta}_{n-1}$, collect $l_{n}$ i.i.d. samples and compute $\widehat{\bh}(\widetilde{\btheta})=1/l_{n} \sum_{i=1}^{l_{n}} \bh_{i}(\widetilde{\btheta})$.
        \For{$t = 1$ to $T$}
          \State Collect $m$ i.i.d. samples, compute $\bar{\bh}_{t} = 1/m\sum_{j=1}^m h_j$ with them, let $\bF_t(\btheta_t) = \bar{\bh}_{t}(\btheta_t)-\bar{\bh}_{t}(\widetilde{\btheta})+\widehat{\bh}(\widetilde{\btheta})$, set $\bF_0(\btheta_0) = \bF_1(\btheta_1)$ and update with 
          \begin{equation}\label{eqn:extrapolat_update}
        \btheta_{t+1}=\btheta_t-\alpha(\bF_{t}(\btheta_t)+ \lambda(\bF_{t}(\btheta_t)-\bF_{t-1}(\btheta_{t-1}))).             
          \end{equation}
        \EndFor
        \State Output of the outer loop:
    \begin{equation*} 
\widehat{\btheta}_n=\sum_{t=2}^{T+1} \btheta_t/T.
 \end{equation*}
      \EndFor
    \end{algorithmic}
    \label{alg:VrFLCTD}
  \end{algorithm}
\begin{theorem}\label{thm:VrFLCTD_convergence}
Consider the generative model with the initial state drawn from the distribution $\pi$. Let $K\geq (1-\gamma)^{-1}$. 
Fix the total number of outer loops $N$ and a positive integer $M$. 
Suppose that for each epoch $n \in[N]$, the parameters $\alpha,\lambda,T , m$ and $l_{n}$ satisfy
\begin{align*}
&\alpha \leq \frac{1}{4(1+\sqrt{\gamma})}, \quad \lambda=1, \quad T \geq \frac{80}{\lambda_{\min}(1-\sqrt{\gamma}) \alpha}, \quad m \geq \max\left\{1,\frac{480\alpha\varsigma^2}{1-\sqrt{\gamma}}\right\}\\ 
&\quad
\text { and } \quad l_{n} \geq\left\{\frac{18 \varsigma^2}{\lambda_{\min}(1-\sqrt{\gamma})^2},\left(\frac{3}{4}\right)^{N-n} M\right\}.
\end{align*}
Then for each $\delta>0$, we have
\begin{equation*}
\EB\norm{\widehat{\btheta}_N-\btheta^{\star}}_{\bJ}^2 \leq \frac{1}{2^N}\|\btheta_{0}-\btheta^{\star}\|_{\bJ}^2+ \frac{15}{M}\tr(\widetilde{\bSigma}_{\be}).
\end{equation*}
To achieve 
$\EB^{1/2}[(\gL(\widehat\btheta_N))^2] = \EB^{1/2}[\iota_K\|\widehat{\btheta}_N-\btheta^{\star}\|_{\bJ}^2]\leq \varepsilon$, $\sum_{n=1}^N(l_{n}+Tm)$ samples are needed, which is bounded by
\begin{equation}\label{eqn:sample_complexity_VrFLCTD_tr_version}
    \mathcal{O}\prn{\frac{\iota_K\tr(\widetilde{\bSigma}_{\be})}{\varepsilon^2}+\frac{1}{\lambda_{\min}(1-\gamma)}\log\prn{\iota_K\frac{\norm{\btheta_0-\btheta^{\star}}_{\bJ}^2}{\varepsilon^2}}+ \frac{\varsigma^2}{\lambda_{\min}(1-\gamma)^2}\log\prn{\iota_K\frac{\norm{\btheta_0-\btheta^{\star}}_{\bJ}^2}{\varepsilon^2}}}.
\end{equation}
Using Lemma~\ref{lem:trace_critical}, this sample complexity can be further upper bounded as
\begin{equation}\label{eqn:sample_complexity_VrFLCTD}
    \mathcal{O}\prn{\frac{\norm{\btheta^{\star}}^2_{V_1}+1}{\varepsilon^2(1-\gamma)\lambda_{\min}}+
    \frac{1}{\lambda_{\min}(1-\gamma)}\log\prn{\iota_K\frac{\norm{\btheta_0-\btheta^{\star}}_{\bJ}^2}{\varepsilon^2}}+ \frac{\varsigma^2}{\lambda_{\min}(1-\gamma)^2}\log\prn{\iota_K\frac{\norm{\btheta_0-\btheta^{\star}}_{\bJ}^2}{\varepsilon^2}}}.
\end{equation}
\end{theorem}
The proof can be found in Appendix~\ref{appendix:proof_VrFLCTD}.
We now turn to a brief discussion. 
Theorem~\ref{thm:VrFLCTD_convergence} is derived under the condition $K\geq (1-\gamma)^{-1}$. 
This condition first appears in the proof of Lemma 5.2 in \citet{peng2025finitesampleanalysisdistributional} where we establish upper bounds on the norm of $\bb_{t}$ in the vectorized formulation $\bA_{t}\btheta_{t}-\bb_{t}$. 
The step size $\alpha$ is a constant independent of the underlying MDP, in contrast to the baseline algorithm {\LCTD} where the best sample complexity is achieved by an instance-dependent step size. 
The step size $\alpha$ is also independent of $K$, which is a feature inherited from {\LCTD}. 

We present two versions of the sample complexity in Theorem~\ref{thm:VrFLCTD_convergence} to facilitate the subsequent analysis. 
The first term of Equation~\eqref{eqn:sample_complexity_VrFLCTD_tr_version} serves as the dominant term and matches the local asymptotic minimax risk $\tr(\widetilde{\bSigma}_{\be})/\prn{K(1-\gamma)}$ derived in Theorem~\ref{thm:minimax_lower_bound}. 
This term can be further upper bounded by the leading term of Equation~\eqref{eqn:sample_complexity_VrFLCTD} due to the upper bound of $\tr(\widetilde{\bSigma}_{\be})$ established in Lemma~\ref{lem:trace_critical}.

Next, we compare this result with previous work. Theorem~\ref{thm:VrFLCTD_convergence} can be seen as a counterpart result to the convergence result in \citep{li2023accelerated} for variance-reduced {\LTD} algorithm. 
The leading term shares a similar trace structure, and the logarithmic terms have the same dependence on $(1-\gamma)^{-1}$ in both the deterministic error term (without $\varsigma^2$) and the stochastic error term (with $\varsigma^2$). 
To compare with the result we establish for the baseline algorithm {\LCTD} in \cite{peng2025finitesampleanalysisdistributional} where the loss metric is $\mu_{\pi}$-weighted $1$-Wasserstein distance, we translate Equation~\eqref{eqn:sample_complexity_VrFLCTD} into the following proposition.

\begin{proposition}[No Sample Size Barrier]
With the same setting as in Theorem~\ref{thm:VrFLCTD_convergence}, the sample complexity for a $\varepsilon$-accurate estimator in the $\mu_{\pi}$-weighted 1-Wasserstein metric is
\begin{equation*}\label{eq:Vr_sample_complexity}  
    \wtilde{\gO}\prn{\varepsilon^{-2}(1-\gamma)^{-2}\lambda_{\min}^{-1}\prn{\norm{\btheta^{\star}}^2_{V_1}+1}}.  
\end{equation*}
\end{proposition}
Here, we omit the logarithmic terms and keep the dominant term. 
The sample size barrier no longer exists in this result since there is no redundant dependency on $\lambda_{\min}$ as in Equation~\eqref{eq:instance_dependent_step_size_l2_sample_complexity}. 
This result aligns with the instance-optimal sample complexity of variance-reduced {\LTD} derived in \citet{li2023accelerated} given by $\wtilde{\gO}(\varepsilon^{-2}(1-\gamma)^{-2}\lambda_{\min}^{-1}(\|\bpsi^\star\|^2_{\bSigma_{\bphi}}+1))$, where $\bpsi^\star$ denotes the parameter for the solution $\bV_{\bpsi^\star}$ to the linear projected Bellman equation. 
To understand why these sample complexities are comparable, recall that the norm-inducing matrix $V_1$ is defined as $ \bI_K\otimes\bSigma_{\bphi}/(\sqrt{K}(1-\gamma))$. 
Through the dimensional analysis in Remark 3 of \citet{peng2025finitesampleanalysisdistributional}, we have shown that $\|\btheta^{\star}\|_{V_1}$ and $\|\bpsi^\star\|_{\bSigma_{\bphi}}$ are of the same order. 
This result highlights that learning the full return distribution is as statistically efficient as learning its expectation under linear-categorical parameterization when $K\geq (1-\gamma)^{-1}$.

\subsection{Markovian Settings}
We next investigate the Markovian setting. 
Recall that in the Markovian observation model, the streaming data $\{\widetilde{\xi_{t}}\}_{t=0}^\infty=\{(s_t,a_t,s_t^{\prime},r_t)\}_{t=0}^\infty$ are collected from the through $s_t\sim\mu_{\pi}(\cdot), a_t\sim\pi(\cdot|s_t), (r_t,s_{t+1})\sim \gP(\cdot,\cdot|s_t,a_t)$, and $s_{t}^{\prime}=s_{t+1}$. 
In this subsection, we extend {\VCTD} from the generative model setting to the Markovian setting. 
For this purpose, the following standard ergodicity assumption is required on the underlying Markov chain, which we adopt from Assumption 3 in \citet{li2023accelerated}.
\begin{assumption}\label{assum:mixing_time}
There exist constants $C_P>0$ and $\rho \in(0,1)$ such that
\begin{equation*}
\max _{s \in S}\norm{\mathbb{P}\prn{s_t=\cdot \mid s_0=s}-\pi}_{\infty} \leq C_P \cdot \rho^t \quad \text { for all } t \in \mathbb{Z}_{+}.
\end{equation*}

In other words, the mixing time is bounded by $t_{\text {mix }} \leq \frac{\log (4 C_P)}{\log (1 / \rho)}$, where
\begin{equation*}
t_{\text {mix}}:=\inf \left\{t \in \mathbb{Z}_{+} \colon \max _{s \in S}\norm{\mathbb{P}\prn{s_t=\cdot \mid s_0=s}-\pi}_{\infty} \leq 1 / 4\right\}.
\end{equation*}
\end{assumption}

\begin{algorithm}[!t]
    \caption{{\VCTD} adapted for Markovian setting}
    \begin{algorithmic}[1]
      \State Initialize $\btheta_{0} = \widehat{\btheta}_0$,  $\alpha>0$, $\lambda\geq0$, integers $m_0, m, l_{0},\{l_{n}\}_{n=1}^{N}$.
      \For{$n = 1$ to $N$}
        \State Set $\btheta_1=\widetilde{\btheta}=\widehat{\btheta}_{n-1}$. Collect $l_{n}$ samples and compute $\widehat{\bh}(\widetilde{\btheta})=\frac{1}{l_{n}-l_{0}} \sum_{i=l_{0}+1}^{l_{n}} \bh_{i}(\widetilde{\btheta})$.
        \For{$t = 1$ to $T$}
          \State Collect $m$ successive Markovian samples, compute $\bar{\bh}_{t} = 1/(m-m_0)\sum_{j=m_0+1}^m h_j$ with them, let $\bF_t(\btheta_t) = \bar{\bh}_{t}(\btheta_t)-\bar{\bh}_{t}(\widetilde{\btheta})+\widehat{\bh}(\widetilde{\btheta})$, set $\bF_0(\btheta_0) = \bF_1(\btheta_1)$ and update with 
          \begin{equation*}
        \btheta_{t+1}=\btheta_t-\alpha(\bF_{t}(\btheta_t)+\lambda(\bF_{t}(\btheta_t)-\bF_{t-1}(\btheta_{t-1}))).             
          \end{equation*}
        \EndFor
        \State Output of the outer loop:
\begin{equation*}     
\widehat{\btheta}_n=(\sum_{t=2}^{T+1} \btheta_t)/T.
\end{equation*}
    \EndFor
    \end{algorithmic}
    \label{alg:VrFLCTD_Mar}
\end{algorithm}
Now we are ready to formally present {\VCTD} algorithm in Markovian setting as detailed in Algorithm~\ref{alg:VrFLCTD_Mar}. 
The convergence properties of this algorithm are rigorously established by the following theorem.

\begin{theorem}\label{thm:Markovian_convergence}
Let $K\geq (1-\gamma)^{-1}$.
Fix the total number of outer loops $N$ and a positive integer $M$. 
Consider a positive integer $\tau$ that satisfies $\rho^\tau \leq \min \{\frac{2(1-\rho) \varsigma}{3 C_{\bphi}}, \frac{(1-\rho)^2}{2 C_{\bphi}}\}$, where $C_{\bphi}:=\frac{C_P\varsigma}{\sqrt{\min _{i \in[D]} \pi_i}}$. 
Suppose that the parameters $l_{0}$ and $m_0$ satisfy
\begin{equation*}
\rho^{l_{0}} \leq \frac{\min _{i \in[D]} \pi_i}{C_P}, \text { and } \quad \rho^{m_0} \leq \min \left\{\frac{\min _{i \in[D]} \pi_i}{C_P}, \frac{\sqrt{\lambda_{\min}} \alpha \tau \varsigma^2(1-\rho)}{C_{\bphi}}\right\}.
\end{equation*}
Further suppose that for each epoch $n \in[N]$, the parameters $\alpha, \lambda, T, m, l_{n}$ satisfy
\begin{align*}
&\alpha \leq \frac{1}{4 (1+\sqrt{\gamma})}, \quad \lambda=1, \quad T \geq \frac{90}{\lambda_{\min}(1-\sqrt\gamma) \alpha}, \quad m-m_0 \geq \max \left\{1, \frac{1000\alpha(\tau+1) \varsigma^2}{1-\sqrt\gamma}\right\}, \\
& \rho^{l_{n}-l_{0}} \leq \frac{\tau(1-\rho)}{4 C_{\bphi}(l_{n}-l_{0})}\quad \text { and } \quad l_{n}-l_{0} \geq\max\left\{\frac{72(\tau+1) \varsigma^2}{\lambda_{\min}(1-\sqrt\gamma)^2},\prn{\frac{3}{4}}^{N-n} M\right\}.
\end{align*}
To achieve 
$\EB^{1/2}[(\gL(\widehat\btheta_N))^2] = \EB^{1/2}[\iota_K\|\widehat{\btheta}_N-\btheta^{\star}\|_{\bJ}^2]\leq \varepsilon$, $\sum_{n=1}^N(l_{n}+Tm)$ samples are needed. 
Let $\bm{\eta}^{\pi,K}$ be the solution to the categorical projected Bellman equation $\bm{\eta}=\bPi_{K}\gT^{\pi}\bm{\eta}$. 
The following results hold.
\begin{itemize}
\item 
If $\bm{\eta}_{\btheta^{\star}} = \bm{\eta}^{\pi,K}$, 
\begin{equation*}
\EB\left[\norm{\widehat{\btheta}_N-\btheta^{\star}}_{\bJ}^2\right] \leq \frac{1}{2^N}\left[\|\btheta_{0}-\btheta^{\star}\|_{\bJ}^2\right]+\frac{27\tr(\widetilde{\bSigma}_{\be})}{M}.
\end{equation*}
The corresponding sample complexity is upper bounded by
 \begin{equation*}
\mathcal{O}\prn{\frac{\iota_K\tr(\widetilde{\bSigma}_{\be})}{\varepsilon^2}
     +\frac{t_{\text{mix}}}{\lambda_{\min}(1-\gamma)}\log\prn{\iota_K\frac{\norm{\btheta_0-\btheta^{\star}}_{\bJ}^2}{\varepsilon^2}}
     +\frac{t_{\text{mix}}\varsigma^2}{\lambda_{\min}(1-\gamma)^2}\log\prn{\iota_K\frac{\norm{\btheta_0-\btheta^{\star}}_{\bJ}^2}{\varepsilon^2}}}.
 \end{equation*}
\item
If $\bm{\eta}_{\btheta^{\star}} \neq \bm{\eta}^{\pi,K}$, with $\widetilde{\bSigma}_{\be,\text{Mkv}}:=\bJ^{1/2}\bar{\bA}^{-1} [\sum_{-\infty}^{\infty}\EB[\widetilde{\be}_{t}\widetilde{\be}_{0}^{\top}]] \bar{\bA}^{-\top}\bJ^{1/2}$ where $\widetilde{\be}_{t}=\bh_{t}(\btheta^{\star})-\bar{\bh}(\btheta^{\star})$,
\begin{align*}
\EB\brk{\norm{\widehat{\btheta}_N-\btheta^{\star}}_{\bJ}^2}
&\leq\frac{1}{2^N}\EB\brk{\|\btheta^{0}-\btheta^{\star}\|_{\bJ}^2}+
\frac{27\tr(\widetilde{\bSigma}_{\be,\text{Mkv}})}{M}\\
&\quad+\frac{81\tau}{M^2\lambda_{\min}(1-\sqrt{\gamma})^2} \frac{\ell_{2,\mu_\pi}^2\prn{\bm{\eta}^{\pi,K},\bm{\eta}_{\btheta^{\star}}}}{\iota_K}+\frac{90\tau^2\tr(\widetilde{\bSigma}_{\be})}{M^2}.
\end{align*}
The corresponding sample complexity is upper bounded by
\begin{align*}
\mathcal{O}\Bigg(\frac{\iota_K \tr(\widetilde{\bSigma}_{\be,\text{Mkv}})}{\varepsilon^2}
&+\frac{t_{\text{mix}}\sqrt{\iota_K\tr(\widetilde{\bSigma}_{\be})}}{\varepsilon}
+\frac{\sqrt{t_{\text{mix}}\ell_{2,\mu_\pi}^2\prn{\bm{\eta}^{\pi,K},\bm{\eta}_{\btheta^{\star}}}}}{\varepsilon\sqrt{\lambda_{\min}}(1-\gamma)}\\
&\quad+\frac{t_{\text{mix}}}{\lambda_{\min}(1-\gamma)}\log\prn{\iota_K\frac{\norm{\btheta_0-\btheta^{\star}}_{\bJ}^2}{\varepsilon^2}}
+\frac{t_{\text{mix}}\varsigma^2}{\lambda_{\min}(1-\gamma)^2}\log\prn{\iota_K\frac{\norm{\btheta_0-\btheta^{\star}}_{\bJ}^2}{\varepsilon^2}}\Bigg).
\end{align*}
\end{itemize}
\end{theorem}
The proof of Theorem~\ref{thm:Markovian_convergence} can be found in Appendix~\ref{appendix:proof_Markovian}. 
To show the impact of linear function approximation (following \citet{li2023accelerated}), we distinguish between cases based on whether the equality $\bm{\eta}_{\btheta^{\star}} = \bm{\eta}^{\pi,K}$ holds. 
Note that $\bm{\eta}^{\pi,K}$, the target of tabular categorical TD learning, can be intuitively interpreted as a special case of $\bm{\eta}_{\btheta^{\star}}$ when the feature dimension equals the size of the state space $\gS$ and all feature vectors are linearly independent. 
A formal introduction of tabular categorical TD learning and $\bm{\eta}^{\pi,K}$ can be found in Appendix~\ref{appendix:CTD}. 
In the following, we provide a discussion of these sample complexity upper-bounds.
 
We begin by comparing the sample complexity in Theorem~\ref{thm:Markovian_convergence} with the main result in the generative model setting. 
In the $\bm{\eta}_{\btheta^{\star}} = \bm{\eta}^{\pi,K}$ case where the linear approximation is accurate, the dominant trace term remains $\tr(\widetilde{\bSigma}_{\be})$ and the mixing time $t_{\text{mix}}$ only appears linearly in the logarithmic terms of $\varepsilon^{-1}$. 
In the $\bm{\eta}_{\btheta^{\star}} \neq \bm{\eta}^{\pi,K}$ case, the dominant trace term is replaced by $\tr(\widetilde{\bSigma}_{\be,\text{Mkv}})$. 
This term has the same structure as the corresponding result for variance-reduced {\LTD} in the Markovian setting, as provided by Theorem 4 in \citet{li2023accelerated}, and this term match the lower bound given by \cite{mou2024optimal}.
Furthermore, the logarithmic terms, while being of the same order with respect to $(1-\gamma)^{-1}$ and $\lambda_{\min}^{-1}$ in both cases, match those in Theorem 4 in \citet{li2023accelerated} as well.

A notable difference in the result manifests in the $\bm{\eta}_{\btheta^{\star}} \neq \bm{\eta}^{\pi,K}$ case, where the estimation error upper-bound contains terms of order $\varepsilon^{-1}$. 
These terms are the counterpart to the term $\sqrt{\gH/\varepsilon}$ and the approximation error term in Theorem 4 in \citet{li2023accelerated}. 
The presence of these terms implies that, under the current analytical framework, the sample size barrier re-emerges in the Markovian setting. 
Resolving this issue is beyond the scope of this work.
 
We now explain how this sample complexity upper bound is valid for different values of $K$. 
As stated before, this property is crucial for {\VCTD} to work effectively with arbitrarily large $K$. 
Since the trace term $\tr(\widetilde{\bSigma}_{\be})$ appearing in Theorem~\ref{thm:VrFLCTD_convergence} does not introduce extra dependence on $K$ (see the discussion following Theorem~\ref{thm:VrFLCTD_convergence}), we only need to deal with the linear approximation error term $\ell_{2,\mu_\pi}^2(\bm{\eta}^{\pi,K}, \bm{\eta}_{\btheta^{\star}})$ and the Markovian-tailored term $\iota_K\tr(\widetilde{\bSigma}_{\be,\text{Mkv}})$.
\begin{itemize}
\item 
$\ell_{2,\mu_\pi}^2(\bm{\eta}^{\pi,K},\bm{\eta}_{\btheta^{\star}})$ does not introduce a worse dependence on $K$, because Lemma~\ref{lem:appro_linear} writes
\begin{equation*}
\ell_{2,\mu_\pi}^2\prn{\bm{\eta}^{\pi,K},\bm{\eta}_{\btheta^{\star}}}\lesssim \frac{1}{(1-\gamma)\lambda_{\min}}.
\end{equation*}
\item
For the Markovian-tailored trace term $\tr(\widetilde{\bSigma}_{\be,\text{Mkv}})$, an useful upper bound for this term can be established as an intermediate result in the proof of Theorem~\ref{thm:Markovian_convergence} as follows. 
\begin{equation*}
\tr(\widetilde{\bSigma}_{\be,\text{Mkv}}) \lesssim
\frac{\ell_{2,\mu_\pi}^2\prn{\bm{\eta}^{\pi,K},\bm{\eta}_{\btheta^{\star}}}}{\lambda_{\min}(1-\sqrt{\gamma})^2\iota_K} +\tau\tr(\widetilde{\bSigma}_{\be}).
\end{equation*}
\end{itemize}
More discussion of the theoretical implications of this upper bound is deferred to Remark~\ref{remark:tr_Markovian_critical}. 
This upper bound guarantees that the dependence of $\tr(\widetilde{\bSigma}_{\be,\text{Mkv}})$ on $K$ is also at most linear. 

Putting these together, we show that the sample complexity result can be further reduced to a K-independent upper bound. 
In other words, even in the more challenging Markovian setting, the support size $K$ can be increased to achieve the best possible approximation accuracy, at a computational time cost comparable to that of variance-reduced {\LTD} (provided that the space complexity is acceptable). 
This finding parallels the landmark finding established by \citet{rowland2024nearminimaxoptimal} and \citet{peng2024statistical} that learning the return distribution distance is statistically efficient as learning its mean in the tabular setting, and extends this conclusion in the linear-categorical parameterization case. 
Thus, for the important case of linear function approximation, we have affirmatively answered the open question we raise in Section~\ref{Section:intro}.
\section{Numerical Experiments}\label{Section:experiments}
\noindent In this section, we perform proposed algorithms on different DPE tasks in different observation settings and compare with the baseline algorithm. We report the negative logarithm of $\|\bar{\btheta}_i-\btheta^{\star}\|_{\bJ}^2/\| \bar{\btheta}_0-\btheta^{\star}\|_{\bJ}^2$ along outer loops, where $\bar{\btheta}_i$ refers to the algorithm output at sample size $i$. All of the experiments are conducted on a server with 4 NVIDIA RTX 4090 GPUs and Intel(R) Xeon(R) Gold 6132 CPU @ 2.60GHz.
\subsection{A simple four-state MDP case}

To empirically evaluate our algorithm {\VCTD} in a toy case, we consider a $4$-state MDP with $K = 20$ and $\gamma = 0.99$. 
We then fix a transition matrix $\bP$ to represent a target policy and execute our algorithm on the samples or trajectories generated by $\bP$ (depending on which observation model is adopted). 
To work with a finite state space $\gS$, we denote by $\bPhi=(\bphi(s))_{s\in\gS}\in\RB^{d\times \gS}$ the feature matrix.
Here, we set the feature matrix $\bPhi$ as a matrix in $\mathbb{R}^{4\times3}$. Note that this matrix is not full-rank, which makes learning $\bm{\eta}_{\btheta^{\star}}$ in this setting strictly different from tabular categorical TD learning introduced in Appendix~\ref{appendix:CTD}.
The following experiments share the same setup: zero initialization $\btheta_0=\bm{0}$, step size $\alpha = 0.01$, and a maximum of 50,000 outer loops. The parameters are set as follows: $N_n=400$, $n_0=m_0=8$, $m=16$, and $T=100$.
\begin{figure}[ht]
    \centering
    \includegraphics[width=0.7\linewidth]{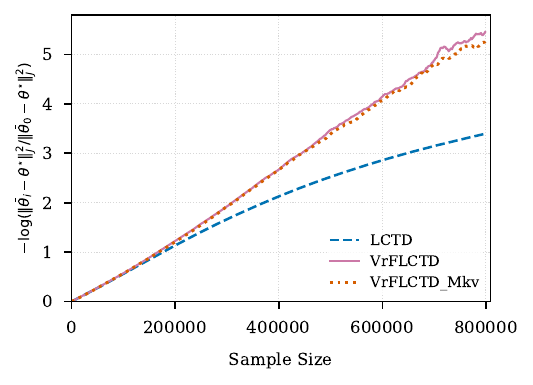} 
    \caption{Convergence results in four-state MDP with different algorithms and observation settings. For this 4-state MDP with $K = 20$ and $\gamma = 0.99$, {\VCTD} shows better convergence performance than the baseline algorithm \LCTD ~in the long run. }
    \label{fig:hard_compare}
\end{figure}

Figure~\ref{fig:hard_compare} shows the convergence results of different algorithms with the same total sample size. 
In the initial phase of the iterations, all algorithms exhibit a similar convergence rate. 
After about 10,000 outer loops (corresponding to about 160,000 samples in our setting), the variance-reduced algorithms maintain their pace, whereas the baseline algorithm {\LCTD} slows down. 
In the long run, {\VCTD} converges slightly slower on the Markovian data stream than on i.i.d. data. 
This phenomenon matches our theoretical expectation because the number of effective samples used per iteration is smaller in the Markovian setting due to the burn-in period.

\subsection{The Markovian setting: 2D Grid World}
Following \citet{li2023accelerated}, we adopt the 2D Grid World environment experiments to validate our algorithm in the Markovian model setting. 
In the original formulation of 2D Grid World, an agent receives a positive reward upon reaching a predetermined goal state and incurs negative rewards when entering certain trap states. 
Here, we apply an affine transformation to the rewards to restrict the reward to be within $[0,1]$. 
The state space has dimension $D=20\times20$, within which we define a single goal state (with reward $ r = 1 $) and 30 trap states (each with reward $ r = 0 $). 
No reward is assigned to the remaining states. 
The transition kernel of the target policy is defined as follows: with probability 0.95, the agent moves in a direction toward the goal, and with probability 0.05, it moves in a random direction. 
For the representation of the state features, we construct a random feature matrix of dimension $ d = 50 $. 
\begin{figure}[ht]
    \centering
    \includegraphics[width=0.9\linewidth]{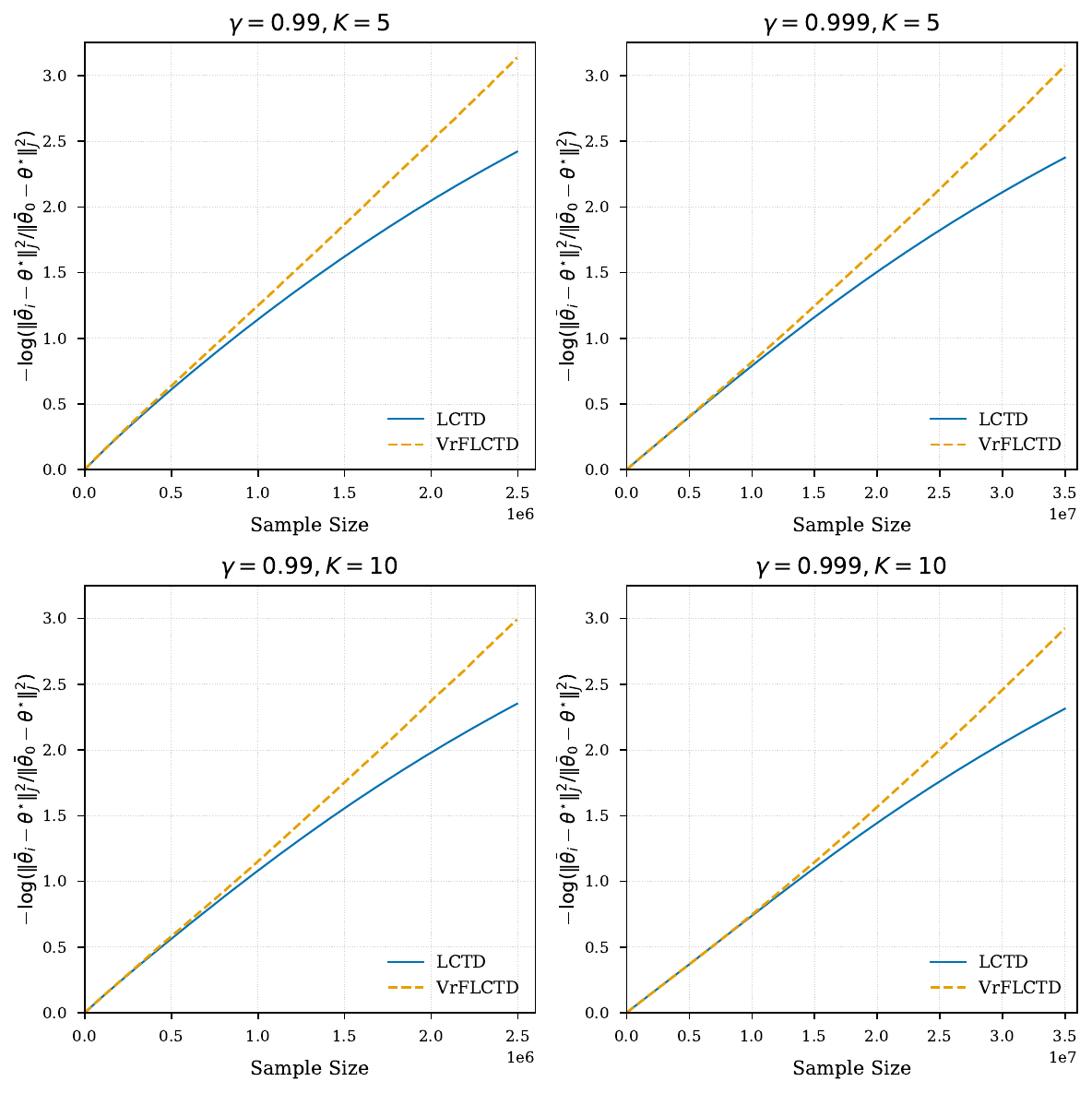} 
    \caption{Convergence results for 2D Grid World in the Markovian Setting for different $K$ and $\gamma$. While achieving similar convergence rate at early stage, {\VCTD} exhibits a notably more consistent and stable convergence rate in the long run. $\gamma$ closer to 1 leads to a slower convergence rate.}
    \label{fig:2D_grid}
\end{figure}
 The size of the support of the return distribution $K$ is set to 5 and 10. All experiments share zero initialization $\btheta_0=\bm{0}$ and a step size $\alpha$ of $0.01$. We provide the parameter setting below:
 \begin{itemize}
 \item 
For $\gamma=0.99$, $n_0=m_0=16$, $m=32$, $N_n=1800$ and $T=100$. 
\item 
For $\gamma=0.999$, $n_0=m_0=32$, $m=64$, $N_n=3000$ and $T=500$.
\end{itemize}

Figure~\ref{fig:2D_grid} presents the convergence results for the 2D Grid World environments, which lead to the following discussion. 
Across all experiments, {\VCTD} maintains a steady convergence rate while {\LCTD} slows down when the sample size grows larger, similarly to the simple MDP situation. 
Note that the hyperparameters are tuned to make their convergence rates comparable. 
We also observe that increasing $K$ while keeping the step size $\alpha$ fixed does not slow the algorithm down. 
This observation aligns with the theoretical analysis in Theorem~\ref{thm:Markovian_convergence}. 

Furthermore, a brief comparison can be made with respect to the number of samples required to reach the same level of convergence. Although the specific hyperparameters differ, the results suggest that as $\gamma$ approaches $1$, the sample complexity increases with the horizon length $(1-\gamma)^{-1}$. This empirical success of {\VCTD} in a relatively large state space with large $K$ and horizon supports our theoretical analysis and shows its potential to address more difficult DPE tasks.
\section{Concluding Remarks}\label{Section:discussion}
\noindent In this paper we have comprehensively explored the statistical efficiency of distributional TD learning with linear function approximation. 
By incorporating variance-reduction techniques into the baseline algorithm {\LCTD} and establishing important statistical bounds, we have proposed a new algorithm {\VCTD} with sharp theoretical guarantees in both the generative model setting and the Markovian setting. 
We have addressed the gap between the lower bound and the upper bound, breaking the sample size barrier of the baseline algorithm and demonstrating that, when employing distributional TD learning with linear function approximation, learning the full return distribution is as statistically efficient as learning its expectation (the value function). 
We have executed {\VCTD} on different DPE tasks and observed promising empirical success. Beyond the central scope of this work, we offer several open questions for further discussion. 
\begin{itemize}
    \item The basic framework of linear-categorical parameterization can lead to invalid return distributions $\bm{\eta}_{\btheta^{\star}}$. The root cause of these invalid instances remains unclear, and an effective solution has yet to be explored.
    An interim solution is provided in Appendix~\ref{appendix:projection}. 
    Since categorical distribution imposes a natural condition on the CDF induced by $\btheta$, this issue falls within constrained linear stochastic approximation. 
    While exploring this possibility falls outside the scope of this work, it may be interesting to design algorithms to elegantly rectify the outliers without introducing dependence on $K$.
    \item Beyond the linear-categorical parameterization we focus on, more formulations can be explored in the context of distributional TD learning with function approximation. Future work may investigate nonlinear parameterizations or different parameterizations, potentially unlocking new insights into the statistical efficiency of distributional reinforcement learning algorithms.
\end{itemize}

\appendix

\section{Technical Lemmas}\label{appendix:technical_lemmas}
\noindent Here, we list the technical lemmas used in this work. 
These lemmas establish crucial bounds on the statistics involved.
\begin{lemma}\label{lem:prob_basic_inequalities}
(Lemma G.1 in \citep{peng2025finitesampleanalysisdistributional})
For any $\nu_1, \nu_2\in\sP^{\sgn}$, we have $W_1(\nu_1,\nu_2)\leq \ell_2(\nu_1,\nu_2)/\sqrt{1-\gamma}$. 
\end{lemma}
\begin{lemma}\label{lem:Spectra_of_ccgcc}
(Lemma G.3 in \citep{peng2025finitesampleanalysisdistributional})
For any $r\in[0, 1]$, we have
\begin{equation*}
    \norm{\bC\tilde{\bG}(r)\bC^{-1}}\leq\sqrt{\gamma},\quad\norm{\prn{\bC^{\top}\bC}^{1/2}\tilde{\bG}(r)\prn{\bC^{\top}\bC}^{-1/2}}\leq \sqrt\gamma.
\end{equation*}
\end{lemma}
\begin{lemma}\label{lem:tr_cov_sigma_e}
(Lemma 5.2 in \citep{peng2025finitesampleanalysisdistributional})
For any $K\geq (1-\gamma)^{-1}$, we have
\begin{equation*}
    \tr\prn{\bSigma_{\be}}\lesssim \norm{\btheta^{\star}}_{\bJ}^2+K(1-\gamma)^2, \quad\norm{\bb}\lesssim \sqrt{K}(1-\gamma).
\end{equation*}
\end{lemma}

\begin{lemma}\label{lem:A_bilinear}
For any $K\geq (1-\gamma)^{-1}$, we have
\begin{equation*}(1-\sqrt{\gamma})\norm{\btheta}^2_{\bJ}\leq\langle \bar{\bA}\btheta,\btheta \rangle \leq (1+\sqrt{\gamma})\norm{\btheta}^2_{\bJ}.
\end{equation*}
\begin{proof}
Let $\bu = \bJ^{1/2}\btheta$ and note that
\begin{equation*}
\begin{aligned}
&\btheta^{\top}(\bar{\bA}-(1-\sqrt{\gamma})\bJ)\btheta=\btheta^{\top}((1+\sqrt{\gamma})\bJ)-\bar{\bA})\btheta\\
&=\btheta^{\top}(\sqrt{\gamma}\bJ-\EB_{s, r, s^\prime}\brk{\prn{\bC\tilde{\bG}(r)\bC^{-1}}\otimes\prn{\bphi(s)\bphi(s^\prime)^{\top}}})\btheta \geq 0\\
&\iff \sqrt{\gamma}\norm{\btheta}_{\bJ}^2 -\btheta^{\top}\EB_{s, r, s^\prime}\brk{\bY(r)\otimes\prn{\bSigma_{\bphi}^{-\frac{1}{2}}\bphi(s)\bphi(s^\prime)^{\top}\bSigma_{\bphi}^{-\frac{1}{2}}}}\btheta \geq 0\\
&\iff \bu^{\top}\EB_{s, r, s^\prime}\brk{\bY(r)\otimes\prn{\bSigma_{\bphi}^{-\frac{1}{2}}\bphi(s)\bphi(s^\prime)^{\top}\bSigma_{\bphi}^{-\frac{1}{2}}}}\bu \leq \sqrt{\gamma}\norm{\bu}^2.
\end{aligned}
\end{equation*}
The last inequality is an application of Equation 29 in \citet{peng2025finitesampleanalysisdistributional} upper bounding the norm of the expectation term.
\end{proof}
\end{lemma}

\begin{lemma}\label{lem:A^TA}
For any $K\geq (1-\gamma)^{-1}$, we have
\begin{equation*}
    \langle \bar{\bA}\btheta,\bar{\bA}\btheta \rangle \leq 2(1+\gamma)\norm{\btheta}^2_{\bJ}.
\end{equation*}
\begin{proof}
Note that Equation 32 in \citet{peng2025finitesampleanalysisdistributional} shows $\bA^{\top}\bA\preccurlyeq2\bI_K\otimes(\bphi(s)\bphi(s)^{\top}+\gamma\bphi(s^\prime)\bphi(s^\prime)^{\top})$. Thus, $ \langle \bar{\bA}\btheta,\bar{\bA}\btheta \rangle\leq\EB[\btheta^{\top}\bA^{\top}\bA\btheta]\leq 2(1+\gamma)\|\btheta\|_{\bJ}^2$.
\end{proof}
\end{lemma}

\section{Categorical TD learning}\label{appendix:CTD}
\noindent In this section, we describe the basics of categorical temporal difference learning and derive corresponding upper bound on the approximation error.
\subsection{Categorical Parameterization and Tabular Categorical TD Learning}
One way to deal with return distributions in a computationally tractable manner is applying the categorical parameterization as in \citet{bellemare2017distributional,rowland2018analysis,rowland2024nearminimaxoptimal,peng2024statistical} to the return distributions.
To be compatible with linear function approximation introduced in the next section, which cannot guarantee non-negative outputs, we will work with $\sP^{\sgn}$, the signed measure space with total mass $1$ as in \citet{bellemare2019distributional,lyle2019comparative,bdr2022} instead of standard probability space $\sP\subset \sP^{\sgn}$:
\begin{equation*}
    \sP^{\sgn}:= \left\{\nu\colon\nu(\RB)=1,\operatorname{supp}(\nu)\subseteq \left[0,(1-\gamma)^{-1}\right] \right\}.
\end{equation*} 
For any $\nu\in\sP^{\sgn}$, we define its cumulative distribution function (CDF) as $F_\nu(x):=\nu([0,x])$. 
We can naturally define the $L^2$ and $L^1$ distances between CDFs as the Cram\'er distance $\ell_2$ and $1$-Wasserstein distance $W_1$ in $\sP^{\sgn}$, respectively.
The distributional Bellman operator (see Equation~\eqref{eq:distributional_Bellman_equation}) can also be extended to the product space $(\sP^{\sgn})^{\gS}$ without modifying its definition.

The space of all categorical parameterized signed measures with total mass $1$ is defined as
\begin{equation*}
    \sP^{\sgn}_K := \{ \nu_\bp=\textstyle\sum_{k=0}^K p_k \delta_{x_k} \colon  \bp=\prn{p_0, \ldots, p_{K{-}1}}^{\top}\in \RB^{K}, p_K=1-
    \textstyle\sum_{k=0}^{K-1}p_k \}, 
\end{equation*}
which is an affine subspace of $\sP^{\sgn}$. Here 
$\{x_k{=}k\iota_K\}_{k{=}0}^K$ are $K{+}1$ equally-spaced points of the support, 
$\iota_K{=}[K(1{-}\gamma)]^{{-}1}$ is the gap between adjacent points,
and $p_k$ is the `probability' (may be negative) that $\nu$ assigns to $x_k$.
We define the categorical
projection operator $\bPi_{K}{\colon}\sP^{\sgn}{\to}\sP^{\sgn}_K$ as
\begin{equation*}
\bPi_{K}\nu:=\argmin\nolimits_{\nu_\bp\in\sP^{\sgn}_{K}}\ell_{2}\prn{\nu, \nu_\bp},\quad \forall\nu\in \sP^{\sgn}. 
\end{equation*}
Following Proposition 5.14 in \citet{bdr2022}, one can show that $\bm{\Pi}_K\nu\in\sP^{\sgn}_K$ is uniquely represented with a vector $\bp_\nu=(p_k(\nu))_{k=0}^{K-1}\in\RB^K$, where
\begin{equation*}
     p_k(\nu)=\textstyle\int_{\brk{0,(1-\gamma)^{-1}}}(1-\abs{(x-x_k)/{\iota_K}})_+ \nu(dx).
\end{equation*}
We lift $\bPi_K$ to the product space by defining
$[\bm{\Pi}_K{\bm{\eta}}](s) := \bm{\Pi}_K\eta(s)$.
Recall that the $\mu_\pi$-weighted Cram\'er distance between $\bm{\eta}_1,\bm{\eta}_2\in(\sP^{\sgn})^\gS$ is defined as $\ell_{2,\mu_\pi}(\bm{\eta}_1,\bm{\eta}_2)=(\EB_{s\sim\mu_{\pi}}[\ell_2^2(\eta_1(s),\eta_2(s))])^{1/2}$.
One can check that the categorical Bellman operator $\bPi_{K}{\gT}^{\pi}$ is a $\sqrt\gamma$-contraction in the Polish space $((\sP_K^{\sgn})^\gS,\ell_{2,\mu_\pi})$. 
Hence, the categorical projected Bellman equation $\bm{\eta}=\bPi_{K}\gT^{\pi}\bm{\eta}$ admits a unique solution $\bm{\eta}^{\pi,K}$, which satisfies $W_{1,\mu_\pi}(\bm{\eta}^\pi,\bm{\eta}^{\pi,K})\leq(1{-}\gamma)^{-1}\ell_{2,\mu_\pi}(\bm{\eta}^\pi,\bPi_K\bm{\eta}^{\pi})$ ( Proposition~3 in \citep{rowland2018analysis}).
Applying linear stochastic approximation to solving the equation yields tabular categorical TD learning, and the iteration rule is given by
\begin{equation*}
    \eta_{t}(s_t)\gets\eta_{t-1}(s_{t})-\alpha[\eta_{t-1}(s_{t})-\bPi_K\prn{b_{r_t,\gamma}}_\#\eta_{t-1}(s_{t+1})],\quad \eta_t(s)\gets \eta_{t-1}(s),\ \forall s\neq s_t.
\end{equation*}

\subsection{Zero-Mass Signed Measure Space}
To evaluate the approximation error of the solution $\bm{\eta}^{\pi,K}$ of the categorical projected Bellman equation, 
we need to first put the distributions of interest into a Hilbert space $(\gM^\gS,\langle\cdot\rangle{\cdot}_{\mu_\pi})$ (for simplicity, we denote by $\gM^\gS$ itself its completion space), where $\gM$ is the zero-mass signed measure space defined as follows.
\begin{equation*}
    \gM:= \left\{\nu\colon\abs{\nu}(\RB)< \infty ,\nu(\RB)=0,\text{supp}(\nu)\subseteq \brk{0,\frac{1}{1{-}\gamma} } \right\},
\end{equation*} 
and
\begin{equation*}
    \inner{\bm{\eta}}{\bm{\eta}^\prime}_{\mu_\pi}:=\EB_{s\sim\mu_\pi}\brk{\int_{0}^{\frac{1}{1-\gamma}}F_{\eta(s)}(x)F_{\eta^\prime(s)}(x)d x}.
\end{equation*} 
Due to the linearity of the distributional Bellman operator $\gT^\pi$, we can rewrite the distributional Bellman equation in the Hilbert space as
\begin{equation*}
    \bm{\eta}^\pi=\gT^\pi\bm{\eta}^\pi \Leftrightarrow \bm{\eta}^\pi-\bxi=\gT^\pi(\bm{\eta}^\pi-\bxi)+(\gT^\pi\bxi-\bxi),
\end{equation*}
where $\bxi:=((K+1)^{-1}\sum_{k=0}^K \delta_{x_k})_{s\in\gS}$ is the vector of discrete uniform distribution. 
Since $\|\gT^\pi\|_{\mu_\pi}\leq\sqrt{\gamma}$, $\bm{\eta}^\pi-\bxi$ is the unique solution to the following fixed-point equation for $\bm{\eta}\in\gM^\gS$:
\begin{equation*}
    \bm{\eta}=\gT^\pi\bm{\eta}+(\gT^\pi\bxi-\bxi).
\end{equation*}
\subsection{Approximation Error of Categorical TD learning}
Recall the space of all categorical parameterized signed measures with total mass $1$
\begin{equation*}
    \sP^{\sgn}_K := \left\{ \nu_\bp=\sum_{k=0}^K p_k \delta_{x_k} \colon  \bp=\prn{p_0, \ldots, p_{K-1}}^{\top}\in \RB^K, p_K=1-\sum_{k=0}^{K-1}p_k \right\}.
\end{equation*}
Similarly, one can define 
\begin{equation*}
    \gM_{K} := \left\{ \nu_\bp=\sum_{k=0}^K p_k \delta_{x_k} \colon  p_0, \ldots, p_{K}\in \RB, \sum_{k=0}^{K}p_k=0 \right\}. 
\end{equation*}
It is easy to check that $\gM_{K}$ is a linear subspace of $\gM$ with dimension $K$.
Specifically,
\begin{equation*}
    \gM_{K}=\operatorname{span}\brc{\delta_{x_0}-\delta_{x_K},\ldots,\delta_{x_{K-1}}-\delta_{x_K}}.
\end{equation*}
The inner product in $\gM$ is defined as
\begin{equation*}
    \inner{\nu}{\nu^\prime}:=\int_{0}^{\frac{1}{1-\gamma}}F_{\nu}(x)F_{\nu^\prime}(x)d x.
\end{equation*}
It is more convenient to work with orthonormal basis, one can check that for any $i,j\in\{0,1,\ldots,K\}$, it holds that
\begin{align*}
    \inner{\delta_{x_i}}{\delta_{x_j}}=&\frac{1}{1-\gamma}\int_{0}^{1}\ind\brc{x\geq \frac{i}{K}}\ind\brc{x\geq \frac{j}{K}} dx\\
    =&\frac{1}{1-\gamma}\prn{1-\frac{\max\brc{i,j}}{K}},
\end{align*}
then
\begin{align*}
    \inner{\delta_{x_i}-\delta_{x_K}}{\delta_{x_j}-\delta_{x_K}}=&\inner{\delta_{x_i}}{\delta_{x_j}}-\inner{\delta_{x_i}}{\delta_{x_K}}-\inner{\delta_{x_j}}{\delta_{x_K}}+\inner{\delta_{x_K}}{\delta_{x_K}}\\
    =&\frac{1}{K(1-\gamma)}\prn{K+K-\max\brc{i,j}-K} \\
    =&\frac{1}{K(1-\gamma)}\prn{K-\max\brc{i,j}}.
\end{align*}
Let $\bSigma_K:=( \langle\delta_{x_i}-\delta_{x_K},\delta_{x_j}-\delta_{x_K}\rangle)_{i,j\in\{0,1,\ldots, K-1}\}\in\RB^{K\times K}$ be the Gram matrix, then we have
\begin{equation*}
    \bSigma_K=\frac{1}{K(1-\gamma)}\bC^{\top}\bC,
\end{equation*}
and $(\nu_1,\ldots,\nu_{K})$ is an orthonormal basis of $\gM_K$ where
\begin{equation*}
    \begin{bmatrix} \nu_1\\ \vdots\\ \nu_K \end{bmatrix}=\bSigma_K^{-\frac{1}{2}}\begin{bmatrix} \delta_{x_0}-\delta_{x_K}\\ \vdots\\ \delta_{x_{K-1}}-\delta_{x_K} \end{bmatrix}=\sqrt{K(1-\gamma)}\prn{\bC^{\top}\bC}^{-\frac{1}{2}}\begin{bmatrix} \delta_{x_0}-\delta_{x_K}\\ \vdots\\ \delta_{x_{K-1}}-\delta_{x_K} \end{bmatrix}.
\end{equation*}

In the following, we will only consider the tabular case, that is, $\gS$ is a finite set.
Now, we consider the product space $\gM_{K}^\gS$, which is a linear subspace of $\gM^\gS$ with dimension $\gS\times K$.
Specifically,
\begin{equation*}
    \gM_{K}^\gS=\brc{\operatorname{span}\brc{\delta_{x_0}-\delta_{x_K},\ldots,\delta_{x_{K-1}}-\delta_{x_K}}}^\gS.
\end{equation*}

We first compute $\tilde\bUp_K$, the matrix representation of the linear operator $\bPi_K\gT^\pi$ in $\gM_K^S$ in view of this basis. In particular, the categorical projection operator $\bPi_{K}{\colon}\sP^{\sgn}{\to}\sP^{\sgn}_K$ is defined as
\begin{equation*}
\bPi_{K}\nu:=\argmin\nolimits_{\nu_\bp\in\sP^{\sgn}_{K}}\ell_{2}\prn{\nu, \nu_\bp},\quad \forall\nu\in \sP^{\sgn}. 
\end{equation*}
According to the proof of Proposition~E.3, we have for any $\bm{\eta}\in\gM^\gS_K$
\begin{align*}
 p_k\prn{\brk{\gT^{\pi}\bm{\eta}}(s)}=&\EB_{X\sim \brk{\gT^{\pi}\bm{\eta}}(s)}\brk{\prn{1-\abs{\frac{X-x_k}{\iota_K}}}_+}\\
=&\EB\brk{\EB_{G\sim \eta(s_1)}\brk{\prn{1-\abs{\frac{r_0+\gamma G-x_k}{\iota_K}}}_+}\Big| s_0=s }\\
=&\EB\brk{\sum_{j=0}^K p_{\eta(s_1)}(j)\prn{1-\abs{\frac{r_0+\gamma x_j-x_k}{\iota_K}}}_+\Big| s_0=s }\\
=&\EB\brk{\sum_{j=0}^Kp_{\eta(s_1)}(j)g_{j,k}(r_0)\Big| s_0=s }\\
=&\EB\brk{\sum_{j=0}^{K-1}p_{\eta(s_1)}(j)\prn{g_{j,k}(r_0)-g_{K,k}(r_0)}\Big| s_0=s },
\end{align*}
that is,
\begin{align*}
     \bp_{\gT^\pi\bm{\eta}}(s)=&\EB\brk{\prn{\bG(r_0)-\bm{1}_K^{\top}\otimes\bg_K(r_0)}  \bp_{\bm{\eta}}(s_1)  \Big| s_0=s }\\
     =&\EB\brk{\tilde{\bG}(r_0) \bp_{\bm{\eta}}(s_1)  \Big| s_0=s }.
\end{align*}
Hence, for any $\bz\in\RB^{\gS\times K}$
\begin{equation*}
     \tilde\bUp_K(s) \bz:=(\tilde\bUp_K \bz)(s)=\EB\brk{\tilde{\bG}(r_0)  \bz(s_1)  \Big| s_0=s },\quad  \tilde\bUp_K\bz=\prn{\tilde\bUp_K(s) \bz}_{s\in\gS}=\prn{\EB\brk{\tilde{\bG}(r_0)  \bz(s_1)  \Big| s_0=s }}_{s\in\gS}.
\end{equation*}
It is easy to check that
\begin{equation*}
    \bigtimes_{s\in\gS}\operatorname{span}\brc{\mu_\pi(s)\nu_1,\ldots,\mu_\pi(s)\nu_K}
\end{equation*}
is an orthonormal basis (not rigorous) of $\gM_K^\gS$.
Without loss of generality, we assume $\min_{s\in\gS}\mu_\pi(s)>0$.
Let $\bLamb:=\operatorname{diag}(\mu_\pi(s))_{s\in\gS}$, then the Gram matrix is
\begin{equation*}
    \bLamb\otimes \bSigma_K= \frac{\bLamb\otimes (\bC^{\top}\bC)}{K(1-\gamma)}.
\end{equation*}
Here, $\bC^{\top}\bC$ appears in the RHS because we use a different parameter permutation.
and the transition matrix is 
\begin{equation*}
    (\bLamb\otimes \bSigma_K)^{-\frac{1}{2}}= \sqrt{K(1-\gamma)}{\bLamb^{-\frac{1}{2}}\otimes (\bC^{\top}\bC)^{-\frac{1}{2}}}.
\end{equation*}
Then the matrix representation of the linear operator $\bPi_K\gT^\pi$ in $\gM_K^S$ under the orthonormal basis is
\begin{equation*}
    \bUp_K=\prn{\bLamb^{\frac{1}{2}}\otimes (\bC^{\top}\bC)^{\frac{1}{2}}}\tilde\bUp_K\prn{\bLamb^{-\frac{1}{2}}\otimes (\bC^{\top}\bC)^{-\frac{1}{2}}}.
\end{equation*}
Note that $\bxi=((K+1)^{-1}\sum_{k=0}^K \delta_{x_k})_{s\in\gS}\in(\sP^{\sgn}_K)^\gS \Rightarrow \bxi=\bPi_{K}\bxi$, we have $\bm{\eta}^{\pi,K}-\bxi$ is the unique solution to the projected fixed-point equation for $\bm{\eta}\in\gM_{K}^\gS$:
\begin{equation*}
    \bm{\eta}=\bPi_{K}\prn{\gT^\pi\bm{\eta}+(\gT^\pi\bxi-\bxi)}.
\end{equation*}
One can check that
\begin{equation*}
    \norm{(\bm{\eta}^{\pi,K}-\bxi)-(\bm{\eta}^\pi-\bxi)}_{\mu_\pi}=\ell_{2,\mu_{\pi}}\prn{\bm{\eta}^{\pi,K},\bm{\eta}^\pi}.
\end{equation*}
\begin{equation*}
    \norm{\bPi_{K}(\bm{\eta}^\pi-\bxi)-(\bm{\eta}^\pi-\bxi)}_{\mu_\pi}=\ell_{2,\mu_{\pi}}\prn{\bPi_{K}\bm{\eta}^\pi,\bm{\eta}^\pi}.
\end{equation*}
Applying the upper bounds derived in \citet{yu10error,doi:10.1287/moor.2022.1341}, we have the following upper bounds.
\begin{align}
\ell_{2,\mu_{\pi}}^2\prn{\bm{\eta}^{\pi,K},\bm{\eta}^\pi}\leq&\prn{1+\norm{(\bm{\gI}-\bPi_{K}\gT^\pi)^{-1}\bPi_{K}\gT^\pi(\bm{\gI}-\bPi_{K})}_{\mu_\pi}^2}   \ell_{2,\mu_{\pi}}^2\prn{\bPi_{K}^{\pi}\bm{\eta}^\pi,\bm{\eta}^\pi}\label{eq:ctd_tight_but_unknown_approx_error}\\
    \leq&\prn{1+\norm{(\bI_{\gS\times K}-\bUp_K)^{-1}(\gamma\bI_{\gS\times K}-\bUp_K\bUp_K^{\top})(\bI_{\gS\times K}-\bUp_K)^{-\top}}}\ell_{2,\mu_{\pi}}^2\prn{\bPi_{K}\bm{\eta}^\pi,\bm{\eta}^\pi}\label{eq:ctd_a_bit_loose_approx_error}\\
    \leq&\frac{1}{1-\gamma}\ell_{2,\mu_{\pi}}^2\prn{\bPi_{K}\bm{\eta}^\pi,\bm{\eta}^\pi},\label{eq:ctd_worst_case_approx_error}
\end{align}
where the upper bound in Equation~\eqref{eq:ctd_tight_but_unknown_approx_error} is tight but not accessible because it involves infinite-dimensional objects $\gT^\pi(\bm{\gI}-\bPi_{K})$, while the looser upper-bound in Equation~\eqref{eq:ctd_a_bit_loose_approx_error} is accessible if the MDP is known and is better than the worst-case upper-bound in Equation~\eqref{eq:ctd_worst_case_approx_error}.
\section{Proofs in Section~\ref{Section:instance-dependent_analysis}}\label{appendix:proof_instance_dependent}
\noindent In this appendix, we provide proofs for the instance-dependent analysis in Section~\ref{Section:linear_ctd}.
\subsection{Approximation Error}\label{sub:proof_fine-grained}
To prove Proposition~\ref{prop:fine_grained_approx_error}, recall the space of linear-categorical parameterized signed measures with total mass $1$:
\begin{equation*}
    \sP^{\sgn}_{\bphi,K} := \left\{\bm{\eta}_{\btheta}=\prn{\eta_{\btheta}(s)}_{s\in\gS}\colon \eta_{\btheta}(s)=\textstyle\sum_{k=0}^Kp_k(s;\btheta)\delta_{x_k}, \btheta=(\btheta(0)^{\top}, \ldots, \btheta(K{-}1)^{\top})^{\top}\in\RB^{dK} \right\},
\end{equation*}
where $p_k(s;\btheta){=}F_k(s;\btheta){-}F_{k{-}1}(s;\btheta)$ and the CDF of $\eta_{\btheta}(s)$ at $x_k$ is given by 
\begin{equation*}
            F_k(s;\btheta)= \begin{cases}             0, & \text{ for}\ k=-1, \\
            \bphi(s)^{\top}\btheta(k)+\frac{k+1}{K+1}, & \text{ for}\ k\in\brc{0, 1, \ldots, K-1}, \\
            1, & \text{ for}\ k=K.
            \end{cases}
\end{equation*}

Similarly, one can define 
\begin{equation*}
    \gM_{\bphi,K} := \left\{\bm{\eta}_{\btheta}=\prn{\eta_{\btheta}(s)}_{s\in\gS}\colon \eta_{\btheta}(s)=\textstyle\sum_{k=0}^Kp_k(s;\btheta)\delta_{x_k}, \btheta=(\btheta(0)^{\top}, \ldots, \btheta(K{-}1)^{\top})^{\top}\in\RB^{dK} \right\},
\end{equation*}
where $p_k(s;\btheta){=}F_k(s;\btheta){-}F_{k{-}1}(s;\btheta)$ and
\begin{equation*}
    F_k(s;\btheta)=
    \begin{cases}             0, & \text{ for}\ k\in\brc{-1,K},\\
    \bphi(s)^{\top}\btheta(k), & \text{ for}\ k\in\brc{0, 1, \ldots, K-1},
    \end{cases}
\end{equation*}
is CDF of $\eta_{\btheta}(s)$ at $x_k$.
It is easy to check that $\gM_{\bphi,K}$ is a linear subspace of $\gM^\gS$ with dimension $dK$.
Specifically, we denote by $\bphi_i\colon \gS\to \RB$ the $i$-th coordinate of $\bphi$, that is, $(\bphi(s))_i=\bphi_i(s)$, then $\bphi_1,\ldots,\bphi_d$ are $d$ linearly independent vectors in $\RB^\gS$, and for any $\nu\in\gM$, $\bphi_i \nu\in\gM^\gS$ is defined by $(\bphi_i \nu)(s):=\bphi_i(s)\nu$.
With these notations, we have a basis corresponding to $\bw\in\RB^{dK}$:
\begin{equation*}
    \gM_{\bphi,K}=\operatorname{span}\brc{\bphi_i(\delta_{k-1}-\delta_K)}_{i\in[d],k\in[K]}.
\end{equation*}
We first compute $\tilde\bUp_K$, the matrix representation of the linear operator $\bPi_{\bphi, K}^{\pi}\gT^\pi$ in $\gM_{\bphi,K}$ on this basis.
According to the proof of Proposition~E.3, we have for any $\bm{\eta}_{\btheta}\in\gM_{\bphi,K}$
\begin{align*}
p_k\prn{\brk{\gT^{\pi}\bm{\eta}_\btheta}(s)}=&\EB_{X\sim \brk{\gT^{\pi}\bm{\eta}_\btheta}(s)}\brk{\prn{1-\abs{\frac{X-x_k}{\iota_K}}}_+}\\
=&\EB\brk{\EB_{G\sim \eta_{\btheta}(s_1)}\brk{\prn{1-\abs{\frac{r_0+\gamma G-x_k}{\iota_K}}}_+}\Big| s_0=s }\\
=&\EB\brk{\sum_{j=0}^Kp_j(s_1;
\btheta)\prn{1-\abs{\frac{r_0+\gamma x_j-x_k}{\iota_K}}}_+\Big| s_0=s }\\
=&\EB\brk{\sum_{j=0}^Kp_j(s_1;
\btheta)g_{j,k}(r_0)\Big| s_0=s }\\
=&\EB\brk{\sum_{j=0}^{K-1}p_j(s_1;
\btheta)\prn{g_{j,k}(r_0)-g_{K,k}(r_0)}\Big| s_0=s }.
\end{align*}
That is,
\begin{align*}
     \bp_{\gT^\pi\bm{\eta}_{\bm{\theta}}}(s)=&\EB\brk{\prn{\bG(r_0)-\bm{1}_K^{\top}\otimes\bg_K(r_0)}  \bp_{\btheta}(s_1)\Big| s_0=s }\\
     =&\EB\brk{\tilde{\bG}(r_0) \otimes \bphi(s_1)^{\top}\Big| s_0=s }\bw\\
     =&\EB\brk{\tilde{\bG}(r_0) \bW^{\top} \bphi(s_1)\Big| s_0=s }.
\end{align*}
According to Proposition 3.1 in \citet{peng2025finitesampleanalysisdistributional}, we have $\bPi_{\bphi, K}^{\pi}\gT^{\pi}\bm{\eta}_{\bm{\theta}}=\bm{\eta}_{\tilde{\bm{\theta}}}$ where
\begin{align*}
\tilde\bW=&\bSigma_{\bphi}^{-1}\EB_{s\sim\mu_{\pi}}\brk{\bphi(s)\bp_{\gT^{\pi}\bm{\eta}_{\bm{\theta}}}(s)^{\top}}\\
=&\bSigma_{\bphi}^{-1}\EB_{s\sim\mu_{\pi}}\brk{\bphi(s)\bphi(s^\prime)^\top \bW\tilde{\bG}(r)^{\top}},
\end{align*}
or equivalently
\begin{align*} \tilde\bw=&\vect\prn{\EB_{s\sim\mu_{\pi}}\brk{\bSigma_{\bphi}^{-1}\bphi(s)\bphi(s^\prime)^\top \bW\tilde{\bG}(r)^{\top}}}\\
     =&\EB_{s\sim\mu_{\pi}}\brk{\tilde{\bG}(r) \otimes \prn{\bSigma_{\bphi}^{-1}\bphi(s)\bphi(s^\prime)^\top}}\bw.
\end{align*}
Hence, 
\begin{align*}
     \tilde\bUp_K =\EB_{s\sim\mu_{\pi}}\brk{\tilde{\bG}(r) \otimes \prn{\bSigma_{\bphi}^{-1}\bphi(s)\bphi(s^\prime)^\top}}.
\end{align*}
It is easy to check that
\begin{equation*}
    \brc{\prn{\bSigma_{\bphi}^{-\frac{1}{2}}\bphi}_i\nu_k}_{i\in[d],k\in[K]}
\end{equation*}
is an orthonormal basis of $\gM_{\bphi,K}$.
The Gram matrix is
\begin{equation*}
    \bSigma_K\otimes\bSigma_{\bphi} = \frac{(\bC^{\top}\bC)\otimes\bSigma_{\bphi}  }{K(1-\gamma)},
\end{equation*}
and the transition matrix is 
\begin{equation*}
    (\bSigma_K\otimes\bSigma_{\bphi})^{-\frac{1}{2}}= \sqrt{K(1-\gamma)}(\bC^{\top}\bC)^{-\frac{1}{2}}\otimes\bSigma_{\bphi}^{-\frac{1}{2}}.
\end{equation*}
Then the matrix representation of the linear operator $\bPi_{\bphi, K}^{\pi}\gT^\pi$ in $\gM_{\bphi,K}$ in the orthonormal basis is
\begin{align}
\bUp_K=&\prn{(\bC^{\top}\bC)^{\frac{1}{2}}\otimes\bSigma_{\bphi}^{\frac{1}{2}}}\tilde\bUp_K\prn{(\bC^{\top}\bC)^{-\frac{1}{2}}\otimes\bSigma_{\bphi}^{-\frac{1}{2}}} \notag\\
=&\EB_{s\sim\mu_{\pi}}\brc{\brk{(\bC^{\top}\bC)^{\frac{1}{2}}\tilde{\bG}(r)(\bC^{\top}\bC)^{-\frac{1}{2}}} \otimes \prn{\bSigma_{\bphi}^{-\frac{1}{2}}\bphi(s)\bphi(s^\prime)^\top\bSigma_{\bphi}^{-\frac{1}{2}}}}.\label{eq:Upsilon_K}
\end{align}
Note that $\bxi=\bm{\eta}_{\bm{0}_{dK}}\in\sP^{\sgn}_{\bphi,K} \Rightarrow \bxi=\bPi_{\bphi, K}^{\pi}\bxi$, we have that $\bm{\eta}_{\btheta^\star}-\bxi$ is the unique solution to the projected fixed-point equation for $\bm{\eta}\in\gM_{\bphi,K}$:
\begin{equation*}
    \bm{\eta}=\bPi_{\bphi, K}^{\pi}\prn{\gT^\pi\bm{\eta}+(\gT^\pi\bxi-\bxi)}
\end{equation*}
One can check that
\begin{equation*}
    \norm{(\bm{\eta}_{\btheta^\star}-\bxi)-(\bm{\eta}^\pi-\bxi)}_{\mu_\pi}=\ell_{2,\mu_{\pi}}\prn{\bm{\eta}_{\btheta^{\star}},\bm{\eta}^\pi}.
\end{equation*}
\begin{equation*}
    \norm{\bPi_{\bphi, K}^{\pi}(\bm{\eta}^\pi-\bxi)-(\bm{\eta}^\pi-\bxi)}_{\mu_\pi}=\ell_{2,\mu_{\pi}}\prn{\bPi_{\bphi, K}^{\pi}\bm{\eta}^\pi,\bm{\eta}^\pi}.
\end{equation*}
Applying upper bounds derived in \citet{yu10error,doi:10.1287/moor.2022.1341}, we have the following upper bounds.
\begin{align}
    \ell^2_{2,\mu_{\pi}}\prn{\bm{\eta}_{\btheta^{\star}},\bm{\eta}^\pi}\leq&\prn{1+\norm{(\bm{\gI}-\bPi_{\bphi, K}^{\pi}\gT^\pi)^{-1}\bPi_{\bphi, K}^{\pi}\gT^\pi(\bm{\gI}-\bPi_{\bphi, K}^{\pi})}_{\mu_\pi}^2}   \ell^2_{2,\mu_{\pi}}\prn{\bPi_{\bphi, K}^{\pi}\bm{\eta}^\pi,\bm{\eta}^\pi}\label{eq:tight_but_unknown_approx_error}\\
    \leq&\prn{1+\norm{(\bI_{dK}-\bUp_K)^{-1}(\gamma\bI_{dK}-\bUp_K\bUp^{\top}_K)(\bI_{dK}-\bUp_K)^{-\top}}}\ell^2_{2,\mu_{\pi}}\prn{\bPi_{\bphi, K}^{\pi}\bm{\eta}^\pi,\bm{\eta}^\pi}\label{eq:a_bit_loose_approx_error}\\
     \leq&\frac{1}{1-\gamma}\ell_{2,\mu_{\pi}}^2\prn{\bPi_{\bphi,K}^\pi\bm{\eta}^\pi,\bm{\eta}^\pi},\label{eq:worst_case_approx_error}
\end{align}
where $\bUp_K$ is defined in Equation~\eqref{eq:Upsilon_K}.
Among these bounds, the upper bound in Equation~\eqref{eq:tight_but_unknown_approx_error} is tight but not accessible because it involves infinite-dimensional objects $\gT^\pi(\bm{\gI}-\bPi_{\bphi,K}^\pi)$, while the looser upper bound in Equation~\eqref{eq:a_bit_loose_approx_error} is accessible if the MDP is known and is better than the worst-case upper-bound in Equation~\eqref{eq:worst_case_approx_error}. This parallels with the ending remarks for the approximation error of categorical TD learning in Appendix~\ref{appendix:CTD}.

\subsection{Proof of Instance-dependent Estimation Error of \LCTD}
\paragraph{Proof of Theorem~\ref{thm:leading_term}}
Recall that $\widetilde{\bSigma}_{\be}:=\bJ^{1 / 2} \bar{\bA}^{-1} \bSigma_{\be} \bar{\bA}^{-\top} \bJ^{1 / 2}$ and $\bJ= \bI_K\otimes\bSigma_{\bphi}$. Let $\mathrm{g}= (1-\sqrt{\gamma})^{-1}$. We have proved in Equation 28 in \citet{peng2025finitesampleanalysisdistributional} that
\begin{align*}
    \prn{\bI_K\otimes\bSigma_{\bphi}^{-\frac{1}{2}}}\bar{\bA}^{\top}\prn{\bI_K\otimes\bSigma_{\bphi}^{-1}}\bar{\bA}\prn{\bI_K\otimes\bSigma_{\bphi}^{-\frac{1}{2}}}\succcurlyeq \prn{1-\sqrt\gamma}^2 \bI_{dK}.
\end{align*}
Thus,
\begin{equation}\label{eq:assum_1}
\bJ^{\frac{1}{2}} \bar{\bA}^{-\top} \bJ \bar{\bA}^{-1} \bJ^{\frac{1}{2}} \preccurlyeq\mathrm{g}^2 \bI.
\end{equation}
Let $\omega = \sqrt{2(1+\gamma)/\lambda_{\min}}$, by the inequality $\EB[\btheta^{\top}\bA^{\top}\bA\btheta]\leq 2(1+\gamma)\|\btheta\|_{\bJ}^2$ in the proof of Lemma~\ref{lem:A^TA},
\begin{equation}\label{eq:assum_2}
\EB\brk{\bA^{\top} \bJ^{-1} \bA} \preccurlyeq \frac{1}{\lambda_{\min}}\EB\brk{\bA^{\top}\bA}\preccurlyeq\omega^2 \bJ.
\end{equation}
Let $\varrho = (1+\sqrt{\gamma})$, then we have
\begin{align}
\tr\left(\bSigma_{\be}\right) & =\tr\left(\bar{\bA}^{\top} \bJ^{-\frac{1}{2}} \bJ^{\frac{1}{2}} \bar{\bA}^{-\top} \bSigma_{\varepsilon} \bar{\bA}^{-1} \bJ^{\frac{1}{2}} \bJ^{-\frac{1}{2}} \bar{\bA}\right) \notag\\
&= \tr\left(\bJ^{-\frac{1}{2}} \bar{\bA} \bJ^{-\frac{1}{2}} \bJ \bJ^{-\frac{1}{2}} \bar{\bA}^{\top} \bJ^{-\frac{1}{2}} \bJ^{\frac{1}{2}} \bar{\bA}^{-\top} \bSigma_{\be} \bar{\bA}^{-1} \bJ^{\frac{1}{2}}\right) \notag\\
&\leq\left\|\bJ^{-\frac{1}{2}} \bar{\bA} \bJ^{-\frac{1}{2}} \bJ \bJ^{-\frac{1}{2}} \bar{\bA}^{\top} \bJ^{-\frac{1}{2}}\right\| \tr\left(\bJ^{\frac{1}{2}} \bar{\bA}^{-\top} \bSigma_{\varepsilon} \bar{\bA}^{-1} \bJ^{\frac{1}{2}}\right) \notag\\
&\leq \norm{\bJ^{-\frac{1}{2}} \bar{\bA} \bJ^{-\frac{1}{2}}}^2\norm{\bJ}\tr\left(\bJ^{\frac{1}{2}} \bar{\bA}^{-\top} \bSigma_{\varepsilon} \bar{\bA}^{-1} \bJ^{\frac{1}{2}}\right).\label{eq:two_trace_derivation}
\end{align}
Recall that $\|\bJ\|\leq 1$ 
and $\bar{\bA}=\bJ-\EB_{s, r, s^\prime}[(\bC\tilde{\bG}(r)\bC^{-1})\otimes(\bphi(s)\bphi(s^\prime)^{\top})]$.
Thus, by Equation 29 in \citet{peng2025finitesampleanalysisdistributional},
\begin{align*}
    \norm{\bJ^{-\frac{1}{2}} \bar{\bA} \bJ^{-\frac{1}{2}}} &= \norm{\bI_{dK}- \EB_{s, r, s^\prime}\brk{\prn{\bC\tilde{\bG}(r)\bC^{-1}}\otimes\prn{\bSigma_{\bphi}^{-\frac{1}{2}}\bphi(s)\bphi(s^\prime)^{\top}\bSigma_{\bphi}^{-\frac{1}{2}}}}}\\
    &\leq 1+\norm{ \EB_{s, r, s^\prime}\brk{\prn{\bC\tilde{\bG}(r)\bC^{-1}}\otimes\prn{\bSigma_{\bphi}^{-\frac{1}{2}}\bphi(s)\bphi(s^\prime)^{\top}\bSigma_{\bphi}^{-\frac{1}{2}}}}}\\
    &\leq 1+\sqrt{\gamma}.
\end{align*}
Substituting into Equation~\eqref{eq:two_trace_derivation}, we obtain the following inequality between two traces.
\begin{equation}\label{eq:assum_3}
\tr(\bSigma_{\be}) \leq \varrho^2 \tr(\widetilde{\bSigma}_{\be}).
\end{equation}
Applying Theorem~14 in \citet{samsonov2024improved}, collecting Equation~\eqref{eq:assum_1}, Equation~\eqref{eq:assum_2} and Equation~\eqref{eq:assum_3} we have for any $\alpha \in (0,(1-\sqrt\gamma)/256)$ and $a=(1-\sqrt\gamma)\lambda_{\min}/2$

\begin{align*}
    \EB\brk{\left\|\bar{\btheta}_T-\btheta^{\star}\right\|_{\bJ}^2} &\lesssim  \frac{\tr(\widetilde{\bSigma}_{\be})}{T}+\frac{ \mathrm{g}^2 \tr(\bSigma_{\be})}{a T}\left(\frac{\left\|\bJ^{-\frac{1}{2}}\right\|^2}{\alpha T}+\omega^2\left\|\bJ^{\frac{1}{2}}\right\|^2 \alpha\right) \\
    & \quad + \mathrm{g}^2(1-\alpha a)^T\left(\frac{\left\|\bJ^{-\frac{1}{2}}\right\|^2}{\alpha^2 T^2}+\frac{\omega^2\left\|\bJ^{\frac{1}{2}}\right\|^2}{\alpha a T^2}\right)\left\|\btheta_0-\btheta^{\star}\right\|^2.
\end{align*}
Using Lemma~\ref{lem:tr_cov_sigma_e} to bound the trace $\tr(\bSigma_{\be})$, one can translate this result into Theorem~\ref{thm:leading_term}.

\paragraph{Proof of Lemma~\ref{lem:trace_critical}}
Note that by Equation~\eqref{eq:assum_1}
\begin{align}
\left\|\bJ^{\frac{1}{2}} \bar{\bA}^{-1} u\right\|^2
&=u^{\top} \bJ^{-\frac{1}{2}} \prn{\bJ^{\frac{1}{2}} \bar{\bA}^{-\top} \bJ \bar{\bA}^{-1} \bJ^{\frac{1}{2}}} \bJ^{-\frac{1}{2}} u \notag\\ 
& \leq \mathrm{g}^2 u^{\top} \bJ^{-1} u \notag\\ 
& \leq \mathrm{g}^2/\lambda_{\min}\norm{u}^2.\label{eqn:norm_of_J1/2_barA_-1}
\end{align}
Thus,
\begin{equation}\label{eq:norm_A_TJA_1}
\norm{\bar{\bA}^{-\top} \bJ\bar{\bA}^{-1}}\leq \mathrm{g}^2/\lambda_{\min} = 1/((1-\sqrt{\gamma})^2\lambda_{\min})\lesssim 1/((1-\gamma)^2\lambda_{\min}).
\end{equation}
Note that Lemma~\ref{lem:tr_cov_sigma_e} shows
$\tr(\bSigma_{\be})\lesssim \|\btheta^{\star}\|_{\bJ}^2+K(1-\gamma)^2$ 
and
\begin{align*}
\tr\left(\widetilde{\bSigma}_{\be}\right) &= \tr\prn{\bJ^{\frac{1}{2}} \bar{\bA}^{-1} \bSigma_{\be} \bar{\bA}^{-\top} \bJ^{\frac{1}{2}}}\\
&=\tr\prn{\bar{\bA}^{-\top} \bJ\bar{\bA}^{-1}\bSigma_{\be}}\\
&\leq \norm{\bar{\bA}^{-\top} \bJ\bar{\bA}^{-1}}\tr(\bSigma_{\be}).
\end{align*}
Putting these together with Equation~\eqref{eq:norm_A_TJA_1}, we have the desired result.

\subsection{Proof of Minimax Lower Bound for Stochastic Error}
First, we state a version of the Hajek-Le Cam local asymptotic minimax theorem \citep{lancaster1972norms,hajek1972local} as follows.

Let $\{\mathbb{P}_{\vartheta}\}_{\vartheta \in \Theta}$ denote a family of parametric models, quadratically mean differentiable with Fisher information matrix $J_{\vartheta^{\prime}}$. 
Fix some parameter $\vartheta \in \Theta$, and consider a function $h: \Theta \rightarrow$ $\mathbb{R}^d$ that is differentiable at $\vartheta$. Then for any quasi-convex loss $l: \mathbb{R}^d \rightarrow \mathbb{R}$, we have
\begin{equation*}
\lim _{c \rightarrow \infty} \lim _{T \rightarrow \infty} \inf _{\widehat{h}_T} \sup _{\vartheta^{\prime}:\left\|\vartheta^{\prime}-\vartheta\right\|_2 \leq c / \sqrt{T}} \EB_{\vartheta^{\prime}}\brk{l(\sqrt{T}\prn{\widehat{h}_N-h\prn{\vartheta^{\prime}}}}=\EB[l(Z)],
\end{equation*}
where the infimum is taken over all estimators $\widehat{h}_N$ that are measurable functions of $N$ i.i.d. data points drawn from $\mathbb{P}_{\vartheta^{\prime}}$, and the expectation is taken over a multivariate Gaussian
\begin{equation*}
Z \sim \mathcal{N}\prn{0, \nabla h(\vartheta)^{\top} J_{\vartheta}^{\dagger} \nabla h(\vartheta)}.
\end{equation*}
In our case, each observation $\hat{\gO}$ takes values of the form: 
\begin{equation*}
        \gO_{s,s^{\prime},r} = \prn{\bI_{K} \otimes (\bphi(s)\bphi(s)^\top) - \EB\brk{ \tilde{\bG}(r) \otimes (\bphi(s)\bphi(s')^\top)},\frac{1}{K+1}\brk{  \prn{ \sum_{j=0}^K \bg_j(r) - \bm{1}_K }\otimes\bphi(s) }}.
\end{equation*}
Here we assume that both the state space and the reward space are discrete sets. 
For notation simplicity, let the state space be $[D]$ and the reward index space be $[S_R]$. The log-likelihood given an observation is 
\begin{equation*}
    L(\mathcal{I}=(\mu,\bP,\bR)|\hat{\gO}) = \sum_{s,s^{\prime} \in [D], r \in [S_R]} \mathbb{I}\{\hat{\gO} = \gO_{s,s^{\prime},r}\} \prn{\log(\mu_{s})+\log(\bP_{s,s^{\prime}})+\log(\bR_{s,s^{\prime},r})}.
\end{equation*}
Taking derivatives, we have 
\begin{align*}
\frac{\partial L}{\partial \mu_i}
&= \sum_{r \in [S_R]} \brk{\sum_{j \in[D]} \prn{\mathbb{I}_{\left\{\widehat{\gO}=\gO_{i, j,r}\right\}} \cdot \mu_i^{-1}-\mathbb{I}_{\left\{\widehat{\gO}=\gO_{D, j,r}\right\}} \cdot \mu_D^{-1}}}, \quad \text{for all } i \in[D-1].\\
\frac{\partial L}{\partial \bP_{i, j}}
&=\sum_{r \in [S_R]} \brk{\mathbb{I}_{\left\{\hat{\gO}=\gO_{i, j,r}\right\}} \bP_{i, j}^{-1}-\mathbb{I}_{\left\{\hat{\gO}=\gO_{i, D,r}\right\}} \bP_{i, D}^{-1}}, \quad \text{for all } i \in[D], j \in[D-1]. \\
\frac{\partial L}{\partial \bR_{i,j,r}} &= \mathbb{I}_{\{\hat{\mathcal{O}} = \mathcal{O}_{i,j,r}\}} \cdot \frac{1}{R_{i,j,r}} - \mathbb{I}_{\{\hat{\mathcal{O}} = \mathcal{O}_{i,j,S_R}\}} \cdot \bR_{i,j,S_R}^{-1}, \quad \text{for all } i,j \in [D],\ r \in [S_R - 1].
\end{align*}
Therefore, the entries in Fisher information matrix that can be non-zeros are given by
\begin{align*}
\mathbb{E} \left[ \frac{\partial L}{\partial \mu_i} \times \frac{\partial L}{\partial \mu_i} \right] &= \mu_i^{-1} + \mu_D^{-1}, &\quad& \text{for all } i \in [D-1]. \\
\mathbb{E} \left[ \frac{\partial L}{\partial \mu_i} \times \frac{\partial L}{\partial \mu_j} \right] &= \mu_D^{-1}, &\quad& \text{for all } i,j \in [D-1],\ i \neq j. \\
\mathbb{E} \left[ \frac{\partial L}{\partial \bP_{i,j}} \times \frac{\partial L}{\partial \bP_{i,j}} \right] &= \frac{\mu_i}{\bP_{i,j}} + \frac{\mu_i}{\bP_{i,D}}, &\quad& \text{for all } i \in [D],\ j \in [D-1]. \\
\mathbb{E} \left[ \frac{\partial L}{\partial \bP_{i,j}} \times \frac{\partial L}{\partial \bP_{i,j'}} \right] &= \frac{\mu_i}{\bP_{i,D}}, &\quad& \text{for all } i \in [D],\ j,j' \in [D-1],\ j \neq j'. \\
\mathbb{E} \left[ \frac{\partial L}{\partial \bR_{i,j,r}} \times \frac{\partial L}{\partial \bR_{i,j,r}} \right] &= \bR_{i,j,r}^{-1} + \bR_{i,j,S_R}^{-1}, &\quad& \text{for all } i,j \in [D],\ r \in [S_R - 1]. \\
\mathbb{E} \left[ \frac{\partial L}{\partial \bR_{i,j,r}} \times \frac{\partial L}{\partial \bR_{i,j,r'}} \right] &= \bR_{i,j,S_R}^{-1}, &\quad& \text{for all } i,j \in [D],\ r,r' \in [S_R - 1],\ r \neq r'.
\end{align*}
Putting these together, we obtain that the Fisher information matrix is a block diagonal matrix of the form
\begin{equation*}
J_{\vartheta} = 
\begin{bmatrix}
J_\mu & 0 & 0 & \cdots & 0 & 0 \\
0 & J_{\bP_1} & 0 & \cdots & 0 & 0 \\
0 & 0 & J_{\bP_2} & \cdots & 0 & 0 \\
\vdots & \vdots & \vdots & \ddots & \vdots & \vdots \\
0 & 0 & 0 & \cdots & J_{\bP_D} & 0 \\
0 & 0 & 0 & \cdots & 0 & J_\bR
\end{bmatrix},
\end{equation*}
\begin{equation*}
J_{\bR} = \textbf{Diag}\{J_{\bR_{i,j}} | i,j \in [D]\}.
\end{equation*}
Thus, the matrix $J_{\vartheta}^{\dagger}$ is also a block diagonal matrix of the form 
\begin{equation*}
J_{\vartheta}^{\dagger} = 
\begin{bmatrix}
J_\mu^{\dagger} & 0 & 0 & \cdots & 0 & 0 \\
0 & J_{\bP_1}^{\dagger} & 0 & \cdots & 0 & 0 \\
0 & 0 & J_{\bP_2}^{\dagger} & \cdots & 0 & 0 \\
\vdots & \vdots & \vdots & \ddots & \vdots & \vdots \\
0 & 0 & 0 & \cdots & J_{\bP_D}^{\dagger} & 0 \\
0 & 0 & 0 & \cdots & 0 & J_{\bR}^{\dagger}
\end{bmatrix},
\end{equation*}
\begin{equation*}
J_{\bR}^{\dagger} = \textbf{Diag}\{J_{\bR_{i,j}}^{\dagger} | i,j \in [D]\}.
\end{equation*}
By using the Sherman-Morrison-Woodbury formula we can compute $J_\mu^{\dagger}$, $J_{P_i}^{\dagger}$ and $J_{R_{i,j}}^{\dagger}$ as
\begin{align*}
J_\mu^{\dagger} & =\operatorname{diag}\prn{\brk{\mu_1, \ldots, \mu_{D-1}}}-\brk{\mu_1, \ldots, \mu_{D-1}}^{\top} \cdot\brk{\mu_1, \ldots, \mu_{D-1}}.\\
J_{\bP_i}^{\dagger} & =\mu_i^{-1} \cdot\prn{\operatorname{diag}\prn{\brk{\bP_{i, 1}, \ldots, \bP_{i, D-1}}}-\brk{\bP_{i, 1}, \ldots, \bP_{i, D-1}}^{\top} \cdot\brk{\bP_{i, 1}, \ldots, \bP_{i, D-1}}}, \quad \text {for all } i \in[D].\\
J_{\bR_{i,j}}^{\dagger} & =\operatorname{diag}\prn{\brk{\bR_{i,j,1}, \ldots, \bR_{i,j,D-1}}}-\brk{\bR_{i,j,1}, \ldots, \bR_{i,j,D-1}}^{\top} \cdot\brk{\bR_{i,j,1}, \ldots, \bR_{i,j,D-1}},\quad \text {for all } i,j \in[D].
\end{align*}
Recall that $\btheta^{\star} = \bA^{-1}\bb$, then for any variable $x$, $\frac{\partial\btheta^{\star}}{\partial x}=\bA^{-1}(\frac{d\bb}{dx} - \frac{d\bA}{dx}\btheta^{\star})$.
Equation~\eqref{eq:instance_solution_of_linear_ctd} writes
\begin{equation*}
\btheta^{\star}(\mathcal{I}) =  \brk{ \bJ - \EB\brk{ \prn{\bC\tilde{\bG}(r)\bC^{-1}} \otimes \bphi(s)\bphi(s')^\top } }^{-1}\EB \brk{  \frac{1}{K+1}\bC \prn{ \sum_{j=0}^K \bg_j(r) - \bm{1}_K }\otimes\bphi(s) } = \bar{\bA}^{-1}{\bar{\bb}}.
\end{equation*}
To prevent confusion with the summation indices in $\bar{\bb}$, the state space is indexed by $k$ and $m$ in the following. 
Calculation shows that the gradients of $\btheta^{\star}$ are given by
\begin{align*}
\frac{\partial\btheta^{\star}}{\partial \mu_k} 
&=\bar{\bA}^{-1} \sum_{r \in [S_R]}\brk{\sum_{m \in [D]} \brk{P_{k,m}\prn{\bb_{k,m,r}-\bA_{k,m,r}\btheta^{\star}}-P_{D,m}\prn{\bb_{D,m,r}-\bA_{D,m,r}\btheta^{\star}}}}\text{ for all } k \in [D-1].\\
\frac{\partial\btheta^{\star}}{\partial \bP_{k,m}} &=\bar{\bA}^{-1}\sum_{r\in [S_R]}\prn{\mu_{k}(\bb_{k,m,r}-\bA_{k,m,r})-\mu_{k}(\bb_{D,m,r}-\bA_{D,m,r})}
\text{ for all } k \in [D] \text{ and } m \in [D-1].\\
\frac{\partial\btheta^{\star}}{\partial \bR_{k,m,r}} &=\bar{\bA}^{-1}\prn{(\bb_{k,m,r}-\bA_{k,m,r})-(\bb_{D,m,r}-\bA_{D,m,r})}\text{ for all } k \in [D] \text{ and } m \in [D-1].
\end{align*}
Thus,
\begin{align*}
    (\nabla_{ \mu}\btheta^{\star})^{\top}J_{\mu}^{\dagger}\nabla_{\mu}\btheta^{\star} &= \bar{\bA}^{-1} \Bigg\{\sum_{k=1}^{D}\mu_{k}\EB_{k}\brk{\prn{\bA_{k,m,r}\btheta^{\star}-\bb_{k,m,r}}}\EB_{k}\brk{\prn{\bA_{k,m,r}\btheta^{\star}-\bb_{k,m,r}}^{\top}}\\
    &-\prn{\sum_{k=1}^{D}\mu_{k}\EB_{k}\brk{\prn{\bA_{k,m,r}\btheta^{\star}-\bb_{k,m,r}}}}\prn{\sum_{k=1}^{D}\mu_{k}\EB_{k}\brk{\prn{\bA_{k,m,r}\btheta^{\star}-\bb_{k,m,r}}}}^{\top}\Bigg\}\bar{\bA}^{-\top},\\
    (\nabla_{ \bP_k}\btheta^{\star})^{\top}J_{\bP_k}^{\dagger}\nabla_{\bP_k}\btheta^{\star} &= \mu_{k}\bar{\bA}^{-1} \Bigg\{\EB_{k}\brk{\brk{\prn{\bA_{k,m,r}\btheta^{\star}-\bb_{k,m,r}}}\brk{\prn{\bA_{k,m,r}\btheta^{\star}-\bb_{k,m,r}}^{\top}}}\\
    &\quad -\EB_{k}\brk{\prn{\bA_{k,m,r}\btheta^{\star}-\bb_{k,m,r}}}\EB_{k}\brk{\prn{\bA_{k,m,r}\btheta^{\star}-\bb_{k,m,r}}^{\top}}\Bigg\}\bar{\bA}^{-\top},\\
    (\nabla_{ \bR_{k,m}}\btheta^{\star})^{\top}J_{\bR_{k,m}}^{\dagger}\nabla_{\bR_{k,m}}\btheta^{\star} &= \bar{\bA}^{-1} \Bigg\{
    \EB_{k,m}\brk{\brk{\prn{\bA_{k,m,r}\btheta^{\star}-\bb_{k,m,r}}}\brk{\prn{\bA_{k,m,r}\btheta^{\star}-\bb_{k,m,r}}^{\top}}}\\
    &\quad -\EB_{k,m}\brk{\prn{\bA_{k,m,r}\btheta^{\star}-\bb_{k,m,r}}}\EB_{k,m}\brk{\prn{\bA_{k,m,r}\btheta^{\star}-\bb_{k,m,r}}^{\top}}\Bigg\}\bar{\bA}^{-\top}.
\end{align*}
Putting these together and applying the covariance decomposition formula yields the following result.
\begin{align*}                  
(\nabla\btheta^{\star})^{\top}J_{\gI}^{\dagger}\nabla\btheta^{\star} = \bar{\bA}^{-1}\operatorname{cov}\{\bA_{k,m,r}\btheta^{\star}-\bb_{k,m,r}\} \bar{\bA}^{-\top} = \bar{\bA}^{-1}\bSigma_{\be}\bar{\bA}^{-\top}.
\end{align*}
Recall that the loss function is set as $\iota_K\|\cdot\|_{\bJ}^2$. By Hajek-Le Cam local asymptotic minimax theorem, the local asymptotic minimax risk is given by
\begin{equation*}
\mathfrak{M}(\mathcal{I})= \iota_K\EB\brk{\norm{Z}^2_{\bJ}}, \quad\text{where }Z\sim \gN\prn{0,\bar{\bA}^{-1}\bSigma_{\be}\bar{\bA}^{-\top}}.
\end{equation*}
Recall that $\widetilde{\bSigma}_{\be}=\bJ^{\frac{1}{2}} \bar{\bA}^{-1} \bSigma_{\be} \bar{\bA}^{-\top} \bJ^{\frac{1}{2}}$. 
Thus, $\mathfrak{M}(\mathcal{I})=\iota_K\tr(\widetilde{\bSigma}_{\be})$.
\section{Proof for VrFLCTD}\label{appendix:proof_VrFLCTD}
\noindent In this appendix, we provide the proof for the finite-sample convergence rate of {\VCTD}. 
Here are some preliminary lemmas for upper bounding the bias introduces by the stochastic temporal difference operator $\bh_t$ which is important in our convergence analysis.
\subsection{Preliminary Analysis}
The following lemma provides an upper bound for the difference between stochastic operators at different $\btheta$, where the subscript $t$ in $\bh_{t}$ is omitted in the generative model setting. 
This result follows Assumption 2 of \citet{li2023accelerated}. 
\begin{lemma}\label{lem:variance_of_stochastic_operator}
There exists a constant $\varsigma = \sqrt{2(1+\gamma)}>0$ such that for every $\btheta, \btheta^{\prime} \in \mathbb{R}^{dK}$,
\begin{equation*}
\EB\norm{(\bh(\btheta)-\bh(\btheta^{\prime}))-(\bar{\bh}(\btheta)-\bar{\bh}(\btheta^{\prime})}_2^2 \leq \varsigma^2\norm{\btheta-\btheta^{\prime}}_{\bJ}^2.
\end{equation*}
\end{lemma}
\begin{proof}
Note that
\begin{align*}
    &\EB\norm{(\bh(\btheta)-\bh(\btheta^{\prime}))-(\bar{\bh}(\btheta)-\bar{\bh}(\btheta^{\prime})}_2^2 \leq \varsigma^2\norm{\btheta-\btheta^{\prime}}_{\bJ}^2\\
    &\impliedby \EB \brk{(\bA-\bar{\bA})^{\top}(\bA-\bar{\bA})}\preccurlyeq \varsigma^2\bI_{K}\otimes\bSigma_{\bphi}\\
    &\iff \EB \brk{\bA^{\top}\bA}-\bar{\bA}^{\top}\bar{\bA}\preccurlyeq\varsigma^2\bI_{K}\otimes\bSigma_{\bphi}\\
    &\impliedby \EB \brk{\bA^{\top}\bA}\preccurlyeq \varsigma^2\bI_{K}\otimes\bSigma_{\bphi}.
\end{align*}
Using the upper bound we derived in Equation 33 in \citet{peng2025finitesampleanalysisdistributional}, that is,  $\EB[\bA^{\top}\bA]\preccurlyeq2(1+\gamma)\bI_K\otimes\bSigma_{\bphi}$, we complete the proof.
\end{proof}

With Lemma~\ref{lem:variance_of_stochastic_operator} at hand, we next qualify the inner loop objective $\underline{\btheta}$ induced by $\widetilde{\btheta}$. Recall $\bF_t(\btheta_t) = \bar{\bh}_{t}(\btheta_t)-\bar{\bh}_{t}(\widetilde{\btheta})+\widehat{\bh}(\widetilde{\btheta})$. 
Let $\underline{\btheta}$ satisfy  $-\bar{\bh}(\underline{\btheta}) = \bar{\bA}(\btheta^{\star}-\underline{\btheta}) = \widehat{\bh}(\widetilde{\btheta})-\bar{\bh}(\widetilde{\btheta})$. 
The following lemma gives an upper bound of the distance between $\underline{\btheta}$ and the real objective $\btheta^{\star}$.
\begin{lemma}\label{lem:underline_theta_discrepancy}
Consider a single epoch with index $n \in[N]$. 
We have
\begin{align*}
    \EB\brk{\norm{\btheta^{\star}-\underline{\btheta}}_{\bJ}^2} 
    &\leq 2/l_{n}\prn{\tr(\widetilde{\bSigma}_{\be})+ \frac{\varsigma^2}{(1-\sqrt{\gamma})^2\lambda_{\min}}\EB \brk{\norm{\btheta^{\star}-\widetilde{\btheta}}_{\bJ}^2}}.
\end{align*}
\end{lemma}
\begin{proof}
By the definition of $\underline{\btheta}$, we have $\bar{\bA}(\btheta^{\star}-\underline{\btheta}) = \widehat{\bh}(\widetilde{\btheta})-\bar{\bh}(\widetilde{\btheta})$. Thus
\begin{equation*}
    \bJ^{1/2}(\btheta^{\star}-\underline{\btheta}) = \bJ^{1/2}\bar{\bA}^{-1}\prn{\widehat{\bh}(\widetilde{\btheta})-\bar{\bh}(\widetilde{\btheta})}.
\end{equation*} 
Therefore, by Equation~\eqref{eq:norm_A_TJA_1} and Lemma~\ref{lem:variance_of_stochastic_operator}, we have
\begin{align*}
    &\EB\brk{\norm{\btheta^{\star}-\underline{\btheta}}_{\bJ}^2}
    = \EB\brk{\norm{\bJ^{1/2}\bar{\bA}^{-1}\prn{\widehat{\bh}(\widetilde{\btheta})-\bar{\bh}(\widetilde{\btheta})}}_2^2}\\
    &\leq 2\prn{\EB\brk{\norm{\bJ^{1/2}\bar{\bA}^{-1}\prn{\widehat{\bh}(\btheta^{\star})-\bar{\bh}(\btheta^{\star})}}_2^2}+\EB\brk{\norm{\bJ^{1/2}\bar{\bA}^{-1}\prn{\widehat{\bh}(\widetilde{\btheta})-\bar{\bh}(\widetilde{\btheta})}-\prn{\widehat{\bh}(\btheta^{\star})-\bar{\bh}(\btheta^{\star})}}_2^2}}\\
    &\leq 2/l_{n}\prn{\EB\brk{\norm{\bJ^{1/2}\bar{\bA}^{-1}\prn{\bh(\btheta^{\star})-\bar{\bh}(\btheta^{\star})}}_2^2}+\EB\brk{\norm{\bJ^{1/2}\bar{\bA}^{-1}\prn{\bh(\widetilde{\btheta})-\bar{\bh}(\widetilde{\btheta})}-\prn{\bh(\btheta^{\star})-\bar{\bh}(\btheta^{\star})}}_2^2}}\\
    &\leq    2/l_{n}\prn{\EB\brk{\norm{\bJ^{1/2}\bar{\bA}^{-1}\prn{\bA\btheta^{\star}-\bb}}_2^2}+ \frac{1}{(1-\sqrt{\gamma})^2\lambda_{\min}}\EB \brk{\norm{\prn{\bh(\widetilde{\btheta})-\bar{\bh}(\widetilde{\btheta})}-\prn{\bh(\btheta^{\star})-\bar{\bh}(\btheta^{\star})}}_2^2}}\\
    &\leq    2/l_{n}\prn{\tr(\widetilde{\bSigma}_{\be})+ \frac{\varsigma^2}{(1-\sqrt{\gamma})^2\lambda_{\min}}\EB \brk{\norm{\btheta^{\star}-\widetilde{\btheta}}_{\bJ}^2}}.
\end{align*}
\end{proof}
\subsection{Proof of convergence results for {\VCTD}}
For convenience, we define $\bm{\Delta}_t(\btheta):=\bar{\bh}_t(\btheta)-\bar{\bh}(\btheta)(=(\bA_{t}-\bar{\bA})\btheta)$ and $\bH_t(\btheta) = \bm{\Delta}_t(\btheta)-\bm{\Delta}_t(\widetilde{\btheta})$. 
Note that for $\lambda=1$, each update is $\btheta_t-\underline{\btheta}=\btheta_{t+1}+\alpha[ \bF_t(\btheta_t)+( \bF_t(\btheta_t)- \bF_{t-1}(\btheta_{t-1}))]-\underline{\btheta}$. 
By Equation 67 in \citet{li2023accelerated}, the following recursive expression is valid for the local error $\|\btheta_t-\underline{\btheta}\|_2^2$.

\begin{equation}\label{eqn:theta_1-underline_theta}
\norm{\btheta_1-\underline{\btheta}}_2^2=\norm{\btheta_{T+1}-\underline{\btheta}}_2^2-2 \alpha\left\langle \bF_{T+1}\left(\btheta_{T+1}\right)- \bF_T\left(\btheta_T\right), \btheta_{T+1}-\underline{\btheta}\right\rangle
+\sum_{t=1}^T 2 \alpha\left\langle \bF_{t+1}\left(\btheta_{t+1}\right), \btheta_{t+1}-\underline{\btheta}\right\rangle+Q_1,
\end{equation}
where the last term is defined as 
\begin{align*}
    Q_1&:=\sum_{t=1}^T \| \btheta_{t+1} -\btheta_t \|_2^2+\sum_{t=1}^T 2 \alpha\left\langle \bar{\bh}(\btheta_t)-\bar{\bh}(\btheta_{t-1}), \btheta_{t+1}-\btheta_t\right\rangle +\sum_{t=1}^T 2 \alpha\left\langle \bH_t(\btheta_t)-\bH_{t-1}(\btheta_{t-1}), \btheta_{t+1}-\btheta_t\right\rangle.
\end{align*}
According to Equation 68 in \citet{li2023accelerated}, the following holds when $2\alpha(1+\sqrt\gamma)\leq 1/2$.
\begin{align*}
Q_1 & \geq \sum_{t=1}^T\norm{\btheta_{t+1}-\btheta_t}_2^2-\sum_{t=1}^T 2 \alpha(1+\sqrt{\gamma}) \norm{\btheta_t-\btheta_{t-1}}_2\norm{\btheta_{t+1}-\btheta_t}_2 \\
& \geq\frac{1}{4}\norm{\btheta_{T+1}-\btheta_T}_2^2-\sum_{t=1}^T 4 \alpha^2\left(
\norm{\bH_t(\btheta_t)}_2^2+
\norm{\bH_{t-1}(\btheta_{t-1})}_2^2\right),
\end{align*}
where we apply Lemma~\ref{lem:A_bilinear} and Young's inequality. 
Substituting into Equation~\eqref{eqn:theta_1-underline_theta} and doing rearrangement, we have
\begin{align*}
&\norm{\btheta_{T+1}-\underline{\btheta}}_2^2+\sum_{t=1}^T 2 \alpha\left\langle \bF_{t+1}\left(\btheta_{t+1}\right), \btheta_{t+1}-\underline{\btheta}\right\rangle -\sum_{t=1}^T 4 \alpha^2\prn{\norm{\bH_t(\btheta_t)}_2^2+\norm{\bH_{t-1}(\btheta_{t-1})}_2^2} \leq \norm{\btheta_1-\underline{\btheta}}_2^2+R,
\end{align*}
where
\begin{equation*}
R := 2 \alpha\left\langle \bF_{T+1}\left(\btheta_{T+1}\right)- \bF_T\left(\btheta_T\right), \btheta_{T+1}-\underline{\btheta}\right\rangle
-\frac{1}{4} \norm{ \btheta_{T+1}- \btheta_T}_2^2.
\end{equation*}
Let
$Q_2:=2 \alpha\langle
\bH_{T+1}(\btheta_{T+1}),
\btheta_{T+1}-\underline{\btheta}\rangle
-2 \alpha\langle
\bH_{T}(\btheta_{T}),
\btheta_T-\underline{\btheta}\rangle$.
By Equation 70 in \citet{li2023accelerated}, if $\|\bar{\bh}(\btheta)-\bar{\bh}(\btheta^{\prime})\|_2 \leq c\|\btheta-\btheta^{\prime}\|_2$ (in our case, $c=\lambda_{\max}(1+\sqrt{\gamma})$), then 
\begin{align}
& (1-8 \alpha^2 c^2)\norm{\btheta_{T+1}-\underline{\btheta}}_2^2+\sum_{t=1}^T 2 \alpha\left\langle \bF_{t+1}\left(\btheta_{t+1}\right), \btheta_{t+1}-\underline{\btheta}\right\rangle \notag\\
& \quad \leq\norm{\btheta_1-\underline{\btheta}}_2^2+\sum_{t=1}^T 8 \alpha^2\norm{\bH_{t}(\btheta_{t})}_2^2+4 \alpha^2\norm{\bH_{T}(\btheta_{T})}_2^2+Q_2.\label{eq:Markovian_start}
\end{align}
We now take the expectation on both sides and note that $\EB Q_2 =0$ (due to $\EB (\bA_t-\bar{\bA})=0$). 
For LHS, Lemma~\ref{lem:A_bilinear} is applied; and for RHS, the basic inequality $\|\btheta_1-\underline{\btheta}\|_{2}^2 \leq \|\btheta_1-\underline{\btheta}\|_{\bJ}^2/\lambda_{\min}$ and Lemma~\ref{lem:variance_of_stochastic_operator} are applied. 
Thus, we have
\begin{align*}
&\left(1-8 \alpha^2 \lambda_{\max}^2(1+\sqrt{\gamma})^2\right)\EB\brk{\norm{\btheta_{T+1}-\underline{\btheta}}_2^2}+\sum_{t=1}^T 2 \alpha(1-\sqrt{\gamma}) \EB\brk{\norm{\btheta_{t+1}-\underline{\btheta}}_{\bJ}^2}\\
& \quad \leq \frac{1}{\lambda_{\min}}\EB\brk{\norm{\btheta_1-\underline{\btheta}}_{\bJ}^2}+  \frac{8\alpha^2\varsigma^2}{m}\sum_{t=1}^T\EB\brk{\norm{\btheta_{t+1}-\widetilde{\btheta}}_{\bJ}^2}
+\frac{4\alpha^2\varsigma^2}{m}\EB\brk{\norm{\btheta_{T}-\widetilde{\btheta}}_{\bJ}^2}.
\end{align*}
Applying Young's inequality on $\|\btheta_{t+1}-\widetilde{\btheta}\|_{\bJ}^2$ and $\alpha<1/(4(1+\sqrt{\gamma}))$, we have
\begin{align*}
    \sum_{t=1}^T \prn{2 \alpha(1-\sqrt{\gamma})-\frac{24\alpha^2\varsigma^2}{m}} \EB\brk{\norm{\btheta_{t+1}-\underline{\btheta}}_{\bJ}^2} \leq \prn{\frac{1}{\lambda_{\min}}+\frac{24T\alpha^2\varsigma^2}{m}}\EB\brk{\norm{\widetilde{\btheta}-\underline{\btheta}}_{\bJ}^2}.
\end{align*}
Given the Polyak-Ruppert average as $\widehat{\btheta}_k:=\sum_{t=2}^{T+1} \btheta_t/T$, applying Young's inequality and Lemma~\ref{lem:underline_theta_discrepancy}, we have 
\begin{align*}
\EB\brk{\norm{\widehat{\btheta}_n-\btheta^{\star}}_{\bJ}^2}&\leq 
(4/C_T+96/C_{m}+2(4/C_T+96/C_{m}+2)/C_{N})\EB\left[\|\widetilde{\btheta}-\btheta^{\star}\|_{\bJ}^2\right]\\
&\quad+(2(4/C_T+96/C_{m}+2)/C_{N})\frac{\tr(\widetilde{\bSigma}_{\be})}{l_{n}},
\end{align*}
where
$
T \geq C_T/(\lambda_{\min}(1-\sqrt{\gamma}) \alpha)$, $m \geq \max\{1,C_m\alpha\varsigma^2/(1-\sqrt{\gamma})\}$, $l_{n} \geq C_{N}\varsigma^2/(\lambda_{\min}(1-\sqrt{\gamma})^2).
$
Setting $C_T= 80,C_m = 480, C_N = 18$ completes the proof.
\section{Proof for Markovian Setting}\label{appendix:proof_Markovian}
\noindent To adapt {\VCTD} for the Markovian data stream $\{\widetilde{\xi}_t\}_{t=0}^\infty=\{(s_t,a_t,s_t^{\prime},r_t)\}_{t=0}^\infty$, we incorporate a standard burn-in period of length $l_{0}$ in each outer loop.  
Consequently, when computing the operator $\widehat{\bh}$ in Algorithm~\ref{alg:VrFLCTD}, only the last $l_{n} - l_{0}$ samples from each segment of $l_{n}$ consecutive samples are used. 
This approach reduces the bias introduced by correlated data and achieves the effect of variance reduction. 
The following theorem is central to proving Theorem~\ref{thm:Markovian_convergence}.
\begin{proposition}\label{prop:Markovian_recursive}
    Consider a single epoch with index $n \in[N]$. Assume $l_{0}$ satisfies $\tau$ satisfies and $\rho^\tau \leq \frac{(1-\rho)^2}{2 C_{\bphi}}$, and $l_{n}$ satisfies $\rho^{l_{n}-l_{0}} \leq \frac{\tau(1-\rho)}{4 C_{\bphi}(l_{n}-l_{0})}$. 
\begin{itemize}
\item 
If $\bm{\eta}_{\btheta^{\star}} = \bm{\eta}^{\pi,K}$, then we have
\begin{equation*}
\EB\brk{\norm{\widehat{\btheta}_n-\btheta^{\star}}_{\bJ}^2}\leq\frac{1}{2}\EB\brk{\|\widehat{\btheta}_{n-1}-\btheta^{\star}\|_{\bJ}^2}+\frac{9\tr(\widetilde{\bSigma}_{\be})}{l_{n}-l_{0}}.
\end{equation*}
\item
If $\bm{\eta}_{\btheta^{\star}} \neq \bm{\eta}^{\pi,K}$, then we have
\begin{align*}
\EB\brk{\norm{\widehat{\btheta}_n-\btheta^{\star}}_{\bJ}^2} 
&\leq \frac{1}{2}\EB\brk{\norm{\widetilde{\btheta}_{n-1}-\btheta^{\star}}_{\bJ}^2}+\frac{9 \tr(\widetilde{\bSigma}_{\be,\text{Mkv}})}{(l_{n}-l_{0})} 
+ \frac{10\tau^2 \tr(\widetilde{\bSigma}_{\be})}{(l_{n}-l_{0})^2}\\ 
&\quad+ \frac{9\tau}{(l_{n}-l_{0})^2\lambda_{\min}(1-\sqrt{\gamma})^2} \frac{\ell_{2,\mu_\pi}^2\prn{\bm{\eta}^{\pi,K},\bm{\eta}_{\btheta^{\star}}}}{\iota_K}.
\end{align*}
\end{itemize}
\end{proposition}
The proof is provided in the next subsection. 
Taking Proposition~\ref{prop:Markovian_recursive} as given for the moment, let us complete the proof of Theorem~\ref{thm:Markovian_convergence}. 
Consider the situation where $\bm{\eta}_{\btheta^{\star}} \neq \bm{\eta}^{\pi,K}$ and the following holds under the condition $l_{n}-l_{0}\geq (3/4)^{N-n}M$.
\begin{align*}
\EB\brk{\norm{\widehat{\btheta}_N-\btheta^{\star}}_{\bJ}^2}
&\leq\frac{1}{2^N}\EB\brk{\|\btheta^{0}-\btheta^{\star}\|_{\bJ}^2}+\sum_{n=1}^{N}(\frac{2}{3})^{N-n}\frac{9\tr(\widetilde{\bSigma}_{\be,\text{Mkv}})}{M}\\
&\quad+\sum_{n=1}^{N}(\frac{8}{9})^{N-n} \prn{\frac{10\tau^2\tr(\widetilde{\bSigma}_{\be})}{M^2}+\frac{9\tau}{M^2\lambda_{\min}(1-\sqrt{\gamma})^2} \frac{\ell_{2,\mu_\pi}^2\prn{\bm{\eta}^{\pi,K},\bm{\eta}_{\btheta^{\star}}}}{\iota_K}}\\
&\leq\frac{1}{2^N}\EB\brk{\|\btheta^{0}-\btheta^{\star}\|_{\bJ}^2}+
\frac{27\tr(\widetilde{\bSigma}_{\be,\text{Mkv}})}{M}+
\frac{81\tau}{M^2\lambda_{\min}(1-\sqrt{\gamma})^2} \frac{\ell_{2,\mu_\pi}^2\prn{\bm{\eta}^{\pi,K},\bm{\eta}_{\btheta^{\star}}}}{\iota_K}+\frac{90\tau^2\tr(\widetilde{\bSigma}_{\be})}{M^2}.
\end{align*}
This completes the proof of the upper bound part of the main theorem. 
We need an extra upper bound on the linear approximation error $\ell_{2,\mu_\pi}^2(\bm{\eta}^{\pi,K},\bm{\eta}_{\btheta^{\star}})/\iota_{K}$ for further analysis. 
Same as Lemma~\ref{lem:trace_critical}, $K$ appears linearly in the upper bound of both this term and the trace term $\tr(\widetilde{\bSigma}_{\be})$. This proves crucial for {\VCTD} to achieve K-independent convergence rate.
\begin{lemma}\label{lem:appro_linear}
For any $K\geq (1-\gamma)^{-1}$,
\begin{equation*}
\frac{\ell_{2,\mu_\pi}^2\prn{\bm{\eta}^{\pi,K},\bm{\eta}_{\btheta^{\star}}}}{\iota_K}\lesssim K/\lambda_{\min}.
\end{equation*}
\end{lemma}
\begin{proof}
Since $\bm{\eta}^{\pi},\bm{\eta}_{\bm{0}} \in \sP = \Delta([0,(1-\gamma)^{-1}])$, their Cram\'er distance is bounded by $(1-\gamma)^{-1}$. 
We have
\begin{align*}
\ell_{2,\mu_\pi}^2\prn{\bm{\eta}^{\pi,K},\bm{\eta}_{\btheta^{\star}}}
&\leq \ell_{2,\mu_\pi}^2\prn{\bm{\eta}^{\pi,K},\bm{\eta}^{\pi}}+\ell_{2,\mu_\pi}^2\prn{\bm{\eta}^{\pi},\bm{\eta}_{\bm{0}}}+\ell_{2,\mu_\pi}^2\prn{\bm{\eta}_{\bm{0}},\bm{\eta}_{\btheta^{\star}}}\\
&\leq \frac{1}{1-\gamma}\ell_{2,\mu_{\pi}}^2\prn{\bPi_{K}\bm{\eta}^\pi,\bm{\eta}^\pi}+\frac{1}{1-\gamma}+\ell_{2,\mu_\pi}^2\prn{\bm{\eta}_{\bm{0}},\bm{\eta}_{\btheta^{\star}}}\\
&\leq \frac{1}{K(1-\gamma)^2}+\frac{1}{1-\gamma}+\ell_{2,\mu_\pi}^2\prn{\bm{\eta}_{\bm{0}},\bm{\eta}_{\btheta^{\star}}},
\end{align*}
where the first term in the decomposition is upper bounded using Equation~\eqref{eq:ctd_worst_case_approx_error} and Proposition~9.18 and Equation 5.28 in \citet{bdr2022}.
Recall that $\btheta^{\star} = \bar{\bA}^{-1}\bar{\bb}$ and consequently
$\ell_{2,\mu_{\pi}}^2(\bm{\eta}_{\btheta^{\star}},\bm{\eta}_{\mathbf{0}})/\iota_K=\|\btheta^{\star}\|_{\bJ}^2=\|\bJ^{1/2}\bar{\bA}^{-1}\bar{\bb}\|^2$.
Using Equation~\eqref{eqn:norm_of_J1/2_barA_-1}, Lemma~\ref{lem:tr_cov_sigma_e} and $\mathrm{g} = 1/(1-\sqrt{\gamma})$,
\begin{equation*}
\left\|\bJ^{1 / 2} \bar{\bA}^{-1} \bar{\bb}\right\|^2 \leq \mathrm{g}^2/\lambda_{\min}\norm{\bar{\bb}}^2 \lesssim K/\lambda_{\min}.
\end{equation*}
Using this upper bound and the condition $K\geq (1-\gamma)^{-1}$,
\begin{equation*}
\frac{\ell_{2,\mu_\pi}^2\prn{\bm{\eta}^{\pi,K},\bm{\eta}_{\btheta^{\star}}}}{\iota_K} 
\leq \frac{2}{\iota_K(1-\gamma)}+K/\lambda_{\min}= 2K+K/\lambda_{\min}
\end{equation*}
Since $\lambda_{\min}\leq 1$, the proof is completed.
\end{proof}
\subsection{Preliminary Lemmas}
Before we proceed to the proof, some preliminary lemmas are necessary to analyze the stochastic operator $\bh_t$ in the Markovian setting. Let $\gF_t$ be the $\sigma$-algebra generated by $\{\widetilde{\xi_0}, \ldots, \widetilde{\xi_t}\}$, which represents the information available up to data stream length $t$.
\begin{lemma}\label{lem:markovian_decay}
For every $t, \tau \in \mathbb{Z}_{+}$, with probability 1,
\begin{equation*}
\norm{\EB\brk{\bh_{t+\tau}(\btheta^{\star}) \mid \gF_t}-\bar{\bh}(\btheta^{\star})}_2 \leq C_{\bphi} \cdot \rho^\tau \frac{\ell_{2,\mu_\pi}\prn{\bm{\eta}^{\pi,K},\bm{\eta}_{\btheta^{\star}}}}{\sqrt{\iota_{K}}},
\end{equation*}
where $C_{\bphi}:=\frac{C_P\sqrt{2(1+\gamma)}}{\sqrt{\min _{i \in[D]} \pi_i}}=\frac{C_P \varsigma}{\sqrt{\min _{i \in[D]} \pi_i}}$.
\begin{proof}
In the following, we use the notation $\bA$ with a tilde $\widetilde{\bA}$ to emphasize the fact that it is generated from Markovian data streams, as we do in the notation for the Markovian streaming data $\{\widetilde{\xi}_{t}\}_{t=0}^\infty=\{(s_t,a_t,s_t^{\prime},r_t)\}_{t=0}^\infty$. 
Without loss of generality, set $t=0$, and by the basic representation of linear-categorical parameterization in Equation~\eqref{eq:def_linear_parametrize}, it holds that
\begin{align*}
     \bp_{\btheta^{\star}}(s_\tau)& =  (\bC^{-1} \otimes \bphi(s_\tau)^{\top}  ) \btheta^{\star} + \frac{1}{K+1} \bm{1}_K,
\end{align*}
\begin{align*}
      \bp_{\bm{\eta}^{\pi,K}}(s_{\tau})
      =\bp_{\gT^\pi\bm{\eta}^{\pi,K}}(s_{\tau})
     &=\EB_{s_{\tau}}\brk{\tilde\bG(r_\tau) \prn{\bp_{\bm{\eta}^{\pi,K}}(s_{\tau+1})-\frac{1}{K+1}\bm{1}_{K}}}+\frac{1}{K+1}\sum_{j=0}^K\EB_{s_{\tau}}\brk{\bg_j(r_\tau) },
\end{align*}
and the expectation of the drift from the initial operator is given by
\begin{align}
\EB\brk{\bh_{0+\tau}(\btheta^{\star}) \mid \gF_0}-\bar{\bh}(\btheta^{\star}) 
&=\sum_{s_{\tau}=1}^{D}\pi_{s_0\rightarrow s_{\tau}}\prn{\EB_{s_{\tau}}\brk{\widetilde{\bA}_{\tau}\btheta^{\star}-\widetilde{\bb}_{\tau}}-(\bar{\bA}\btheta^{\star}-\bar{\bb})}\notag\\
&=\sum_{s_{\tau}=1}^{D}\prn{\pi_{s_0\rightarrow s_{\tau}} -\pi_{s_\tau}}\prn{\EB_{s_{\tau}}\brk{\widetilde{\bA}_{\tau}\btheta^{\star}-\widetilde{\bb}_{\tau}}}.\label{eq:h_tau}
\end{align}
Recall that the categorical projected Bellman equation $\bm{\eta}=\bPi_{K}\gT^{\pi}\bm{\eta}$ admits a unique solution $\bm{\eta}^{\pi,K}$. 
The second equation is based on Proposition C.3 in \citet{peng2025finitesampleanalysisdistributional}. 
For convenience, let $\bH_{s_\tau}:=\bC(\bp_{\btheta^{\star}}(s_\tau)-\bp_{\bm{\eta}^{\pi,K}}(s_{\tau}))$. 
It holds that
\begin{align}
    &\EB_{s_{\tau}}\brk{\widetilde{\bA}_{\tau}\btheta^{\star}-\widetilde{\bb}_{\tau}} \notag\\
    &\quad= \Bigg(\brk{\bI_K{\otimes}\prn{\bphi(s_{\tau})\bphi(s_{\tau})^{\top}}}{-}\EB_{s_{\tau}}\brk{\prn{\bC\tilde{\bG}(r_\tau)\bC^{-1}}{\otimes}\prn{\bphi(s_{\tau})\bphi(s_{\tau+1})^{\top}}}\Bigg)\btheta^{\star}\notag\\
    &\quad\quad-\frac{1}{K+1}\EB_{s_{\tau}}\brk{\bC\prn{\sum_{j=0}^K\bg_j(r_\tau)-\bm{1}_K}}\otimes\bphi(s_{\tau})\notag\\
    &\quad=\bC\prn{\bp_{\btheta^{\star}}(s_\tau)-\bp_{\bm{\eta}^{\pi,K}}(s_{\tau})}\otimes \bphi(s_{\tau})+\notag\\
    &\qquad \bC \prn{\EB_{s_{\tau}}\brk{\tilde\bG(r_\tau) \prn{\bp_{\bm{\eta}^{\pi,K}}(s_{\tau+1})-\frac{1}{K+1}\bm{1}_{K}}-\prn{\brk{\tilde\bG(r_t) \bC^{-1}}\otimes \bphi(s_{t+1})^{\top}} \btheta^{\star}}} \otimes \bphi(s_{\tau})\notag\\
    &\quad=\bC\prn{\bp_{\btheta^{\star}}(s_\tau)-\bp_{\bm{\eta}^{\pi,K}}(s_{\tau})}\otimes \bphi(s_{\tau})+\bC \prn{\EB_{s_{\tau}}\brk{\tilde\bG(r_\tau) \prn{\bp_{\bm{\eta}^{\pi,K}}(s_{\tau+1})-\bp_{\btheta^{\star}}(s_{\tau+1})}}} \otimes \bphi(s_{\tau})\notag\\
    &\quad=\bC\prn{\bp_{\btheta^{\star}}(s_\tau)-\bp_{\bm{\eta}^{\pi,K}}(s_{\tau})}\otimes \bphi(s_{\tau})+\prn{\EB_{s_{\tau}}\brk{\bY(r_\tau)\bC \prn{\bp_{\bm{\eta}^{\pi,K}}(s_{\tau+1})-\bp_{\btheta^{\star}}(s_{\tau+1})}}} \otimes \bphi(s_{\tau}) \notag\\
    &\quad=\bH_{s_\tau}\otimes \bphi(s_{\tau})-\prn{\sum_{s_{\tau+1}=1}^{D}P_{s_{\tau},s_{\tau+1}}\EB_{s_{\tau},s_{\tau+1}}\bY(r_\tau)\bH_{s_{\tau+1}}}\otimes \bphi(s_{\tau}).\label{eq:Atheta-b}
\end{align}
By Lemma~\ref{lem:Spectra_of_ccgcc}, $\bY(r_\tau)=\bC \tilde\bG(r_\tau) \bC^{-1}$ has a norm less than $\sqrt{\gamma}$. Note that $\bphi(s_{\tau})$ has a norm less than $1$ and for any matrices $\bB_1$ and $\bB_2$, it holds that $\|\bB_1 \otimes \bB_2|| = \|\bB_1\|\|\bB_2\|$.
Taking norm of both sides, by Young's inequality and Jensen's inequality, we have
\begin{align*}
\norm{\sum_{s_\tau=1}^{D}\sqrt{\pi_{s_\tau}}\EB_{s_{\tau}}\brk{\widetilde{\bA}_{\tau}\btheta^{\star}-\widetilde{\bb}_{\tau}}}^2 
&\leq \sum_{s_\tau=1}^{D}\pi_{s_\tau}\norm{\EB_{s_{\tau}}\brk{\widetilde{\bA}_{\tau}\btheta^{\star}-\widetilde{\bb}_{\tau}}}^2\\
&\leq 2\sum_{s_\tau=1}^{D}\pi_{s_\tau} \prn{\norm{\bH_{s_\tau}}^2+\gamma\sum_{s_{\tau+1}=1}^{D}P_{s_{\tau},s_{\tau+1}}\norm{\bH_{s_{\tau+1}}}^2}\\
&= 2\sum_{s_\tau=1}^{D}\pi_{s_\tau} \norm{\bH_{s_\tau}}^2+\gamma\sum_{s_{\tau+1}=1}^{D}\pi_{s_{\tau+1}}\norm{\bH_{s_{\tau+1}}}^2\\
&\leq 2(1+\gamma)\sum_{s=1}^{D}\pi_{s} \norm{\bH_{s}}^2.
\end{align*}
Thus, taking norm of Equation~\eqref{eq:h_tau} and absorbing Assumption~\ref{assum:mixing_time} yields
\begin{align*}
\norm{\EB\brk{\bh_{0+\tau}(\btheta^{\star}) \mid \gF_0}-\bar{\bh}(\btheta^{\star})}^2 &= \norm{\sum_{s_{\tau}=1}^{D}\frac{\pi_{s_0\rightarrow s_{\tau}} -\pi_{s_\tau}}{\sqrt{\pi_{s_\tau}}}\sqrt{\pi_{s_\tau}}\prn{\EB_{s_{\tau}}\brk{\widetilde{\bA}_{\tau}\btheta^{\star}-\widetilde{\bb}_{\tau}}}}^2\\
&\leq \frac{C_P^2\rho^{2\tau}}{\min_{i\in [D]}\pi_{i}}\norm{\sum_{s_{\tau}=1}^{D}\sqrt{\pi_{s_\tau}}\prn{\EB_{s_{\tau}}\brk{\widetilde{\bA}_{\tau}\btheta^{\star}-\widetilde{\bb}_{\tau}}}}^2\\
&\leq \frac{C_P^2\rho^{2\tau}}{\min_{i\in [D]}\pi_{i}}2(1+\gamma)\sum_{s=1}^{D}\pi_{s} \norm{\bH_{s}}^2\\
&=\frac{2C_P^2\rho^{2\tau}(1+\gamma)}{(\min_{i\in [D]}\pi_{i})\iota_K}\iota_K\EB_{s\sim\pi}\brk{\norm{\bC\prn{\bp_{\bm{\eta}^{\pi,K}}(s)-\bp_{\btheta^{\star}}(s)}}^2}\\
&=\frac{2C_P^2\rho^{2\tau}(1+\gamma)}{(\min_{i\in [D]}\pi_{i})\iota_K}\EB_{s\sim\pi}\brk{\ell_{2}^2\prn{\eta^{\pi,K}(s),\eta_{\btheta^{\star}}(s)}}\\
&=\frac{2C_P^2\rho^{2\tau}(1+\gamma)}{(\min_{i\in [D]}\pi_{i})\iota_K}\ell_{2,\pi}^2\prn{\bm{\eta}^{\pi,K},\bm{\eta}_{\btheta^{\star}}}.
\end{align*}
The proof is completed.
\end{proof}
\end{lemma} 
\begin{lemma}\label{lem:delicate_h_bias}
For every $t, \tau \in \mathbb{Z}_{+}$ and $\btheta, \btheta^{\prime} \in \mathbb{R}^d$, with probability 1,
\begin{equation*}
\norm{\EB\brk{\bh_{t+\tau}\prn{\btheta}\mid \gF_t}-\EB\brk{\bh_{t+\tau}\prn{\btheta^{\prime}}\mid \gF_t}-\brk{\bar{\bh}(\btheta)-\bar{\bh}\prn{\btheta^{\prime}}}}_2 \leq C_{\bphi} \cdot \rho^\tau\norm{\btheta-\btheta^{\prime}}_{\bJ}.
\end{equation*}
\begin{proof}
Similar to the former lemma, set $t=0$ and it holds that
\begin{align*}
    &\norm{\EB\brk{\bh_{\tau}\prn{\btheta}\mid \gF_0}-\EB\brk{\bh_{\tau}\prn{\btheta^{\prime}}\mid \gF_0}-\brk{\bar{\bh}(\btheta)-\bar{\bh}(\btheta^{\prime})}}_2^2=\norm{\sum_{s_{\tau}=1}^{D}\pi_{s_0\rightarrow s_{\tau}}\prn{\EB_{s_{\tau}}\brk{\prn{\widetilde{\bA}_{\tau}-\bar{\bA}}\prn{\btheta-\btheta^{\prime}}}}}_2^2,
\end{align*}
where, the same as in Equation~\eqref{eq:Atheta-b},
\begin{align*}
    \EB_{s_{\tau}}\brk{\prn{\widetilde{\bA}_{\tau}-\bar{\bA}}\prn{\btheta-\btheta^{\prime}}}&=\Bigg(\brk{\bI_K{\otimes}\prn{\bphi(s_{\tau})\bphi(s_{\tau})^{\top}}}{-}\EB_{s_{\tau}}\brk{\prn{\bC\tilde{\bG}(r_\tau)\bC^{-1}}{\otimes}\prn{\bphi(s_{\tau})\bphi(s_{\tau+1})^{\top}}}\Bigg)(\btheta-\btheta^{\prime})\\
    &=\bC\prn{\bp_{\btheta}(s_\tau)-\bp_{\btheta^{\prime}}(s_{\tau})}\otimes \bphi(s_{\tau})+\prn{\EB_{s_{\tau}}\brk{\bY(r_\tau)\bC \prn{\bp_{\btheta}(s_{\tau+1})-\bp_{\btheta^{\prime}}(s_{\tau+1})}}} \otimes \bphi(s_{\tau}).
\end{align*}
Hence, we have
\begin{align*}
    \norm{\sum_{s_{\tau}=1}^{D}\pi_{s_0\rightarrow s_{\tau}}\prn{\EB_{s_{\tau}}\brk{\prn{\widetilde{\bA}_{\tau}-\bar{\bA}}\prn{\btheta-\btheta^{\prime}}}}}_2^2
    &\leq\frac{2C_P^2\rho^{2\tau}(1+\gamma)}{\min_{i\in [D]}\pi_{i}}\EB_{s\sim\pi}\brk{\norm{\bC\prn{\bp_{\btheta}(s)-\bp_{\btheta^{\prime}}(s)}}^2}\\
    &=\frac{2C_P^2\rho^{2\tau}(1+\gamma)}{\min_{i\in [D]}\pi_{i}}\norm{\btheta-\btheta^{\prime}}_{\bJ}^2.
\end{align*}
\end{proof}
\end{lemma}

\begin{lemma}\label{lem:markovian_variance_stochastic_operator}
For every $t \in \mathbb{Z}_{+}, \btheta, \btheta^{\prime} \in \mathbb{R}^d$, and $\tau \in \mathbb{Z}_{+}$ such that 
\begin{equation*}\label{eq:con_mixing_time}
    C_P \cdot \rho^\tau \leq \min _{i \in[D]} \pi_i,
\end{equation*}
then with probability 1,
\begin{equation*}
\EB\brk{\norm{\bh_{t+\tau}\prn{\btheta}-\bh_{t+\tau}(\btheta^{\prime})-\bar{\bh}(\btheta)+\bar{\bh}(\btheta^{\prime})}_2^2 \mid \mathcal{F}_t} \leq 2 \varsigma^2\norm{\btheta-\btheta^{\prime}}_{\bJ}^2.
\end{equation*}

\end{lemma} 
\begin{proof}
Combining Assumption~\ref{assum:mixing_time} and Lemma~\ref{lem:variance_of_stochastic_operator}, we have
\begin{align*}
    &\EB\brk{\norm{\bh_{t+\tau}\prn{\btheta}-\bh_{t+\tau}(\btheta^{\prime})-\bar{\bh}(\btheta)+\bar{\bh}(\btheta^{\prime})}_2^2 \mid \mathcal{F}_t}\\
    &=\sum_{k=1}^{D} \mathbb{P}(s_{t+\tau}=k|s_t)\cdot\EB_{s_{t+\tau}}\brk{\norm{\prn{\widetilde{\bA}_{t+\tau}-\bar{\bA}}\prn{\btheta-\btheta^{\prime}}}_2^2}\\
    &\leq \sum_{k=1}^{D}\norm{\mathbb{P}(s_{t+\tau}=k|s_t)-\pi_{k}}\EB_{s_{t+\tau}}\brk{\norm{\prn{\widetilde{\bA}_{t+\tau}-\bar{\bA}}\prn{\btheta-\btheta^{\prime}}}_2^2}+\sum_{k=1}^{D}\pi_{k}\EB_{s_{t+\tau}}\brk{\norm{\prn{\widetilde{\bA}_{t+\tau}-\bar{\bA}}\prn{\btheta-\btheta^{\prime}}}_2^2}\\
    &\leq 2\sum_{k=1}^{D}\pi_{k}\EB_{s_{t+\tau}}\brk{\norm{\prn{\widetilde{\bA}_{t+\tau}-\bar{\bA}}\prn{\btheta-\btheta^{\prime}}}_2^2}\\
    &=2\EB\brk{\norm{\prn{\bA-\bar{\bA}}\prn{\btheta-\btheta^{\prime}}}_2^2}\\
    &=2\EB\brk{\norm{(\bh(\btheta)-\bh(\btheta^{\prime}))-(\bar{\bh}(\btheta)-\bar{\bh}(\btheta^{\prime})}_2^2}\\
    &\leq 2\varsigma^2\norm{\btheta-\btheta^{\prime}}_{\bJ}^2.
\end{align*}
\end{proof}
With Lemma~\ref{lem:variance_of_stochastic_operator}, Lemma~\ref{lem:delicate_h_bias} and Lemma~\ref{lem:markovian_variance_stochastic_operator} at hand, the following two lemmas are direct corollaries paralleling Lemma 12 and Lemma 13 in \citet{li2023accelerated}.
\begin{lemma}\label{lem:markovian_2nd_moment_stochastic_operator}
For every $\btheta, \btheta^{\prime} \in \mathbb{R}^d, l_{0}, \tau \in \mathbb{Z}_{+}$, if $l_{0}$ satisfies and $\tau$ satisfies
\begin{equation*}
\rho^\tau \leq \frac{2(1-\rho)\varsigma}{3 C_{\bphi}},
\end{equation*}
then we have with probability 1,
\begin{equation*}
\EB\norm{(\widehat{\bh}(\btheta)-\widehat{\bh}(\btheta^{\prime}))-(\bar{\bh}(\btheta)-\bar{\bh}(\btheta^{\prime}))}_2^2 \leq \frac{4(\tau+1) \varsigma^2}{l_{n}-l_{0}}\norm{\btheta-\btheta^{\prime}}_{\bJ}^2,
\end{equation*}
and similarly
\begin{equation*}
\EB\norm{(\bh_t(\btheta)-\bh_t(\btheta^{\prime}))-(\bar{\bh}(\btheta)-\bar{\bh}(\btheta^{\prime}))}_2^2 \leq \frac{4(\tau+1) \varsigma^2}{m-m_0}\norm{\btheta-\btheta^{\prime}}_{\bJ}^2.
\end{equation*}
\end{lemma}

\begin{lemma}\label{lem:delicate_h_cut}
For every $\btheta, \btheta^{\prime} \in \mathbb{R}^d, m_0 \in[m]$ and $t \in[T]$, with probability 1 ,
\begin{equation*}
\norm{\EB\brk{\bar{\bh}_t(\btheta) \mid \mathcal{F}_{t-1}}-\EB\brk{\bar{\bh}_t(\btheta^{\prime}) \mid \mathcal{F}_{t-1}}-(\bar{\bh}(\btheta)-\bar{\bh}(\btheta^{\prime}))}_2 \leq \frac{C_{\bphi} \rho^{m_0}}{(1-\rho)(m-m_0)}\norm{\btheta-\btheta^{\prime}}_{\bJ}.
\end{equation*}
\end{lemma}

We distinguishes between two situations based on whether $\bm{\eta}_{\btheta^{\star}} = \bm{\eta}^{\pi,K}$ to highlight the impact of linear function approximation following \citet{li2023accelerated}, where $\bm{\eta}^{\pi,K}$ is the solution to the categorical projected Bellman equation $\bm{\eta}=\bPi_{K}\gT^{\pi}\bm{\eta}$.
\begin{lemma}\label{lem:estimation_error_each_epoch}
Consider a single epoch with index $n \in[N]$. Assume $l_{0}$ satisfies $\tau$ satisfies and $\rho^\tau \leq \frac{(1-\rho)^2}{2 C_{\bphi}}$, and $l_{n}$ satisfies $\rho^{l_{n}-l_{0}} \leq \frac{\tau(1-\rho)}{4 C_{\bphi}(l_{n}-l_{0})}$. 
\begin{itemize}
\item 
If $\bm{\eta}_{\btheta^{\star}} = \bm{\eta}^{\pi,K}$, then we have
\begin{equation*}
\EB\brk{\|\underline{\btheta}-\btheta^{\star}\|_{\bJ}^2} \leq \frac{4 \tr(\widetilde{\bSigma}_{\be})}{l_{n}-l_{0}}+\frac{8(\tau+1) \varsigma^2}{(1-\sqrt{\gamma})^2 \lambda_{\min}(l_{n}-l_{0})}\norm{\btheta^{\star}-\widetilde{\btheta}}_{\bJ}^2.
\end{equation*}
\item 
If $\bm{\eta}_{\btheta^{\star}} \neq \bm{\eta}^{\pi,K}$, then we have
\begin{align*}
    \EB\brk{\norm{\underline{\btheta}-\btheta^{\star}}_{\bJ}^2} \leq & 
    \frac{4 \tr(\widetilde{\bSigma}_{\be,\text{Mkv}})}{l_{n}-l_{0}} 
    + \frac{(4\tau^2+2\tau+2) \tr(\widetilde{\bSigma}_{\be})}{(l_{n}-l_{0})^2} 
    + \frac{2(\tau+1)}{(l_{n}-l_{0})^2} \frac{\ell_{2,\mu_\pi}^2\prn{\bm{\eta}^{\pi,K},\bm{\eta}_{\btheta^{\star}}}}{\lambda_{\min}(1-\sqrt{\gamma})^2\iota_K}\\
    & \quad+ \frac{8(\tau+1) \varsigma^2}{(1-\sqrt\gamma)^2 \lambda_{\min}(l_{n}-l_{0})}\norm{\btheta^{\star}-\widetilde{\btheta}}_{\bJ}^2.
\end{align*}
\end{itemize}
\end{lemma}
\begin{proof}
Recall that $\bar{\bA}(\btheta^{\star}-\underline{\btheta}) = \widehat{\bh}(\widetilde{\btheta})-\bar{\bh}(\widetilde{\btheta})$.
In the following, $\bh_{i}(\btheta) = \bA_{i}\btheta-\bb_{i}$ is induced by the data stream in the Markovian setting. By Young's inequality, first decompose the expectation of local error into two parts as
\begin{align}
    &\EB\brk{\norm{\underline{\btheta}-\btheta^{\star}}_{\bJ}^2}
    = \EB\brk{\norm{\bJ^{1/2}\bar{\bA}^{-1}\prn{\widehat{\bh}(\widetilde{\btheta})-\bar{\bh}(\widetilde{\btheta})}}_2^2} \notag\\
    &\leq 2\prn{\EB\brk{\norm{\bJ^{1/2}\bar{\bA}^{-1}\prn{\widehat{\bh}(\btheta^{\star})-\bar{\bh}(\btheta^{\star})}}_2^2}+\EB\brk{\norm{\bJ^{1/2}\bar{\bA}^{-1}\brk{\prn{\widehat{\bh}(\widetilde{\btheta})-\bar{\bh}(\widetilde{\btheta})}-\prn{\widehat{\bh}(\btheta^{\star})-\bar{\bh}(\btheta^{\star})}}}_2^2}}\notag\\
    &:=2(R_1+R_2).\label{eq:loss_of_Markovian}
\end{align}
According to Equation 105 in \citet{li2023accelerated}, it holds that
\begin{equation}\label{eq:R1}
    R_1 \leq\frac{2}{(l_{n}-l_{0})^2}\brk{(l_{n}-l_{0})\Xi_{0}+\sum_{k=1}^{l_{n}-l_{0}-1}2(l_{n}-l_{0}-k)\Xi_{k}},
\end{equation}
  where $\Xi_{t} :=\EB\langle\bJ^{1/2}\bar{\bA}^{-1}(\bh_{0}(\btheta^{\star})-\bar{\bh}(\btheta^{\star})),\bJ^{1/2}\bar{\bA}^{-1}(\bh_{t}(\btheta^{\star})-\bar{\bh}(\btheta^{\star}))\rangle$. Analogous to $\widetilde{\bSigma}_{\be}$, let $\widetilde{\be}_{t}=\bh_{t}(\btheta^{\star})-\bar{\bh}(\btheta^{\star})$, $\bSigma_{\be,\text{Mkv}}=\sum_{-\infty}^{\infty}\EB[\widetilde{\be}_{t}\widetilde{\be}_{0}^{\top}]$ and let
  \begin{equation*}
\widetilde{\bSigma}_{\be,\text{Mkv}}=\bJ^{1/2}\bar{\bA}^{-1}\bSigma_{\be,\text{Mkv}}\bar{\bA}^{-\top}\bJ^{1/2}.
  \end{equation*}
Then it holds that $\tr(\widetilde{\bSigma}_{\be,\text{Mkv}}) = \Xi_{0} + 2\sum_{t=1}^{\infty}\Xi_{t}$ and it is clear that $\Xi_{0} = \tr(\widetilde{\bSigma}_{\be})$ since samples are from a stationary Markovian trajectory.

\paragraph{If $\bm{\eta}_{\btheta^{\star}} = \bm{\eta}^{\pi,K}$}
\begin{align*}
 \bp_{\bm{\eta}^{\pi,K}}(s_t) &= \bp_{\btheta^{\star}}(s_t) =  (\bC^{-1} \otimes \bphi(s_t)^{\top}  ) \btheta^{\star} + \frac{1}{K+1} \bm{1}_K,\\
     \bp_{\gT^\pi\bm{\eta}^{\pi,K}}(s_t)
     &=\EB_{s_t}\brk{\tilde\bG(r_t) \prn{\bp_{\bm{\eta}^{\pi,K}}(s_{t+1})-\frac{1}{K+1}\bm{1}_{K}}}+\frac{1}{K+1}\sum_{j=0}^K\EB_{s_t}\brk{\bg_j(r_t) }\\
     &=\EB_{s_t}\brk{\tilde\bG(r_t) \prn{\bp_{\btheta^{\star}}(s_{t+1})-\frac{1}{K+1}\bm{1}_{K}}}+\frac{1}{K+1}\sum_{j=0}^K\EB_{s_t}\brk{\bg_j(r_t) }\\
     &=\EB_{s_t}\brk{\prn{\brk{\tilde\bG(r_t) \bC^{-1}}\otimes \bphi(s_{t+1})^{\top}} \btheta^{\star}} 
     +\frac{1}{K+1}\sum_{j=0}^K\EB_{s_t}\brk{\bg_j(r_t) }.
\end{align*}
By the definition of $\bm{\eta}^{\pi,K}$, these two equations should be the same. Thus, using the basic property of Kronecker products, we have that
\begin{align*}
    &\frac{1}{K+1}\EB_{s_t}\brk{\bC\prn{\sum_{j=0}^K\bg_j(r_t)-\bm{1}_K}}\otimes\bphi(s_t)\\&\quad=\bC\prn{\frac{1}{K+1}\sum_{j=0}^K\EB_{s_t}\brk{\bg_j(r_t)}-\frac{1}{K+1} \bm{1}_K}\otimes \bphi(s_t)\\
        &\quad= \bC\prn{(\bC^{-1} \otimes \bphi(s_t)^{\top}  ) \btheta^{\star}-\EB_{s_t}\brk{\prn{\brk{\tilde\bG(r_t) \bC^{-1}}\otimes \bphi(s_{t+1})^{\top}} \btheta^{\star}}}\otimes\bphi(s_t)\\
    &\quad=\Bigg(\brk{\bI_K{\otimes}\prn{\bphi(s_t)\bphi(s_t)^{\top}}}{-}\EB_{s_t}\brk{\prn{\bC\tilde{\bG}(r_t)\bC^{-1}}{\otimes}\prn{\bphi(s_t)\bphi(s_{t+1})^{\top}}}\Bigg)\btheta^{\star}.
\end{align*}
This result is in fact a restriction of linear-categorical projected Bellman equation to given state $s_t$.
Thus, for $t\neq 0$ it holds that
\begin{align}
    \EB_{s_t}\brk{\widetilde{\bA}_{t}\btheta^{\star}-\widetilde{\bb}_{t}} 
    &= \Bigg(\brk{\bI_K{\otimes}\prn{\bphi(s_t)\bphi(s_t)^{\top}}}{-}\EB_{s_t}\brk{\prn{\bC\tilde{\bG}(r_t)\bC^{-1}}{\otimes}\prn{\bphi(s_t)\bphi(s_{t+1})^{\top}}}\Bigg)\btheta^{\star} \notag\\
    &\quad-\frac{1}{K+1}\EB_{s_t}\brk{\bC\prn{\sum_{j=0}^K\bg_j(r_t)-\bm{1}_K}}\otimes\bphi(s_t) = 0.\label{eq:xi_0}
\end{align}
Note that by definition of $\Xi_{t}$,
\begin{align*}
    \Xi_{t} &=\EB\langle\bJ^{1/2}\bar{\bA}^{-1}\prn{\bh_{0}(\btheta^{\star})-\bar{\bh}(\btheta^{\star})},\bJ^{1/2}\bar{\bA}^{-1} \EB\brk{\prn{\bh_{t}(\btheta^{\star})-\bar{\bh}(\btheta^{\star})}|\gF_{0}}\rangle\\
    &= \EB\langle\bJ^{1/2}\bar{\bA}^{-1}\prn{\bh_{0}(\btheta^{\star})-\bar{\bh}(\btheta^{\star})},\bJ^{1/2}\bar{\bA}^{-1} \sum_{s_t=1}^{D}\pi_{s_0\rightarrow s_t}\prn{\EB_{s_t}\brk{\widetilde{\bA}_{t}\btheta^{\star}-\widetilde{\bb}_{t}}-(\bar{\bA}\btheta^{\star}-\bar{\bb})}\rangle.
\end{align*}
By Equation~\eqref{eq:xi_0}, it is clear that $\Xi_{t}$ is $0$ for $t\neq 0$ and $\tr(\widetilde{\bSigma}_{\be,\text{Mkv}}) = \Xi_{0} + 2\sum_{t=1}^{\infty}\Xi_{t} = \Xi_{0}= \tr(\widetilde{\bSigma}_{\be})$. 
Combining Equation~\eqref{eq:R1}, we have
\begin{equation*}
    R_1\leq \frac{2\Xi_0}{l_{n}-l_{0}} = 2 \frac{\tr(\widetilde{\bSigma}_{\be})}{l_{n}-l_{0}}.
\end{equation*}

\paragraph{If $\bm{\eta}_{\btheta^{\star}} \neq \bm{\eta}^{\pi,K}$}
Using Equation~\eqref{eqn:norm_of_J1/2_barA_-1} and Lemma~\ref{lem:markovian_decay}, we obtain that
\begin{align*}
    |\Xi_{i}|
    &=|\EB\langle\bJ^{1/2}\bar{\bA}^{-1}\prn{\bh_{0}(\btheta^{\star})-\bar{\bh}(\btheta^{\star})},\bJ^{1/2}\bar{\bA}^{-1} \EB\brk{\prn{\bh_{t}(\btheta^{\star})-\bar{\bh}(\btheta^{\star})}|\gF_{0}}\rangle|\\
    &\leq C_{\bphi}\rho^{i}\frac{1}{\sqrt{\lambda_{\min}}(1-\sqrt{\gamma})}\EB\brk{\frac{\ell_{2,\mu_\pi}\prn{\bm{\eta}^{\pi,K},\bm{\eta}_{\btheta^{\star}}}}{\sqrt{\iota_{K}}}\bJ^{1/2}\bar{\bA}^{-1}\prn{\bh_{0}(\btheta^{\star})-\bar{\bh}(\btheta^{\star})}}\\
    &\leq C_{\bphi}\rho^{i}\prn{\frac{\ell_{2,\mu_\pi}^2\prn{\bm{\eta}^{\pi,K},\bm{\eta}_{\btheta^{\star}}}}{2\lambda_{\min}(1-\sqrt{\gamma})^2\iota_K}+\frac{\tr(\widetilde{\bSigma}_{\be})}{2}}.
\end{align*}
Another useful upper bound for $|\Xi_i|$ is that $|\Xi_i|\leq |\Xi_0|=\tr(\widetilde{\bSigma}_{\be})$. 
From Equation~\eqref{eq:R1}, it holds that
\begin{align*}
R_1 & \leq \frac{2\brk{\left(l_{n}-l_{0}\right) \widetilde{\bSigma}_{\be,\text{Mkv}}-\sum_{i=1}^{l_{n}-l_{0}-1} i \Xi_{i}-2\left(l_{n}-l_{0}\right) \sum_{i=l_{n}-l_{0}}^{\infty} \Xi_{i}}}{\left(l_{n}-l_{0}\right)^2}.
\end{align*}
By Holder's inequality, $|\Xi_{i}|\leq|\Xi_{0}|=\tr(\widetilde{\bSigma}_{\be})$, we write
\begin{align*}
\left|\sum_{i=1}^{l_{n}-l_{0}-1} i \Xi_{i}\right|&\leq
\sum_{i=1}^\tau \tau |\Xi_i| + \sum_{i=\tau+1}^\infty i|\Xi_i|\\
&\leq \tau^2 \tr(\widetilde{\bSigma}_{\be}) + C_{\phi} \prn{\frac{\ell_{2,\mu_\pi}^2\prn{\bm{\eta}^{\pi,K},\bm{\eta}_{\btheta^{\star}}}}{2\lambda_{\min}(1-\sqrt{\gamma})^2\iota_K}+\frac{\tr(\widetilde{\bSigma}_{\be})}{2}} \sum_{i=\tau+1}^\infty i \cdot \rho^i \\
&\leq \tau^2 \tr(\widetilde{\bSigma}_{\be}) + \frac{C_{\phi} [(\tau+1)\rho^{\tau+1}+\rho^{\tau+2}/(1-\rho)]}{1-\rho} \cdot \prn{\frac{\ell_{2,\mu_\pi}^2\prn{\bm{\eta}^{\pi,K},\bm{\eta}_{\btheta^{\star}}}}{2\lambda_{\min}(1-\sqrt{\gamma})^2\iota_K}+\frac{\tr(\widetilde{\bSigma}_{\be})}{2}}.
\end{align*}
And
\begin{align*}
    \left|\sum_{i=l_{n}-l_{0}}^\infty \Xi_i\right| \leq\sum_{i=l_{n}-l_{0}}^\infty |\Xi_i| &\leq \frac{C_{\phi} \rho^{l_{n}-l_{0}}}{1-\rho}\prn{\frac{\ell_{2,\mu_\pi}^2\prn{\bm{\eta}^{\pi,K},\bm{\eta}_{\btheta^{\star}}}}{2\lambda_{\min}(1-\sqrt{\gamma})^2\iota_K}+\frac{\tr(\widetilde{\bSigma}_{\be})}{2}}.
\end{align*}
Combine these upper bounds and recall that $\rho^\tau \leq (1-\rho)^2/(2C_{\phi}) $ and $\rho^{l_{n}-l_{0}} \leq \tau(1-\rho)/(4C_{\phi} (l_{n}-l_{0}))$:
\begin{align*}
    R_1 &\leq \frac{2 \tr(\widetilde{\bSigma}_{\be,\text{Mkv}})}{l_{n}-l_{0}} 
    + \frac{2\tau^2 \tr(\widetilde{\bSigma}_{\be})}{(l_{n}-l_{0})^2} 
    + \frac{\tau+1}{(l_{n}-l_{0})^2} \prn{\frac{\ell_{2,\mu_\pi}^2\prn{\bm{\eta}^{\pi,K},\bm{\eta}_{\btheta^{\star}}}}{\lambda_{\min}(1-\sqrt{\gamma})^2\iota_K}+\tr(\widetilde{\bSigma}_{\be})}.
\end{align*}
By Lemma~\ref{lem:markovian_2nd_moment_stochastic_operator}, we have
\begin{equation*}
R_2 \leq \frac{4(\tau+1)\varsigma^2}{(1-\sqrt\gamma)^2\lambda_{\min}(l_{n}-l_{0})} \norm{\btheta^{\star}-\widetilde{\btheta}}_{\bJ}^2.
\end{equation*}
Substituting bounds for $R_1$ and $R_2$ into Equation~\eqref{eq:loss_of_Markovian} finishes the proof.
\end{proof}
\begin{remark}\label{remark:tr_Markovian_critical}
By similar arguments, the following upper bound holds for $\tr(\widetilde{\bSigma}_{\be,\text{Mkv}})$.
\begin{equation*}
\tr(\widetilde{\bSigma}_{\be,\text{Mkv}})\leq
\frac{C_{\bphi}\rho^{\tau}}{1-\rho}\prn{\frac{\ell_{2,\mu_\pi}^2\prn{\bm{\eta}^{\pi,K},\bm{\eta}_{\btheta^{\star}}}}{\lambda_{\min}(1-\sqrt{\gamma})^2\iota_K}+\tr(\widetilde{\bSigma}_{\be})} +(2\tau+1)\tr(\widetilde{\bSigma}_{\be}).
\end{equation*}
Note that the coefficient $C_{\bphi}\rho^{\tau}/(1-\rho)$ is by an absolute constant under our assumption. 
If this upper bound is absorbed into the analysis in Lemma~\ref{lem:estimation_error_each_epoch}, then the derived sample complexity of {\VCTD} in the Markovian setting will not achieve the best dependence on $\lambda_{\min}^{-1}$ and $(1-\gamma)^{-1}$ as in the generative setting, which is due to the extra $\lambda_{\min}^{-1}$ and $(1-\gamma)^{-1}$ in this bound. 
This observation shows a distinction between the analysis of {\VCTD} in two observation settings.
\end{remark}

\subsection{Proof of Proposition~\ref{prop:Markovian_recursive}}
Recall that $\bF(\btheta) = \bar{\bh}(\btheta)-\bar{\bh}(\widetilde{\btheta})+\widehat{\bh}(\widetilde{\btheta})$ and the inner loop objective $\underline{\btheta}$ satisfies $\bF(\underline{\btheta})=0$.
Following Equation 88 in \citet{li2023accelerated}, we start with rearranging Equation~\eqref{eq:Markovian_start} which yields 
\begin{align}
& \left(1-8 \alpha^2 \lambda_{\max}^2(1+\sqrt{\gamma})^2\right)\norm{\btheta_{T+1}-\underline{\btheta}}_2^2+\sum_{t=1}^T 2 \alpha  (1-\sqrt{\gamma})\norm{\btheta_{t+1}-\underline{\btheta}}_{\bJ}^2 \notag\\
& \quad \leq\norm{\btheta_1-\underline{\btheta}}_2^2
+4 \alpha^2\prn{2\sum_{t=1}^T \norm{\bH_t(\btheta_t)}_2^2
+\norm{\bH_T(\btheta_T)}_2^2}+Q_3,\label{eq:Markovian_new_start_2}
\end{align}
where
\begin{align*}
Q_3
&=2 \alpha\left\langle\bm{\Delta}_{T+1}\left(\btheta_{T+1}\right)-\bm{\Delta}_{T+1}(\widetilde{\btheta}), \btheta_{T+1}-\underline{\btheta}\right\rangle+2 \alpha\left\langle\bm{\Delta}_T(\widetilde{\btheta})-\bm{\Delta}_T\left(\btheta_T\right), \btheta_T-\underline{\btheta}\right\rangle\\
&\quad-\sum_{t=1}^{T}2\alpha\langle \bm{\Delta}_{t+1}(\btheta_{t+1})-\bm{\Delta}_{t+1}(\widetilde{\btheta}),\btheta_{t+1}-\underline{\btheta}\rangle.
\end{align*}
An upper bound of $\EB Q_3$ can be achieved by using Lemma~\ref{lem:delicate_h_cut} for each term in the expectation
\begin{align*}
    &\EB Q_3=2 \alpha \EB\langle\bm{\Delta}_{T+1}\left(\btheta_{T+1}\right)-\bm{\Delta}_{T+1}(\widetilde{\btheta}), \btheta_{T+1}-\underline{\btheta}\rangle+2\alpha\EB\langle\bm{\Delta}_T(\widetilde{\btheta})-\bm{\Delta}_T\left(\btheta_T\right), \btheta_T-\underline{\btheta}\rangle\\
    &\quad\quad-\sum_{t=1}^{T}2\alpha\EB\langle \bF_{t+1}(\btheta_{t+1})-\bF(\btheta_{t+1}),\btheta_{t+1}-\underline{\btheta}\rangle\\
    &\leq \frac{ 2\alpha C_{\bphi}\rho^{m_0}}{(1-\rho)(m-m_0)}\EB\brk{\norm{\btheta_{T+1}-\widetilde{\btheta}}_{\bJ}\norm{\btheta_{T+1}-\underline{\btheta}}_2+\norm{\btheta_{T}-\widetilde{\btheta}}_{\bJ}\norm{\btheta_{T}-\underline{\btheta}}_2+\sum_{t=1}^{T}\norm{\btheta_{t+1}-\widetilde{\btheta}}_{\bJ}\norm{\btheta_{t+1}-\underline{\btheta}}_2}\\
    &\leq C_{1}^{\prime}\EB\brk{\norm{\btheta_{T+1}-\widetilde{\btheta}}_{\bJ}\norm{\btheta_{T+1}-\underline{\btheta}}_{\bJ}+\norm{\btheta_{T}-\widetilde{\btheta}}_{\bJ}\norm{\btheta_{T}-\underline{\btheta}}_{\bJ}+\sum_{t=1}^{T}\norm{\btheta_{t+1}-\widetilde{\btheta}}_{\bJ}\norm{\btheta_{t+1}-\underline{\btheta}}_{\bJ}}.
\end{align*}
where the basic inequality $\|\btheta\|_2\leq\|\btheta\|_{\bJ}/\sqrt{\lambda_{\min}}$ is applied and $C_{1}^{\prime}:=2\alpha C_{\bphi}\rho^{m_0}/(\sqrt{\lambda_{\min}}(1-\rho)(m-m_0))$. 
By Lemma~\ref{lem:markovian_2nd_moment_stochastic_operator}, the expectation of the intermediate terms admits the following upper bound.
\begin{equation*}
4\alpha^2\prn{2\EB\sum_{t=1}^T \norm{\bH_t(\btheta_t)}_2^2
+\EB\norm{\bm{\Delta}_T(\widetilde{\btheta})-\bm{\Delta}_T\left(\btheta_T\right)}_2^2}\leq C_{2}^{\prime}\prn{2\sum_{t=1}^{T}\EB\brk{\norm{\btheta_{t+1}-\widetilde{\btheta}}_{\bJ}^2}+\EB\brk{\norm{\btheta_{T}-\widetilde{\btheta}}_{\bJ}^2}},
\end{equation*}
where $C_{2}^{\prime}:=16\alpha^2(\tau+1)\varsigma^2/(m-m_0)$. 
Taking the expectation of both sides of Equation~\eqref{eq:Markovian_new_start_2} yields
\begin{align*}
    & \left(1-8 \alpha^2 (1+\sqrt{\gamma})^2\right)\EB\brk{\norm{\btheta_{T+1}-\underline{\btheta}}_2^2}+\sum_{t=1}^T 2 \alpha (1-\sqrt{\gamma})\EB\brk{\norm{\btheta_{t+1}-\underline{\btheta}}_{\bJ}^2}-\frac{1}{\lambda_{\min}}\norm{\btheta_1-\underline{\btheta}}_{\bJ}^2\\
    & \quad \leq 6C_{2}^{\prime}\prn{\sum_{t=1}^{T}\EB\brk{\norm{\btheta_{t+1}-\underline{\btheta}}_{\bJ}^2}}
    +3C_{1}^{\prime}\prn{\sum_{t=1}^{T}\EB\brk{\norm{\btheta_{t+1}-\underline{\btheta}}_{\bJ}^2}}+T(2C_{1}^{\prime}+6C_{2}^{\prime})\EB\brk{\norm{\underline{\btheta}-\widetilde{\btheta}}_{\bJ}^2}.
\end{align*}
Note that $\alpha<1/4(1+\sqrt{\gamma})$, we have
\begin{equation*}
\sum_{t=1}^T \prn{2 \alpha (1-\sqrt{\gamma})-3C_{1}^{\prime}-6C_{2}^{\prime}}\EB\brk{\norm{\btheta_{t+1}-\underline{\btheta}}_{\bJ}^2}\leq\prn{\frac{1}{\lambda_{\min}}+T(2C_{1}^{\prime}+6C_{2}^{\prime})\EB\brk{\norm{\underline{\btheta}-\widetilde{\btheta}}_{\bJ}^2}}.
\end{equation*}
If $\rho^{m_0}\leq C_1\alpha(\sqrt{\lambda_{\min}}\tau\varsigma^2(1-\rho))/C_{\bphi}$ and $C_2(m-m_0)\geq (\alpha (\tau+1)\varsigma^2)/(1-\sqrt{\gamma})$, then
\begin{equation*}
        C_{2}^{\prime}\leq 16\alpha C_2(1-\sqrt{\gamma}), \quad C_{1}^{\prime}\leq 2\alpha^2 C_1\frac{\tau\varsigma^2}{m-m_0}\leq 2\alpha C_1 C_2(1-\sqrt{\gamma}).
\end{equation*}
 By Jensen's inequality and $T\geq 1/(C_3\lambda_{\min}\alpha(1-\sqrt{\gamma}))$ we obtain (where $3C_{1}^{\prime}+6C_{2}^{\prime}\leq C_0\alpha(1-\sqrt{\gamma})$, that is, $6C_1C_2+96C_2\leq C_0$)
\begin{align*}
\EB\brk{\norm{\widehat{\btheta}_n-\underline{\btheta}}_{\bJ}^2} &\leq \frac{1}{2-C_0}\cdot\frac{\sum_{t=1}^{T}\alpha(1-\sqrt{\gamma})\EB\brk{\norm{\btheta_{t+1}-\underline{\btheta}}_{\bJ}^2}}{\alpha(1-\sqrt{\gamma})T}\\
&\leq \frac{1}{2-C_0}\cdot\frac{1/\lambda_{\min}+T(2C_{1}^{\prime}+6C_{2}^{\prime})}{\alpha(1-\sqrt{\gamma})T}\EB\brk{\norm{\underline{\btheta}-\widetilde{\btheta}}_{\bJ}^2}\\
&\leq C_4\prn{\EB\norm{\btheta^{\star}-\widetilde{\btheta}}_{\bJ}^2+\EB\norm{\underline{\btheta}-\btheta^{\star}}_{\bJ}^2},
\end{align*}
where $C_4 \geq 2(C_3+96C_2+4C_1C_2)/(2-C_0)$. By Young's inequality, it holds that
\begin{align*}
\EB\brk{\norm{\widehat{\btheta}_n-\btheta^{\star}}_{\bJ}^2} &\leq 2\EB\brk{\norm{\widehat{\btheta}_n-\underline{\btheta}}_{\bJ}^2}+2\EB\brk{\norm{\underline{\btheta}-\btheta^{\star}}_{\bJ}^2}\\
&\leq 2C_4 \EB\brk{\norm{\widetilde{\btheta}-\btheta^{\star}}_{\bJ}^2}+(2C_4+2)\EB\brk{\norm{\underline{\btheta}-\btheta^{\star}}_{\bJ}^2}.
\end{align*}
The last term can be upper bounded using Lemma~\ref{lem:estimation_error_each_epoch} and we extract the coefficient of $\EB[\|\widetilde{\btheta}-\btheta^{\star}\|_{\bJ}^2]$, which should satisfy:
\begin{equation*}
    2C_4+(2C_4+2)\frac{8(\tau+1) \varsigma^2}{(1-\sqrt{\gamma})^2 \lambda_{\min}(l_{n}-l_{0})}\leq 2C_4+(2C_4+2)C_5 \leq 1/2,
\end{equation*}
where $l_{n}-l_{0}\geq 8(\tau+1) \varsigma^2/\prn{C_5 (1-\sqrt{\gamma})^2 \lambda_{\min}}$. 
Setting $C_0=2/9,C_4=1/8,C_5=1/9, C_3= 1/90, C_2= 1/1000$, and $C_1=1$ satisfies all the conditions. 
Collecting the rest of the terms yields what we desire. 
For example, in the $\bm{\eta}_{\btheta^{\star}} \neq \bm{\eta}^{\pi,K}$ case, by $\tau\geq 1$:
\begin{align*}
\EB\brk{\norm{\widehat{\btheta}_n-\btheta^{\star}}_{\bJ}^2} 
&\leq \frac{1}{2}\EB\brk{\norm{\widetilde{\btheta}-\btheta^{\star}}_{\bJ}^2}\\
&\quad+\frac{9}{4} \prn{\frac{4 \tr(\widetilde{\bSigma}_{\be,\text{Mkv}})}{l_{n}-l_{0}} 
+ \frac{(4\tau^2+2\tau+2) \tr(\widetilde{\bSigma}_{\be})}{(l_{n}-l_{0})^2} 
+ \frac{2(\tau+1)}{(l_{n}-l_{0})^2} \frac{\ell_{2,\mu_\pi}^2\prn{\bm{\eta}^{\pi,K},\bm{\eta}_{\btheta^{\star}}}}{\lambda_{\min}(1-\sqrt{\gamma})^2\iota_K}}\\
&\leq \frac{1}{2}\EB\brk{\norm{\widetilde{\btheta}-\btheta^{\star}}_{\bJ}^2}+\frac{9 \tr(\widetilde{\bSigma}_{\be,\text{Mkv}})}{l_{n}-l_{0}}\\
&\quad+ \frac{10\tau^2 \tr(\widetilde{\bSigma}_{\be})}{(l_{n}-l_{0})^2} 
+ \frac{9\tau}{(l_{n}-l_{0})^2\lambda_{\min}(1-\sqrt{\gamma})^2} \frac{\ell_{2,\mu_\pi}^2\prn{\bm{\eta}^{\pi,K},\bm{\eta}_{\btheta^{\star}}}}{\iota_K}.
\end{align*}
\section{Projection}\label{appendix:projection}
This appendix addresses the issue that the signed probability distribution $\bm{\eta}_{\btheta^{\star}}$ of the linear-categorical projected Bellman equation might not be valid categorical distributions. 

To deal with possible invalid probability distributions, technically one needs to take a further projection step of $\bm{\eta}_{\btheta^{\star}}$ to the valid categorical probability measure in $\sP_{K}$ (at the end of the algorithm execution) while guaranteeing that this projection does not increase the approximation error by an intolerable factor.

Recall that the valid categorical
projection operator $\bPi_{K}^{\sP}{\colon}\sP^{\sgn}_K{\to}\sP_K$ is given by
\begin{equation*}
\bPi_{K}^\sP\nu:=\argmin\nolimits_{\nu_\bp\in\sP_{K}}\ell_{2}\prn{\nu, \nu_\bp},\quad \forall\nu\in \sP^{\sgn}_K,
\end{equation*}
where the signed point mass function of $\nu$ is given by $\bp_{\nu}$.
Recall that the loss can be calculated through 
\begin{equation*}
\ell_{2}^2(\nu,\nu_{\bp})=\iota_K\norm{\bC\prn{\bp_{\nu}-\bp}}^2.
\end{equation*}
Hence $\bm{\Pi}_K^\sP\nu\in\sP_K$ is uniquely identified with a vector $\bp^\prime$ satisfying
\begin{equation*}
     \bp^{\prime}=\argmin\nolimits_{\bp\in\RB^{K}_+,\bp^\top \bm{1}_K\leq 1}\frac{1}{2}\norm{\bC\prn{\bp-\bp_{\nu}}}^2.
\end{equation*}
This is a standard quadratic programming problem, which can be efficiently solved numerically in $\gO(K)$ time, despite the absence of a closed-form solution. 

The following lemma shows that $\bPi_{K}^\sP\nu$ is always a better approximation of any ground truth distribution. In practice, one should do this projection on a possibly valid signed distribution for the states of interest. 
\begin{lemma} 
$\ell_{2}^2\prn{\nu_{\pi},\bPi_{K}^\sP\nu} \leq \ell_{2}^2\prn{\nu_{\pi},\nu},\text{ for any } \nu\in \sP^{\sgn}_K \text{ and }\nu_{\pi}\in \sP.$ 
\end{lemma}
\begin{proof}
Since the categorical projection operator is orthogonal projection Lemma~9.17 in \citep{bdr2022}, it follows that 
\begin{align*}
\ell_{2}^2\prn{\nu_{\pi},\nu}&=\ell_{2}^2\prn{\nu_{\pi},\bm{\Pi}_K\nu_{\pi}}+\ell_{2}^2\prn{\bm{\Pi}_K\nu_{\pi},\nu},\\
\ell_{2}^2\prn{\nu_{\pi},\bPi_{K}^\sP\nu}&=\ell_{2}^2\prn{\nu_{\pi},\bm{\Pi}_K\nu_{\pi}}+\ell_{2}^2\prn{\bm{\Pi}_K\nu_{\pi},\bPi_{K}^\sP\nu}.
\end{align*}
Recall that $\bm{\Pi}_K\nu_{\pi}\in\sP^{\sgn}_K$ is uniquely represented with a vector representing the point mass function $\bp_\nu=\prn{p_k(\nu)}_{k=0}^{K}\in\RB^{K+1}$, where
\begin{equation*}
     p_k(\nu)=\textstyle\int_{\brk{0,(1-\gamma)^{-1}}}(1-\abs{(x-x_k)/{\iota_K}})_+ \nu(dx).
\end{equation*}
Clearly this categorical projection produces a valid probability distribution $\bm{\Pi}_K\nu_{\pi} \in \sP_K$. Since $\sP_K$ is a convex set and finite-dimensional Euclidean projection to a convex set is non-expansion, we have that 
\begin{equation*}
   \ell_{2}^2\prn{\bm{\Pi}_K\nu_{\pi},\bPi_{K}^\sP\nu} \leq \ell_{2}^2\prn{\bm{\Pi}_K\nu_{\pi},\nu}.
\end{equation*}
The proof is completed by combining the orthogonal decomposition at the beginning.
\end{proof}

\bibliographystyle{IEEEtran}
\bibliography{bib_TIT}
\end{document}